\documentclass{article} % For LaTeX2e
\usepackage{iclr2023_conference,times}

\usepackage{amsmath,amsfonts,bm}

\def\eqref#1{equation~\ref{#1}}
\def\Eqref#1{Equation~\ref{#1}}
\def\1{\bm{1}}

\DeclareMathAlphabet{\mathsfit}{\encodingdefault}{\sfdefault}{m}{sl}
\SetMathAlphabet{\mathsfit}{bold}{\encodingdefault}{\sfdefault}{bx}{n}

\usepackage{hyperref}
\usepackage{url}

\usepackage{algorithm}

\usepackage[utf8]{inputenc} % allow utf-8 input
\usepackage[T1]{fontenc}    % use 8-bit T1 fonts
\usepackage{booktabs}       % professional-quality tables
\usepackage{amsfonts}       % blackboard math symbols
\usepackage{nicefrac}       % compact symbols for 1/2, etc.
\usepackage{microtype}      % microtypography
\usepackage{graphicx}
\usepackage{amsthm}
\usepackage{amsmath}
\usepackage{amssymb}
\usepackage{bbm}
\usepackage[utf8]{inputenc}
\usepackage[english]{babel}
\usepackage{subcaption}

\usepackage[noend]{algorithmic}

\usepackage{makecell}

\usepackage[toc,page]{appendix}

\usepackage{commath}

\usepackage{cleveref}
\crefformat{footnote}{#2\footnotemark[#1]#3}

\newtheorem{theorem}{Theorem}

\newtheorem{lemma}{Lemma}

\newtheorem{assumption}{Assumption}

\title{Federated Neural Bandits}

\author{Zhongxiang Dai, Yao Shu\thanks{Corresponding author.}, Arun Verma, Flint Xiaofeng Fan \& Bryan Kian Hsiang Low \\
Department of Computer Science, National University of Singapore \\
\texttt{\{daizhongxiang,shuyao,arun,xiaofeng,lowkh\}@comp.nus.edu.sg} \\
\And
Patrick Jaillet \\
Department of Electrical Engineering and Computer Science, MIT\\
\texttt{jaillet@mit.edu}
}

\iclrfinalcopy % Uncomment for camera-ready version, but NOT for submission.
\begin{document}

\maketitle

\begin{abstract}
Recent works on \emph{neural contextual bandits} have achieved compelling performances due to their ability to leverage the strong representation power of neural networks (NNs) for reward prediction. Many applications of contextual bandits involve multiple agents who collaborate without sharing raw observations, thus giving rise to the setting of \emph{federated contextual bandits}. Existing works on federated contextual bandits rely on linear or kernelized bandits, which may fall short when modeling complex real-world reward functions. So, this paper introduces the \emph{federated neural-upper confidence bound} (\texttt{FN-UCB}) algorithm. To better exploit the federated setting, \texttt{FN-UCB} adopts a weighted combination of two UCBs: $\text{UCB}^{a}$ allows every agent to additionally use the observations from the other agents to accelerate exploration (without sharing raw observations), while $\text{UCB}^{b}$ uses an NN with aggregated parameters for reward prediction in a similar way to federated averaging for supervised learning. Notably, the weight between the two UCBs required by our theoretical analysis is amenable to an interesting interpretation, which emphasizes $\text{UCB}^{a}$ initially for \emph{accelerated exploration} and relies more on $\text{UCB}^{b}$ later after enough observations have been collected to train the NNs for accurate reward prediction (i.e., \emph{reliable exploitation}). We prove sub-linear upper bounds on both the cumulative regret and the number of communication rounds of \texttt{FN-UCB}, and empirically demonstrate its competitive performance.
\end{abstract}

\vspace{-1mm}
\section{Introduction}
\label{sec:introduction}
\vspace{-1.8mm}
The stochastic multi-armed bandit is a prominent method for sequential decision-making problems due to its principled ability to handle the exploration-exploitation trade-off \citep{auer2002using,MAL-024,lattimore2020bandit}.
In particular, the stochastic contextual bandit problem has received enormous attention due to its widespread real-world applications such as recommender systems \citep{li2010contextual}, advertising \citep{li2010exploitation}, and healthcare \citep{greenewald2017action}.
In each iteration of a stochastic contextual bandit problem, an agent receives a context (i.e., a $d$-dimensional feature vector) for each of the $K$ arms, selects one of the $K$ contexts/arms, and observes the corresponding reward.
The goal of the agent is to sequentially pull the arms 
%in an intelligent way 
in order to maximize the cumulative reward (or equivalently, minimize the \emph{cumulative regret}) in $T$ iterations.
\vspace{-0.5mm}

To minimize the cumulative regret, \emph{linear contextual bandit} algorithms assume that the rewards can be modeled as a linear function of the input contexts \citep{dani2008stochastic}
% ,rusmevichientong2010linearly} 
and select the arms via classic methods such as upper confidence bound (UCB) \citep{auer2002using} or Thompson sampling (TS) \citep{thompson1933likelihood}, consequently yielding the Linear UCB \citep{abbasi2011improved} and Linear TS \citep{agrawal2013thompson} algorithms.
%The potentially restrictive assumption of a linear model was later relaxed to allow the use of non-linear methods.
%A representative line of works is \emph{kernelized contextual bandit} algorithms \cite{chowdhury2017kernelized,valko2013finite}, 
The potentially restrictive assumption of a linear model was later relaxed by \emph{kernelized contextual bandit} algorithms \citep{chowdhury2017kernelized,valko2013finite}, which assume that the reward function belongs to a reproducing kernel Hilbert space (RKHS) and hence model the reward function using kernel ridge regression or Gaussian process (GP) regression.
However, this assumption may still be restrictive \citep{zhou2020neural} and the kernelized model may fall short when the reward function is very complex and difficult to model. 
To this end, neural networks (NNs), which excel at modeling complex real-world functions, have been adopted to model the reward function in contextual bandits, thereby leading to \emph{neural contextual bandit} algorithms such as Neural UCB \citep{zhou2020neural} and Neural TS \citep{zhang2020neural}.
%Thanks to their ability to exploit the strong expressive power of NNs, Neural UCB and Neural TS have been shown to outperform both linear and kernelized contextual bandit methods \cite{zhang2020neural,zhou2020neural}.
Due to their ability to use the highly expressive NNs for better reward prediction (i.e., \emph{exploitation}),
%and the random features embedding of the \emph{neural tangent kernel} (NTK) \cite{arora2019exact,cao2019generalization,jacot2018neural} for \emph{exploration}, 
Neural UCB and Neural TS have been shown to outperform both linear and kernelized contextual bandit algorithms in practice. 
% \citep{zhang2020neural,zhou2020neural}.
Moreover, the cumulative regrets of Neural UCB and Neural TS have been analyzed by leveraging the theory of the \emph{neural tangent kernel} (NTK) \citep{jacot2018neural},
% \citep{arora2019exact,cao2019generalization,jacot2018neural},
%%sy: More interestingly,
%the cumulative regrets of Neural UCB and Neural TS have been 
%%sy: can be theoretically guaranteed
%analyzed by leveraging the recently developed theoretical framework of the \emph{neural tangent kernel} (NTK) 
%%and the generalization of deep NNs 
%\cite{arora2019exact,cao2019generalization,jacot2018neural}, 
hence making these algorithms both provably efficient and practically effective.
We give a comprehensive review of the related works on neural bandits in App.~\ref{sec:related:works}.
\vspace{-0.5mm}

The contextual bandit algorithms discussed above are only applicable to problems with a single agent.
%sy: on the other hand, moreover, further, different from standard contextual bandit setting with only one single agent,
However, many modern applications of contextual bandits involve multiple agents who (\emph{a}) collaborate with each other for better performances and yet (\emph{b}) are unwilling to share their raw observations (i.e., the contexts and rewards).
For example, 
%multiple 
companies may collaborate to improve their contextual bandits-based recommendation algorithms without sharing their sensitive user data \citep{huang2021federated}, while 
hospitals deploying contextual bandits for personalized treatment may collaborate to improve their treatment strategies without sharing their sensitive patient information \citep{dai2020federated}.
These applications naturally
%sy: naturally
fall under the setting of \emph{federated learning} (FL) \citep{kairouz2019advances,li2021survey} which facilitates collaborative learning of supervised learning models (e.g., NNs) without sharing the raw data.
%For example, the most classic FL method of federated averaging (FedAvg) \cite{mcmahan2016communication}, averages the parameters of the NNs trained by the individual agents.
%sy: the purpose of this 'for example'?
%In this regard, a number of \emph{federated bandit} algorithms have been developed recently, which adopted the setting of FL to allow different agents to collaborate to improve their bandit algorithms \cite{dai2020federated,dubey2020differentially,shi2021federated}.
In this regard, a number of \emph{federated contextual bandit} algorithms have been developed to allow bandit agents to collaborate in the federated setting \citep{shi2021federated}.
% \citep{dai2020federated,dubey2020differentially,huang2021federated,wang2019distributed}.
We present a thorough discussion of the related works on federated contextual bandits in App.~\ref{sec:related:works}.
Notably, \citet{wang2019distributed} have adopted the Linear UCB policy and developed a mechanism to allow every agent to additionally use the observations from the other agents to \textbf{accelerate exploration}, while only requiring the agents to exchange some sufficient statistics 
%(for calculating the Linear UCB policy) 
instead of their raw observations.
However, these previous works 
% on federated contextual bandits 
have only relied on either linear \citep{dubey2020differentially,huang2021federated} or kernelized \citep{dai2020federated,dai2021differentially} methods
% to model the reward function, 
which, as 
% we have 
discussed above, may lack the expressive power to model complex real-world reward functions \citep{zhou2020neural}.
%To this end, in this work, we develop the first \emph{federated neural contextual bandit} algorithm named \emph{federated neural-UCB} (\texttt{FN-UCB}), which is able to exploit the strong representation power of NNs for federated bandit.
%As a results, a natural question arises: \emph{Can federated bandit exploit the strong representation power of NNs?}
Therefore, this naturally
%sy: finally
brings up the need to \textbf{use NNs for better exploitation} (i.e., reward prediction) in federated contextual bandits,
thereby motivating the need for a \emph{federated neural contextual bandit} algorithm.
%sy: federated neural contextual bandit?
\vspace{-0.5mm}

To develop a federated neural contextual bandit algorithm, 
%an important technical challenge is how to simultaneously achieve the following two goals:
an important technical challenge is how to leverage the federated setting to simultaneously
(\emph{a}) \textbf{accelerate exploration} by allowing every agent to additionally use the observations from the other agents without requiring the exchange of raw observations (in a similar way to that of~\citet{wang2019distributed}),
and (\emph{b}) \textbf{improve exploitation} by further enhancing the quality of the NN for reward prediction through the federated setting (i.e., without requiring centralized training using the observations from all agents).
%(\emph{a}) allowing every agent to use the observations from the other agents to \textbf{accelerate exploration} (in a similar way to \cite{wang2019distributed}) without requiring the exchange of raw observations,
%and (\emph{b}) enhancing the quality of the NN for reward prediction to \textbf{improve exploitation}, by leveraging the federated setting without requiring centralized train using the observations from all agents.
In this work, we provide a theoretically grounded solution to tackle this challenge by deploying a weighted combination of two upper confidence bounds (UCBs).
The first UCB, denoted by $\text{UCB}^{a}$, incorporates the neural tangent features (i.e., the random features embedding of NTK) into the Linear UCB-based mechanism adopted by \citet{wang2019distributed}, which achieves the first goal of 
%letting every agent additionally use the observations from the other agents for 
accelerating \emph{exploration}.
The second UCB, denoted by $\text{UCB}^{b}$, adopts an aggregated NN whose parameters are the average of the parameters of the NNs trained by all agents using their local observations for better reward prediction (i.e., better \emph{exploitation} in the second goal).
%The second UCB (denoted as $\text{UCB}^{b}$) adopts an NN for reward prediction (i.e., exploitation) whose parameters are the average of the parameters of the trained NNs from all agents (using their local observations).
Hence, $\text{UCB}^{b}$ improves the quality of the NN for reward prediction in a similar way to the most classic FL method of
%standard FL for supervised learning such as 
federated averaging (FedAvg) for supervised learning \citep{mcmahan2016communication}.
%sy: our propose principled way to decide the weight between ...
Notably, our choice of the weight between the two UCBs, which naturally arises during our theoretical analysis, 
%required by our theoretical analysis 
has an interesting practical interpretation (Sec.~\ref{subsec:weight:two:ucbs}): More weight is given to $\text{UCB}^{a}$ in earlier iterations, which allows us to use the observations from the other agents to accelerate the exploration in the early stage; more weight is assigned to $\text{UCB}^{b}$ only in later iterations after every agent has collected enough local observations to train its NN for accurate reward prediction (i.e., reliable exploitation). 
%sy: i.e., more exploration during the initial phase and more exploitation during the end phase.
Of note, our novel design of the weight (Sec.~\ref{subsec:weight:two:ucbs}) is crucial for our theoretical analysis and may be of broader interest for future works on related topics.
\vspace{-0.5mm}

%sy: should we connect these two paragraphs?
%In this work, 
%With this solution to the above-mentioned technical challenge, 
This paper introduces the first federated neural contextual bandit algorithm which we call \emph{federated neural-UCB} (\texttt{FN-UCB}) (Sec.~\ref{sec:fn_ucb}).
%We analyze both the regret and communication complexity of \texttt{FN-UCB} (Sec.~\ref{sec:theoretical:analyis}). 
We derive an upper bound on its total cumulative regret from all $N$ agents: 
% $R_T = \widetilde{O}(\widetilde{d}\sqrt{TN} + \widetilde{d}_{\max} N \sqrt{TN})$,
$R_T = \widetilde{O}(\widetilde{d}\sqrt{TN} + \widetilde{d}_{\max} N \sqrt{T})$\footnote{The $\widetilde{\mathcal{O}}$ ignores all logarithmic factors.} where $\widetilde{d}$ is the effective dimension of the contexts from all $N$ agents and $\widetilde{d}_{\max}$ represents the maximum among the $N$ individual effective dimensions of the contexts from the $N$ agents (Sec.~\ref{sec:background}).
The communication complexity (i.e., total number of communication rounds in $T$ iterations) of \texttt{FN-UCB}
%, denoted as $C_T$, 
can be upper-bounded by $C_T = \widetilde{\mathcal{O}}(\widetilde{d}\sqrt{N})$.
Finally, we use both synthetic and real-world contextual bandit experiments to explore the interesting insights about our \texttt{FN-UCB} and demonstrate its effective practical performance (Sec.~\ref{sec:experiments}).

\vspace{-2mm}
\section{Background and Problem Setting}
\label{sec:background}
\vspace{-2mm}
Let $[k]$ denote the set $\{1,2,\ldots,k\}$ for a positive integer $k$, $\mathbf{0}_k$ represent a $k$-dimensional vector of $0$'s, and $\mathbf{0}_{k\times k}$ denote an all-zero matrix with dimension $k\times k$.
Our setting 
%of federated neural contextual bandit 
involves $N$ agents with the same reward function $h$ defined on a domain $\mathcal{X}\subset \mathbb{R}^d$.
We consider centralized and synchronous communication: The communication 
%among agents 
is coordinated by a central server and every agent exchanges information with the central server during a \emph{communication round}.
In each iteration $t\in[T]$, every agent $i\in[N]$ receives a set $\mathcal{X}_{t,i}\triangleq\{x^k_{t,i}\}_{k\in[K]}$ of $K$ context vectors and selects one of them $x_{t,i} \in \mathcal{X}_{t,i}$ to be queried
%After that, a noisy observation 
to observe a noisy output
$y_{t,i}\triangleq h(x_{t,i}) + \epsilon$ 
%is produced  
where $\epsilon$ is an $R$-sub-Gaussian noise.
We will analyze the total \emph{cumulative regret} from all $N$ agents in $T$ iterations: $R_T \triangleq \sum^N_{i=1}\sum^T_{t=1}r_{t,i}$  where $r_{t,i} \triangleq h(x_{t,i}^*) - h(x_{t,i})$ and $x_{t,i}^* \triangleq {\arg\max}_{x\in\mathcal{X}_{t,i}}h(x)$.

%sy: should we also give a cite for the formulation of this NN we have used?
Let $f(x;\theta)$ denote the output of a fully connected NN for input $x$ with parameters $\theta$ (of dimension $p_0$) and $g(x;\theta)$ denote the corresponding (column) gradient vector.
We focus on NNs with ReLU activations, and use $L\geq2$ and $m$ to denote its depth and width (of every layer), respectively.
We follow the initialization technique from~\citet{zhang2020neural,zhou2020neural} to initialize the NN parameters $\theta_0 \sim \text{init}(\cdot)$.
Of note, 
%following~\cite{dai2020federated,dai2021differentially}, 
as a common ground for collaboration, we let all $N$ agents share the same initial parameters $\theta_0$ when training their NNs and computing their \emph{neural tangent features}: $g(x;\theta_0)/\sqrt{m}$ (i.e.,  the random features embedding of  NTK~\citep{zhang2020neural}).
Also, let $\mathbf{H}$ denote the 
$(TKN) \times (TKN)$-dimensional
NTK matrix on the set of all received $TKN$ contexts~\citep{zhang2020neural,zhou2020neural}.
Similarly, let $\mathbf{H}_i$ denote the $(TK) \times (TK)$-dimensional NTK matrix 
%for agent $i$, which only makes use of 
on the set of $TK$ contexts received by agent $i$.
We defer the details on the definitions of $\mathbf{H}$ and $\mathbf{H}_i$'s, the NN $f(x;\theta)$, and the initialization scheme $\theta_0 \sim \text{init}(\cdot)$ to App.~\ref{app:more:background} due to limited space.
% lack of space.

Next, let $\mathbf{h}\triangleq [h(x^k_{t,i})]_{t\in[T],i\in[N],k\in[K]}$ denote the $(TKN)$-dimensional column vector 
%containing the reward functions values at all observed contexts: 
of reward function values at all received contexts and $B$ be an absolute constant s.t.~$\sqrt{2 \mathbf{h}^{\top} \mathbf{H}^{-1} \mathbf{h}} \leq B$.
This is related to the commonly adopted assumption in kernelized bandits that $h$ lies in the RKHS $\mathcal{H}$ induced by the NTK \citep{chowdhury2017kernelized,srinivas2009gaussian} (or, equivalently, that the RKHS norm $\norm{h}_{\mathcal{H}}$ of $h$ is upper-bounded by a constant) because $\sqrt{\mathbf{h}^{\top} \mathbf{H}^{-1} \mathbf{h}} \leq \norm{h}_{\mathcal{H}}$ \citep{zhou2020neural}.
Following the works of~\citet{zhang2020neural,zhou2020neural}, we define the effective dimension of $\mathbf{H}$ as $\widetilde{d}\triangleq \frac{\log\det (I + \mathbf{H}/\lambda)}{\log(1+TKN/\lambda)}$ with regularization parameter $\lambda>0$.
%We also define $\mathbf{H}_i$ as the $(TK) \times (TK)$-dimensional NTK matrix for agent $i$, which only consists of the $TK$ contexts observed by agent $i$.
Similarly, we define the effective dimension for agent $i$ as $\widetilde{d}_i\triangleq \frac{\log\det (I + \mathbf{H}_i/\lambda)}{\log(1+TK/\lambda)}$ and also define $\widetilde{d}_{\max}\triangleq \max_{i\in[N]}\widetilde{d}_i$.
Note that the effective dimension is related to the \emph{maximum information gain} $\gamma$ which is a commonly adopted notion in kernelized bandits~\citep{zhang2020neural}: $\widetilde{d} \leq  2\gamma_{TKN} / \log(1+TKN/\lambda)$ and $\widetilde{d}_i \leq 2\gamma_{TK} / \log(1+TK/\lambda),\forall i\in[N]$.
%\begin{equation}
%\widetilde{d} \leq  2\gamma_{TKN} / \log(1+TKN/\lambda), \widetilde{d}_i \leq 2\gamma_{TK} / \log(1+TK/\lambda),\forall i\in[N].
%\label{eq:eff:dim:upper:bounded:by:max:info:gain}
%\end{equation}
Consistent with the works on neural contextual bandits \citep{zhang2020neural,zhou2020neural}, our only assumption on the reward function $h$ is its boundedness: $|h(x)| \leq 1,\forall x\in\mathcal{X}$.
We also make the following assumptions for our theoretical analysis, all of which are mild and easily achievable, as discussed in \citep{zhang2020neural,zhou2020neural}:
\begin{assumption}
\label{assumption:main}
%We assume that 
There exists $\lambda_0 > 0$ s.t.~$\mathbf{H} \succeq \lambda_0 I$ and $\mathbf{H}_i \succeq \lambda_0 I, \forall i\in[N]$. 
%We also assume that 
Also, all contexts satisfy $\norm{x}_{2}=1$ and $[x]_{j}=[x]_{j+d/2}$, $\forall x\in\mathcal{X}_{t,i},\forall t\in[T],i\in[N]$.
\end{assumption}
\vspace{-2mm}
\section{Federated Neural-Upper Confidence Bound (\texttt{FN-UCB})}
\label{sec:fn_ucb}
\vspace{-2mm}

%Our \emph{Federated Neural-Upper Confidence Bound} (\texttt{FN-UCB}) algorithm is described in Algo.~\ref{algo:agent} (the agents' part) and Algo.~\ref{algo:server} (the central server's part).
Our \texttt{FN-UCB} algorithm is described in Algo.~\ref{algo:agent} (agents' part) and Algo.~\ref{algo:server} (central server's part).

%\vspace{-6mm}
%Let $\alpha \in [0,1]$.
%For technical reasons, we analyze a version of our algorithm where $\alpha > 0$ only for the first iteration after a sychronization round, i.e., we set $\alpha=0$ after the first iteration of a synchronization round.

% \textbf{Overview of \texttt{FN-UCB}.}
\vspace{-1.8mm}
\subsection{Overview of \texttt{FN-UCB} Algorithm}
\vspace{-1.8mm}
Before the beginning of the algorithm, we sample the initial parameters $\theta_0$ and share it with all agents (Sec.~\ref{sec:background}).
In each iteration $t\in[T]$, every agent $i\in[N]$ receives a set $\mathcal{X}_{t,i}=\{x^k_{t,i}\}_{k\in[K]}$ of $K$ contexts (line 3 of Algo.~\ref{algo:agent}) and then uses a weighted combination of $\text{UCB}^{a}_{t,i}$ and $\text{UCB}^{b}_{t,i}$ to select a context $x_{t,i}\in\mathcal{X}_{t,i}$ to be queried (lines 4-7 of Algo.~\ref{algo:agent}).
Next, each agent $i$ observes a noisy output $y_{t,i}$ (line 8 of Algo.~\ref{algo:agent}) and then updates its local information (lines 9-10 of Algo.~\ref{algo:agent}).
After that, every agent checks if it has collected enough information since the last communication round (i.e., checks the criterion in line 11 of Algo.~\ref{algo:agent}); if so, it sends a synchronization signal to the central server (line 12 of Algo.~\ref{algo:agent}) who then tells all agents to start a communication round (line 2 of Algo.~\ref{algo:server}).
During a communication round, every agent $i$ uses its current history of local observations to train an NN (line 14 of Algo.~\ref{algo:agent}) and sends its updated local information to the central server (line 16 of Algo.~\ref{algo:agent}); the central server then aggregates these information from all agents (lines 4-5 of Algo.~\ref{algo:server}) and broadcasts the aggregated information back to all agents (line 6 of Algo.~\ref{algo:server}) to start the next iteration.
%After some iterations, a \emph{communication round} may be initiated, during which every agent sends its updated local information to the central server, who aggregates the received information and then broadcasts the aggregated information back to all agents to start the next iteration.
%Intuitively, a communication round is started whenever any agent has collected enough information since the last communication round, and all agents share the same parameter $D$ which is used to determine whether a communication round should be started (line 11 of Algo.~\ref{algo:agent}).
We refer to those iterations between two communication rounds as an \emph{epoch}.\footnote{The first (last) epoch is between a communication round and the beginning (end) of \texttt{FN-UCB} algorithm.} So, our \texttt{FN-UCB} algorithm consists of a number of epochs which are separated by communication rounds.
%We use $p$ to index different epochs, and use $P$ to represent the total number of epochs.
%Moreover, $t_p$ and $E_p$ represent the starting iteration and the length (i.e., the total number of iterations) of epoch $p$.

\begin{algorithm}[t]
\begin{algorithmic}[1]
\STATE \textbf{inputs:} $\lambda=1+2/T$, $\theta_0 \sim \text{init}(\cdot)$,
%$V_0 = \lambda I$, 
$W_{\text{sync}}=\mathbf{0}_{p_0\times p_0}$, $W_{\text{new},i}=\mathbf{0}_{p_0\times p_0}$, $B_{\text{sync}}=\mathbf{0}$, $B_{\text{new},i}=\mathbf{0}_{p_0}$, $\alpha=0$, $V^{\text{local}}_{t,i}=\lambda I$, $V_{\text{sync,NN}}^{-1}=\lambda^{-1} I$, $V_{\text{last}}=\lambda I$, $t_{\text{last}}=0$, $\theta_{\text{sync,NN}}=\theta_0$.
% , $D=\mathcal{O}({T}/({N \widetilde{d}}))\ .$
	\FOR{$t=1,2,\ldots, T$}
            \STATE Receive a set $\mathcal{X}_{t,i}=\{x^k_{t,i}\}_{k\in[K]}$ of $K$ contexts
		\STATE Compute $\overline{V}_{t,i} = \lambda I + W_{\text{sync}} + W_{\text{new},i}\ ,\ \  \overline{\theta}_{t,i} = \overline{V}_{t,i}^{-1} (B_{\text{sync}} + B_{\text{new},i})$
%		\IF{not the first iteration after a synchronization round}
%		\STATE $\alpha=0$
%		\ENDIF
		\STATE Compute $\text{UCB}^{a}_{t,i}(x)\triangleq \langle g(x;\theta_0)/\sqrt{m}, \overline{\theta}_{t,i} \rangle + \nu_{TKN} \sqrt{\lambda}\  \norm{g(x;\theta_0) / \sqrt{m}}_{\overline{V}_{t,i}^{-1}}$
		% \STATE If $\alpha_t \neq 0$, calculate $\text{UCB}^{b}_{t,i}(x)=f(x;\theta_{\text{sync,NN}}) + \nu_{TK} \sqrt{\lambda} \norm{g(x;\theta_0) / \sqrt{m}}_{V_{\text{sync,NN}}^{-1}}$
		\STATE If $\alpha \neq 0$, compute $\text{UCB}^{b}_{t,i}(x)\triangleq f(x;\theta_{\text{sync,NN}}) + \nu_{TK} \sqrt{\lambda}{N}^{-1}\sum^N_{j=1} \norm{g(x;\theta_0) / \sqrt{m}}_{(V^{\text{local}}_{j})^{-1}}$\vspace{-2.5mm}
		\STATE Select $x_{t,i} \triangleq {\arg\max}_{x\in\mathcal{X}_{t,i}} (1- \alpha)\  \text{UCB}^{a}_{t,i}(x) + \alpha\  \text{UCB}^{b}_{t,i}(x)$
		\STATE Query $x_{t,i}$ to observe $y_{t,i}$
		\STATE Update $W_{\text{new},i} \leftarrow W_{\text{new},i} + g(x_{t,i};\theta_0)\  g(x_{t,i};\theta_0)^{\top} / m  ,\ B_{\text{new},i} \leftarrow B_{\text{new},i} + y_{t,i}\  g(x_{t,i};\theta_0) / \sqrt{m}$%\vspace{-2.5mm}
		\STATE Update $V^{\text{local}}_{t,i} = V^{\text{local}}_{t-1,i} + g(x_{t,i};\theta_0)\  g(x_{t,i};\theta_0)^{\top} / m$
%		\STATE $V_{t,i} = \lambda I + W_{\text{sync}} + W_{\text{new},i}$
%		\IF{$ (t-t_{\text{last}}) \log \frac{\text{det}(V_{t,i}) }{\text{det}V_{\text{last}} }  > D$}
		\IF{$ (t-t_{\text{last}}) \log\left( {\text{det}(\lambda I + W_{\text{sync}} + W_{\text{new},i}) }/{\text{det}(V_{\text{last}})}\right) > D$}
		\STATE Send a synchronisation signal to the central server to start a communication round
		\ENDIF
		\IF{a communication round is started}
		\STATE Train an NN with gradient descent using all agent $i$'s \emph{local} observations $\mathcal{D}_{t,i}=\{(x_{\tau,i}, y_{\tau,i})\}_{\tau \in[t]}$ based on initial parameters $\theta_0$, learning rate $\eta$, number $J$ of iterations, and~\Eqref{eq:nn:loss:function} as loss function to obtain $\theta^{i}_{t}$
		\STATE Compute $\alpha_{t,i} = \widetilde{\sigma}^{\text{local}}_{t,i,\min} / \widetilde{\sigma}^{\text{local}}_{t,i,\max}$ (Sec.~\ref{subsec:weight:two:ucbs})
		\STATE \textbf{send} $\{\ W_{\text{new},i},\  B_{\text{new},i},\ \theta^{i}_{t},\  \alpha_{t,i},\  (V^{\text{local}}_{t,i})^{-1}\ \}$ to the central server
		% \STATE \textbf{receive} $\{W_{\text{sync}}, B_{\text{sync}}, \theta_{\text{sync,NN}}, V_{\text{sync,NN}}^{-1}, \alpha_t\}$ from the central server
		\STATE \textbf{receive} $\{\  W_{\text{sync}},\ B_{\text{sync}},\  \theta_{\text{sync,NN}},\ \alpha,\  \{(V^{\text{local}}_{i})^{-1}\ \}_{i\in[N]}\ \}$ from the central server
		\STATE Set $V_{\text{last}} = W_{\text{sync}} + \lambda I\ ,\  t_{\text{last}}=t\ ,\  W_{\text{new},i}=\mathbf{0}_{p_0\times p_0}\ ,\  B_{\text{new},i}=\mathbf{0}$
		\ENDIF
    \ENDFOR
\end{algorithmic}
\caption{\texttt{FN-UCB} (Agent $i$)}
\label{algo:agent}
\end{algorithm}
% \vspace{-1.8mm}

\begin{algorithm}[t]
\begin{algorithmic}[1]
\IF{a synchronization signal is received from \emph{any agent}}
\STATE Send a signal to all agents to start a communication round
\ENDIF
	\STATE \textbf{receive} $\{\ W_{\text{new},i},\  B_{\text{new},i},\ \theta^{i}_{t}, \ \alpha_{t,i},\  (V^{\text{local}}_{t,i})^{-1}\  \}_{i\in[N]}$
	\STATE Compute $\theta_{\text{sync,NN}} = {N}^{-1}\sum^N_{i=1} \theta^{i}_{t}\ ,    
 % $V_{\text{sync,NN}}^{-1} = \frac{1}{N}\sum^{N}_{i=1} (V^{\text{local}}_{t,i})^{-1}$, 
  \ \ \alpha=\min_{i\in[N]} \alpha_{t,i}\ $; let $(V^{\text{local}}_{i})^{-1}=(V^{\text{local}}_{t,i})^{-1},\forall i\in[N]$
	\STATE Update $W_{\text{sync}} \leftarrow W_{\text{sync}} + \sum^{N}_{i=1}W_{\text{new},i}\ ,   \ \ B_{\text{sync}} \leftarrow B_{\text{sync}} + \sum^{N}_{i=1}B_{\text{new},i}$
	% \STATE Broadcast $\{W_{\text{sync}}, B_{\text{sync}}, \theta_{\text{sync,NN}}, V_{\text{sync,NN}}^{-1}, \alpha_t \}$ to all agents
	\STATE Broadcast $\{\ W_{\text{sync}},\  B_{\text{sync}},\  \theta_{\text{sync,NN}},\ \alpha,\  \{(V^{\text{local}}_{i})^{-1}\}_{i\in[N]}\ \}$ to all agents
\end{algorithmic}
\caption{Central Server}
\label{algo:server}
\end{algorithm}
% \vspace{-3mm}

Note that every agent $i$ only needs to train an NN in every communication round, i.e., only after the change in the log determinant of the covariance matrix of any agent exceeds a threshold $D$ (line 11 of Algo.~\ref{algo:agent}). This has the additional benefit of reducing the computational cost due to the training of NNs.
Interestingly, this is in a similar spirit as the adaptive batch size scheme in \citet{gu2021batched} which only retrains the NN in Neural UCB after the change in the determinant of the covariance matrix exceeds a threshold and is shown to only slightly degrade the performance of Neural UCB.

\vspace{-1.2mm}
\subsection{The Two Upper Confidence Bounds (UCBs)}
\label{subsec:two:ucbs}
\vspace{-1.2mm}
% \textbf{The Two UCBs.}
%sy: the discussion of exploration and exploitation
% We now introduce the details of $\text{UCB}^{a}_{t,i}$ and $\text{UCB}^{b}_{t,i}$ in the next two paragraphs.
%\textbf{The Two Upper Confidence Bounds (UCBs).}
%In every iteration $t$, agent $i$ selects the pulled arm $x_{t,i}$ by maximizing the weighted combination of two UCBs, denoted as $\text{UCB}^{a}_{t,i}(x)$ and $\text{UCB}^{b}_{t,i}(x)$ (line 6 of Algo.~\ref{algo:agent}). Next, we introduce both UCBs in more detail.

Firstly, $\text{UCB}^{a}_{t,i}$ can be interpreted as the Linear UCB policy~\citep{abbasi2011improved} using the neural tangent features $g(x;\theta_0)/\sqrt{m}$ as the input features. 
%In iteration $t$, denote the index of the current epoch as $p$, then the calculation of $\text{UCB}^{b}_{t,i}(x)$ makes use of two types of observations: (\emph{a}) the observations from \emph{all $N$ agents} before epoch $p$ (via $W_{\text{sync}}$ and $B_{\text{sync}}$, see line 3 of Algo.~\ref{algo:agent} and line 5 of Algo.~\ref{algo:server}), and (\emph{b}) the newly collected local observations of agent $i$ in epoch $p$ (via $W_{\text{new},i}$ and $B_{\text{new},i}$, see lines 3 and 8 of Algo.~\ref{algo:agent}).
In iteration $t$, let $p$ denote the index of the current epoch. Then, computing $\text{UCB}^{a}_{t,i}$ (line 5 of Algo.~\ref{algo:agent}) makes use of two types of information.
%The first type of information utilizes the observations from \emph{all $N$ agents before epoch $p$}, which are used in the calculation of $\text{UCB}^{a}_{t,i}$ via $W_{\text{sync}}$ and $B_{\text{sync}}$ (line 3 of Algo.~\ref{algo:agent}).
The first type of information, which \emph{uses the observations from all $N$ agents before epoch $p$},
% and hence helps \emph{accelerate exploration} via these additional observations, 
is used for computing $\text{UCB}^{a}_{t,i}$ via $W_{\text{sync}}$ and $B_{\text{sync}}$ (line 4 of Algo.~\ref{algo:agent}).
Specifically, as can be seen from line 5 of Algo.~\ref{algo:server}, $W_{\text{sync}}$ and $B_{\text{sync}}$ are computed 
%sy: on
by the central server by summing up the $W_{\text{new},i}$'s and $B_{\text{new},i}$'s from all agents (i.e., by aggregating the information from all agents)
%sy: move it forward?
where $W_{\text{new},i}$ and $B_{\text{new},i}$ are computed using the local observations of agent $i$ (line 9 of Algo.~\ref{algo:agent}).
The second type of information used by $\text{UCB}^{a}_{t,i}$ (via $W_{\text{new},i}$ and $B_{\text{new},i}$ utilized in line 4 of Algo.~\ref{algo:agent}) exploits the newly collected local observations of agent $i$ \emph{in epoch $p$}.
%As a result, $\text{UCB}^{a}_{t,i}(x)$ allows us to improve over standard sequential contextual bandit algorithms by utilizing the observations from all agents via the federated setting, without requiring the agents to share their raw observations.
As a result, $\text{UCB}^{a}_{t,i}$ allows us to utilize the observations from all agents via the federated setting for accelerated exploration  without requiring the agents to share their raw observations.
$\text{UCB}^{a}_{t,i}$ is computed with 
%The parameter $\nu_{TKN}$ used for computing  is defined as 
the defined parameter $\nu_{TKN} \triangleq B + R [2(\log(3/\delta) + 1) + \widetilde{d} \log(1+TKN/\lambda)]^{1/2}$ where $\delta\in(0,1)$.

Secondly, $\text{UCB}^{b}_{t,i}$ 
%sy: aim to (to show the purpose of such a design, corresponds to the intro and abstract?)
leverages the federated setting to improve the quality of NN for reward prediction (to achieve better exploitation) in a similar way to FedAvg, i.e., by averaging the parameters of the NNs trained by all agents using their local observations \citep{mcmahan2016communication}.
Specifically, when a communication round is started, every agent $i\in[N]$ uses its local observations $\mathcal{D}_{t,i}\triangleq \{(x_{\tau,i}, y_{\tau,i})\}_{\tau \in[t]}$ to train an NN (line 14 of Algo.~\ref{algo:agent}). 
It uses initial parameters $\theta_0$ (i.e., shared among all agents (Sec.~\ref{sec:background})) and trains the NN using gradient descent with learning rate  $\eta$ for $J$ training iterations (see Theorem~\ref{theorem:regret} for the choices of $\eta$ and $J$)
% in our theoretical analysis) 
to minimize the following loss function:
%When agent $i$ trains its neural network using its local observations $\mathcal{D}_{t,i}=\{(x_{\tau,i}, y_{\tau,i})\}_{\tau \in[t]}$, it uses $\theta_0$ as the initial parameters (which is shared among all agents $i\in[N]$) and trains the NN with gradient descent (using step size $\eta$ and number of iterations $J$) to obtain $\theta^{i}_{t}$, where the loss function is
\begin{equation}
\textstyle\mathcal{L}_{t,i}(\theta) \triangleq 0.5\sum^t_{\tau=1} ( f(x_{\tau,i};\theta) - y_{\tau,i} )^2  + 0.5 m \lambda \norm{\theta-\theta_0}_{2}^2 \ .
\label{eq:nn:loss:function}
\end{equation}
%For agent $i$, the NN parameters obtained after the training is denoted as $\theta^{(i)}_t$.
The resulting NN parameters $\theta^{i}_t$'s from all $N$ agents are sent to the central server (line 16 of Algo.~\ref{algo:agent}) 
% (line 16 of Algo.~\ref{algo:agent}), 
who averages them (line 4 of Algo.~\ref{algo:server}) and broadcasts the aggregated $\theta_{\text{sync,NN}} \triangleq {N}^{-1}\sum^N_{i=1} \theta^{i}_{t}$ back to all agents to be used in the next epoch.
In addition, 
%to ensure that $\text{UCB}^{b}_{t,i}(x)$ is a valid (high-probability) upper bound on the reward function $h$, 
to compute the second term of $\text{UCB}^{b}_{t,i}$,
every agent needs to compute the matrix $V^{\text{local}}_{t,i}$ using its local inputs (line 10 of Algo.~\ref{algo:agent}) and send its inverse to the central server (line 16 of Algo.~\ref{algo:agent}) during a communication round; 
% after that, the central server averages the received matrices to obtain $V_{\text{sync,NN}}^{-1} = \frac{1}{N}\sum^{N}_{i=1} (V^{\text{local}}_{t,i})^{-1}$ and broadcasts it back to all agents to be used in the second term of $\text{UCB}^{b}_{t,i}$ (line 5 of Algo.~\ref{algo:agent}).
after that, the central server broadcasts these matrices $\{(V^{\text{local}}_{i})^{-1} \}_{i\in[N]}$ received from each agent back to all agents to be used in the second term of $\text{UCB}^{b}_{t,i}$ (line 6 of Algo.~\ref{algo:agent}).
Refer to Sec.~\ref{subsec:proof:sketch} for a  detailed explanation on the validity of $\text{UCB}^{b}_{t,i}$ as a high-probability upper bound on $h$ (up to additive error terms).
%The parameter $\nu_{TK}$ used for computing 
$\text{UCB}^{b}_{t,i}$ is computed with the defined parameter
$\nu_{TK} \triangleq B + R [2(\log(3N/\delta) + 1) + \widetilde{d}_{\max}\log(1+TK/\lambda)]^{1/2}$.

% For both UCBs, unlike Neural UCB \citep{zhou2020neural} and Neural TS \citep{zhang2020neural} which use $\theta_t$ (the parameters of trained NNs) to calculate the exploration term (the second terms of $\text{UCB}^{a}_{t,i}$ and $\text{UCB}^{b}_{t,i}$), we instead use $\theta_0$. This is consistent with \citet{kassraie2021neural} who have shown that the use of $\theta_0$ gives accurate uncertainty estimation.

%\paragraph{The Weight between the Two UCBs.}
% \textbf{The Weight between the Two UCBs.}
\vspace{-1.8mm}
\subsection{Weight between the Two UCBs}
\label{subsec:weight:two:ucbs}
\vspace{-1.8mm}
Our choice of the weight $\alpha$ between the two UCBs,
%(line 6 of Algo.~\ref{algo:agent}) 
which naturally arises during our theoretical analysis (Sec.~\ref{sec:theoretical:analyis}),
%required by our theoretical analysis (Sec.~\ref{subsec:proof:sketch}) 
has an interesting interpretation in terms of the relative strengths of the two UCBs and the exploration-exploitation trade-off.
Specifically,  $\widetilde{\sigma}^{\text{local}}_{t,i}(x) \triangleq \sqrt{\lambda}\  \norm{g(x;\theta_0) / \sqrt{m}}_{(V^{\text{local}}_{t,i})^{-1}}$  intuitively represents our \emph{uncertainty} about the reward at $x$ after conditioning on the local observations of agent $i$ up to iteration $t$ \citep{kassraie2021neural}.\footnote{Formally, $\widetilde{\sigma}^{\text{local}}_{t,i}(x)$ is the Gaussian process posterior standard deviation at $x$ conditioned on the local observations of agent $i$ till iteration $t$ and computed using the kernel $\widetilde{k}(x,x')=g(x;\theta_0)^{\top} g(x';\theta_0) / m$.}
Next, $\widetilde{\sigma}^{\text{local}}_{t,i,\min} \triangleq \min_{x\in\mathcal{X}} \widetilde{\sigma}^{\text{local}}_{t,i}(x)$ and $\widetilde{\sigma}^{\text{local}}_{t,i,\max} \triangleq \max_{x\in\mathcal{X}} \widetilde{\sigma}^{\text{local}}_{t,i}(x)$  represent our smallest and largest uncertainties across the entire domain, respectively.
Then, we choose $\alpha\triangleq \min_{i\in[N]} \alpha_{t,i}$ (line 4 of Algo.~\ref{algo:server}) where $\alpha_{t,i} \triangleq \widetilde{\sigma}^{\text{local}}_{t,i,\min} / \widetilde{\sigma}^{\text{local}}_{t,i,\max}$ (line 15 of Algo.~\ref{algo:agent}).
In other words, $\alpha_{t,i}$ is \emph{the ratio between the smallest and largest uncertainty across the entire domain} for agent $i$, and $\alpha$ is the smallest such ratio $\alpha_{t,i}$ among all agents.
Therefore, $\alpha$ is expected to be generally \emph{increasing with the number of iterations/epochs}: $\widetilde{\sigma}^{\text{local}}_{t,i,\min}$ is already small after the first few iterations since the uncertainty at the queried contexts
%input locations (contexts) 
is very small; 
on the other hand, $\widetilde{\sigma}^{\text{local}}_{t,i,\max}$ is expected to be very large in early iterations and become smaller in later iterations only after 
%we have observed a large number of inputs (contexts) 
a large number of contexts has been queried
to sufficiently reduce the overall uncertainty in the entire domain.
This implies that we give \emph{more weight to $\text{UCB}^{a}_{t,i}$ in earlier iterations} and assign \emph{more weight to $\text{UCB}^{b}_{t,i}$ in later iterations}.
This, interestingly, turns out to have an intriguing practical interpretation: 
Relying more on $\text{UCB}^{a}_{t,i}$ in earlier iterations is reasonable because $\text{UCB}^{a}_{t,i}$ is able to utilize the observations from all agents to accelerate \emph{exploration} in the early stage (Sec.~\ref{subsec:two:ucbs});
% (see discussions of $\text{UCB}^{a}_{t,i}$ above); 
it is also sensible to give more emphasis to $\text{UCB}^{b}_{t,i}$ only in later iterations because the NN trained by every agent is only able to accurately model the reward function (for reliable \emph{exploitation}) after it has collected enough observations to train its NN.
In our practical implementation (Sec.~\ref{sec:experiments}), we will use the analysis here as an inspiration to design an increasing sequence of $\alpha$.

\vspace{-1.mm}
\subsection{Communication Cost}
\label{subsec:comm:cost}
\vspace{-1.8mm}
% In our \texttt{FN-UCB} algorithm (Algo.~\ref{algo:agent} and \ref{algo:server}), the central server needs to broadcast all $N$ $p_0\timesp_0$ matrices: $\{(V^{\text{local}}_{i})^{-1}\}_{i\in[N]}$ to the agents ($p_0$ is the total number of parameters of the NN), which may be undesirable when $N$ is large.
% To improve the communication efficiency of our \texttt{FN-UCB}, we slightly modify it to propose another variant of \texttt{FN-UCB} referred to as \texttt{FN-UCB} (\texttt{Less Comm.}), which differs from our original \texttt{FN-UCB} (Algo.~\ref{algo:agent} and Algo.~\ref{algo:server}) in two ways.

To achieve a better communication efficiency, we propose here a variant of our main \texttt{FN-UCB} algorithm called \texttt{FN-UCB} (\texttt{Less Comm.}) which differs from \texttt{FN-UCB} (Algos.~\ref{algo:agent} and~\ref{algo:server}) in two aspects.
\textbf{Firstly}, the central server averages the matrices $\{(V^{\text{local}}_{t,i})^{-1} \}_{i\in[N]}$ received from all agents 
% (line 3 of Algo.~\ref{algo:server}) 
to produce a single matrix $V_{\text{sync,NN}}^{-1} = {N}^{-1}\sum^{N}_{i=1} (V^{\text{local}}_{t,i})^{-1}$ and hence only broadcasts the single matrix $V_{\text{sync,NN}}^{-1}$ instead of all $N$ received matrices
$\{(V^{\text{local}}_{t,i})^{-1} \}_{i\in[N]}$ 
to all agents (see line 6 of Algo.~\ref{algo:server}).
\textbf{Secondly}, the $\text{UCB}^{b}_{t,i}$ of every agent $i$ (line 6 of Algo.~\ref{algo:agent}) is modified to use the matrix $V_{\text{sync,NN}}^{-1}$: $\text{UCB}^{b}_{t,i}(x)\triangleq f(x;\theta_{\text{sync,NN}}) + \nu_{TK} \sqrt{\lambda} \norm{g(x;\theta_0) / \sqrt{m}}_{V_{\text{sync,NN}}^{-1}}$.
% As a result, \texttt{FN-UCB} (\texttt{Less Comm.}) incurs significantly less communication cost than \texttt{FN-UCB}.
% Importantly, the modified $\text{UCB}^{b}_{t,i}$ is also a high-probability upper bound on the reward function $h$ (App.~\ref{app:proof:regret:fn:ucb:less:comm}), which ensures that \texttt{FN-UCB} (\texttt{Less Comm.}) is also equipped with a regret upper bound (Sec.~\ref{subsec:theoretical:results}).
% In addition, 
To further reduce the communication cost of both \texttt{FN-UCB} and \texttt{FN-UCB} (\texttt{Less Comm.}) especially when the NN is large (i.e., its total number $p_0$ of parameters is large), we can follow the practice of previous works~\citep{zhang2020neural,zhou2020neural}
%on neural contextual bandits~\cite{zhou2020neural,zhang2020neural} 
to diagonalize the $p_0 \times p_0$ matrices, i.e., by only keeping the diagonal elements of the matrices.
Specifically, we can diagonalize $W_{\text{new},i}$ (line 9 of Algo.~\ref{algo:agent}) and $V^{\text{local}}_{t,i}$ (line 10 of Algo.~\ref{algo:agent}), and let the central server aggregate only the diagonal elements of the corresponding matrices to obtain $W_{\text{sync}}$ and $V^{-1}_{\text{sync,NN}}$.
This reduces both the communication and computational costs.
% This not only improves the communication efficiency but also reduces the computational cost.
% (for \texttt{FN-UCB} (\texttt{Less Comm.})).
% We will also analyze the total number of required communication rounds by \texttt{FN-UCB}, as well as the variant \texttt{FN-UCB} (\texttt{Less Comm.}), in Sec.~\ref{subsec:theoretical:results}.
\vspace{-1mm}

As a result, during a communication round, the parameters that an agent sends to the central server include $\{W_{\text{new},i}, B_{\text{new},i}, \theta^{i}_{t}, \alpha_{t,i}, (V^{\text{local}}_{t,i})^{-1}\}$ (line 16 of Algo.~\ref{algo:agent}) which constitute $p_0 + p_0 + p_0 + 1 + p_0=\mathcal{O}(p_0)$ parameters and are the same for \texttt{FN-UCB} and \texttt{FN-UCB} (\texttt{Less Comm.}).
The parameters that the central server broadcasts to the agents include $\{W_{\text{sync}}, B_{\text{sync}}, \theta_{\text{sync,NN}}, \alpha, \{(V^{\text{local}}_{t,i})^{-1} \}_{i\in[N]}\}$ for \texttt{FN-UCB} (line 6 of Algo.~\ref{algo:server}) which amount to $p_0 + p_0 + p_0 + 1 + N p_0=\mathcal{O}(N p_0)$ parameters. Meanwhile, \texttt{FN-UCB} (\texttt{Less Comm.}) only needs to broadcast $\mathcal{O}(p_0)$ parameters because the $N$ matrices $\{(V^{\text{local}}_{t,i})^{-1} \}_{i\in[N]}$ are now replaced by a single matrix $V^{-1}_{\text{sync,NN}}$.
Therefore, the total number of exchanged parameters by \texttt{FN-UCB} (\texttt{Less Comm.}) is $\mathcal{O}(p_0)$ which is \emph{comparable to the number of exchanged parameters in standard FL for supervised learning} (e.g., FedAvg) where the parameters (or gradients) of the NN are exchanged \citep{mcmahan2016communication}.
%In addition, 
% As we will discuss in Sec.~\ref{subsec:theoretical:results}, although \texttt{FN-UCB} incurs more communication cost, its regret upper bound is tighter than that of \texttt{FN-UCB} (\texttt{Less Comm.}); meanwhile, their empirical performances are very similar (Sec.~\ref{sec:experiments}).
% We will also analyze the total number of required communication rounds by our algorithms in Sec.~\ref{subsec:theoretical:results}.
We will also analyze the total number of required communication rounds by \texttt{FN-UCB}, as well as by \texttt{FN-UCB} (\texttt{Less Comm.}), in Sec.~\ref{subsec:theoretical:results}.
% \vspace{-1mm}

As we will discuss in Sec.~\ref{subsec:theoretical:results}, the variant \texttt{FN-UCB} (\texttt{Less Comm.}) has a looser regret upper bound than our main \texttt{FN-UCB} algorithm (Algos.~\ref{algo:agent} and~\ref{algo:server}).
However, in practice, \texttt{FN-UCB} (\texttt{Less Comm.}) is recommended over \texttt{FN-UCB} because it achieves a very similar empirical performance as \texttt{FN-UCB} (which we have verified in Sec.~\ref{subsec:synthetic:experiments}) and yet incurs less communication cost.

\vspace{-1.7mm}
\section{Theoretical Analysis}
\label{sec:theoretical:analyis}
\vspace{-1.7mm}

\subsection{Theoretical Results}
\label{subsec:theoretical:results}
\vspace{-1.2mm}
%\subsubsection{Upper Bound on The Cumulative Regret}
%\textbf{Upper Bound on The Cumulative Regret.}
\textbf{Regret Upper Bound.}
For simplicity, we analyze the regret of a simpler version of our algorithm where we only choose the weight $\alpha$ using the method described in Sec.~\ref{subsec:weight:two:ucbs} in the first iteration after every communication round (i.e., first iteration of every epoch) and set $\alpha=0$ in all other iterations. 
Note that when communication occurs after each iteration (i.e., when $D$ is sufficiently small), this version coincides with our original \texttt{FN-UCB} described in Algos.~\ref{algo:agent} and~\ref{algo:server} (Sec.~\ref{sec:fn_ucb}).
% We also analyze the general algorithm which does not set $\alpha=0$ in any iteration and present the results in Appendix~\ref{app:extended:analysis:general:algo}, which requires an additional asssumption and introduces an additional multiplicative constant to the regret upper bound.
The regret upper bound of \texttt{FN-UCB} is given by the following result (proof in Appendix \ref{app:proof:regret}):
\vspace{-0.8mm}
\begin{theorem}
%[\textbf{Regret}]
\label{theorem:regret}
Let $\delta\in(0,1)$, $\lambda=1+2/T$, and $D=\mathcal{O}(T / (N \widetilde{d}))$.
Suppose that the NN width $m$ grows polynomially:
$m \geq \text{\emph{poly}}( \lambda, T, K, N, L, \log(1/\delta), 1/\lambda_0)$.
For the gradient descent training (line 14 of Algo.~\ref{algo:agent}), let $\eta=C_4(m\lambda + mTL)^{-1}$ for some constant $C_4>0$ and $J=\widetilde{O}\left(TL/(\lambda C_4) \right)$.
%($\widetilde{O}$ ignores all log factors).
Then, with probability of at least $1-\delta$,
%\[
%R_T = \widetilde{O}\Big(\widetilde{d} \sqrt{TN} + \widetilde{d}_{\max} N \sqrt{TN}\Big).
%\]
$
R_T = \widetilde{O}\big(\widetilde{d}\sqrt{TN} + \widetilde{d}_{\max} N \sqrt{T}\big)
$.
% $
% R_T = \widetilde{O}\big(\widetilde{d}\sqrt{TN} + \widetilde{d}_{\max} N \sqrt{TN}\big)
% $.
\end{theorem}
\vspace{-1.2mm}
Refer to Appendix~\ref{app:conditions:on:m} for the detailed conditions on the NN width $m$ as well as the learning rate $\eta$ and  number $J$ of iterations for the gradient descent training (line 14 of Algo.~\ref{algo:agent}).
Intuitively, the \emph{effective dimension} $\widetilde{d}$ measures the actual underlying dimension of the set of all $TKN$ contexts for all agents~\citep{zhang2020neural}, and $\widetilde{d}_{\max} \triangleq \max_{i\in[N]}\widetilde{d}_i$ is the maximum among the underlying dimensions of the set of $TK$ contexts for each of the $N$ agents.
%$\widetilde{d}_{\max}$ is defined as $\widetilde{d}_{\max}=\max_{i\in[N]}\widetilde{d}_i$ in which $\widetilde{d}_i$ measures the underlying dimension of the set of $TK$ observed contexts of agent $i$.
\citet{zhang2020neural} showed that if all contexts lie in a $d'$-dimensional subspace of the RKHS induced by the NTK, then the effective dimension of these contexts can be upper-bounded by the constant $d'$.
% In this case, our regret upper bound becomes $\widetilde{\mathcal{O}}(\sqrt{T}N^{3/2})$ which is sub-linear in $T$.

%The two terms in the regret upper bound in Theorem~\ref{theorem:regret} arise due to the use of $\text{UCB}^{b}_{t,i}$ and $\text{UCB}^{a}_{t,i}$, respectively. 
\textbf{The first term} $\widetilde{d}\sqrt{TN}$ in the regret upper bound (Theorem \ref{theorem:regret}) arises due to 
% the use of 
$\text{UCB}^{a}_{t,i}$ and reflects the benefit of the federated setting.
In particular, this term matches the regret upper bound of standard Neural UCB \citep{zhou2020neural} running for $TN$ iterations and so, the average regret $\widetilde{d}\sqrt{T/N}$ across all agents decreases with a larger number $N$ of agents.
\textbf{The second term} $\widetilde{d}_{\max} N \sqrt{T}$ results from 
% the use of 
$\text{UCB}^{b}_{t,i}$ which involves two major components of our algorithm: the use of NNs for reward prediction 
%(line 5 of Algo.~\ref{algo:agent}) 
and the aggregation of the NN parameters.
% (line 4 of Algo.~\ref{algo:server}). 
Although 
%sy: the second term may
not reflecting the benefit of a larger $N$ in the regret bound,
% (in terms of the average regret across all agents), 
both components are important to our algorithm.
Firstly, the use of NNs for reward prediction is a crucial component in neural contextual bandits in order to exploit the strong representation power of NNs.
This is similar in spirit to the works on neural contextual bandits~\citep{zhang2020neural,zhou2020neural} in which the use of NNs for reward prediction does not improve the regret upper bound (compared with using the linear prediction given by the first term of $\text{UCB}^{a}_{t,i}$) and yet significantly improves the practical performance.
Secondly, the aggregation of the NN parameters 
%(line $4$ of Algo.~\ref{algo:server}) 
is also important for the performance of our \texttt{FN-UCB} since it allows us to exploit the federated setting in a similar way to FL for supervised learning which has been repeatedly
%sy: widely
shown to improve the 
%practical 
performance \citep{kairouz2019advances}.
% Moreover, 
We have also empirically verified (Sec.~\ref{subsec:synthetic:experiments}) that both components (i.e., the use of NNs for reward prediction 
%(i.e., the use of $\text{UCB}^{a}_{t,i}$) 
and the aggregation of NN parameters) are important to the practical performance of our algorithm.
The work of \citet{huang2021fl} has leveraged the NTK to analyze the convergence of FedAvg for supervised learning \citep{mcmahan2016communication}  which also averages the NN parameters in a similar way to our algorithm.
% the FL algorithm of FedAvg which, similar to our algorithm, also relies on the averaging of NN parameters. 
Note that their convergence results
% , consistent with many previous works on convergence of FL algorithms, 
also do not improve with a larger number $N$ of agents but in fact become worse with a larger $N$.
% Therefore, 

Of note, in the single-agent setting where $N=1$, we have that $\widetilde{d}=\widetilde{d}_{\max}$ (Sec.~\ref{sec:background}). Therefore, our regret upper bound from 
Theorem \ref{theorem:regret}
% \eqref{eq:final:regret:upper:bound:added:rebuttal} 
reduces to $R_T=\widetilde{O}(\widetilde{d}\sqrt{T})$, which, interestingly, matches the regret upper bounds of standard neural bandit algorithms including
Neural UCB \citep{zhou2020neural} and Neural TS \citep{zhang2020neural}.
We also prove (App.~\ref{app:proof:regret:fn:ucb:less:comm}) that \texttt{FN-UCB} (\texttt{Less Comm.}), which is a variant of our \texttt{FN-UCB} with a better communication efficiency (Sec.~\ref{subsec:comm:cost}), enjoys a regret upper bound of $R_T = \widetilde{O}(\widetilde{d}\sqrt{TN} + \widetilde{d}_{\max} N \sqrt{TN})$, whose second term is worse than that of \texttt{FN-UCB} (Theorem \ref{theorem:regret}) by a factor of $\sqrt{N}$.
In addition, we have also analyzed our general algorithm which \emph{does not set $\alpha=0$ in any iteration} (results and analysis in Appendix~\ref{app:extended:analysis:general:algo}), which requires an additional assumption and only introduces an additional multiplicative constant to the regret bound.

\textbf{Communication Complexity.}
% Next, 
The following result (proof in App.~\ref{sec:communication:complexity}) gives a theoretical guarantee on the communication complexity of \texttt{FN-UCB}, including its variant \texttt{FN-UCB} (\texttt{Less Comm.}):
% Next, the following result (proof in App.~\ref{sec:communication:complexity}) gives an upper bound on the total number of communication rounds of our \texttt{FN-UCB} algorithm, including its variant  \texttt{FN-UCB} (\texttt{Less Comm.}):
%\subsection{Communication Complexity}
%\label{subsec:communication:complexity}
\begin{theorem}
\label{theorem:communication}
With the same parameters as Theorem \ref{theorem:regret},
if the NN width $m$ satisfies $m \geq \text{\emph{poly}}(T, K, N, L, \log(1/\delta))$, then with probability of at least $1-\delta$, the total number of communication rounds for \texttt{FN-UCB} satisfies
$C_T = \widetilde{\mathcal{O}}(\widetilde{d}\sqrt{N})$.
\end{theorem}
\vspace{-2mm}
The specific condition on $m$ required by Theorem~\ref{theorem:communication} corresponds to condition 1 listed in App.~\ref{app:conditions:on:m} (see App.~\ref{sec:communication:complexity} for details) which
%Therefore, the condition on $m$ needed by Theorem~\ref{theorem:communication} 
is a subset of the conditions required by Theorem~\ref{theorem:regret}.
Following the same discussion on the effective dimension $\widetilde{d}$ presented above, if all contexts lie in a $d'$-dimensional subspace of the RKHS induced by the NTK, then $\widetilde{d}$ can be upper-bounded by the constant $d'$, consequently leading to a communication complexity of $C_T = \widetilde{\mathcal{O}}(\sqrt{N})$.
% Moreover, without this assumption on the contexts, the connection between $\widetilde{d}$ and $\gamma_{TKN}$ allows us to derive a worst-case communication complexity of $C_T = \widetilde{\mathcal{O}}( \gamma_{TKN}\sqrt{N} )=\widetilde{O}(T^{\frac{d-1}{d}} K^{\frac{d-1}{d}} N^{\frac{3d-2}{d2}})$, which is still sub-linear in $T$.

%$R_T = \widetilde{O}( \gamma_{TKN} \sqrt{TN} + \gamma_{TK} N^{3/2} \sqrt{T} )=\widetilde{O}(K^{\frac{(d-1)}{d}}T^{\frac{3d-2}{2d}} N^{3/2} )$.
\vspace{-1.5mm}
\subsection{Proof Sketch}
\label{subsec:proof:sketch}
\vspace{-1.5mm}
We give a brief sketch of our regret analysis for Theorem \ref{theorem:regret} (detailed proof in Appendix \ref{app:proof:regret}).
To begin with, we need to prove that both $\text{UCB}^{a}_{t,i}$ and $\text{UCB}^{b}_{t,i}$ are valid high-probability upper bounds on the reward function $h$ (App.~\ref{subsec:validity:of:ucbs}) given that the conditions on $m$, $\eta$, and $J$ in App.~\ref{app:conditions:on:m} are satisfied.

Since $\text{UCB}^{a}_{t,i}$ can be viewed as 
% the 
Linear UCB 
% policy~
% \citep{abbasi2011improved} 
using the neural tangent features $g(x;\theta_0)/\sqrt{m}$ as the input features (Sec.~\ref{sec:fn_ucb}), its validity as a high-probability upper bound on $h$ can be proven following similar steps as that of standard linear and kernelized bandits \citep{chowdhury2017kernelized} (see Lemma~\ref{lemma:confidence:bound:ucb:2} in App.~\ref{subsec:validity:of:ucbs}).
Next, to prove that $\text{UCB}^{b}_{t,i}$ is also a high-probability upper bound on $h$ (up to additive error terms), let $\theta_{t,i}^{\text{local}} \triangleq (V^{\text{local}}_{t,i})^{-1}(\sum^{t}_{\tau=1} y_{\tau,i} g(x_{\tau,i};\theta_0) / \sqrt{m} )$ which is defined in the same way as $\overline{\theta}_{t,i}$ (line 4 of Algo.~\ref{algo:agent}) except that $\theta_{t,i}^{\text{local}}$ only uses the local observations of agent $i$.
\textbf{Firstly}, we show that $f(x;\theta_{\text{sync,NN}})$ (i.e., the NN prediction using the aggregated parameters) is close to $N^{-1}\sum^N_{i=1}\langle g(x;\theta_0)/\sqrt{m}, \theta_{t,i}^{\text{local}}\rangle$ which is the linear prediction using $\theta_{t,i}^{\text{local}}$ averaged over all agents.
This is achieved by showing that the linear approximation of the NN at 
% the initial parameters 
$\theta_0$ is close to both 
% of these two 
terms.
\textbf{Secondly}, we show that the absolute difference between the linear prediction $\langle g(x;\theta_0)/\sqrt{m}, \theta_{t,i}^{\text{local}}\rangle$ of agent $i$ and the reward function $h(x)$ can be upper-bounded by $\nu_{TK} \sqrt{\lambda} ||g(x;\theta_0) / \sqrt{m}||_{(V_{t,i}^{\text{local}})^{-1}}$. This can be done following similar steps as the proof for $\text{UCB}^{a}_{t,i}$ mentioned above.
\textbf{Thirdly}, using the averaged linear prediction $N^{-1}\sum^N_{i=1}\langle g(x;\theta_0)/\sqrt{m}, \theta_{t,i}^{\text{local}}\rangle$ as an intermediate term, the difference between $f(x;\theta_{\text{sync,NN}})$ and $h(x)$ can be upper-bounded.
This implies the validity of $\text{UCB}^{b}_{t,i}$ as a high-probability upper bound on $h$ (up to additive error terms 
which are small given the conditions on $m$, $\eta$, and $J$ presented in App.~\ref{app:conditions:on:m}), 
%which can be made small with the appropriate choices of $m$, $\eta$ and $J$ described in Appendix~\ref{app:conditions:on:m}), 
as formalized by Lemma~\ref{lemma:confidence:bound:ucb:1} in App.~\ref{subsec:validity:of:ucbs}.

Next, following similar footsteps as the analysis in \citet{wang2019distributed}, we separate all epochs into ``good'' epochs (intuitively, those epochs during which the amount of newly collected information from all agents is not too large) and ``bad'' epochs (details in App.~\ref{subsubsec:definition:good:bad:epochs}), and then separately upper-bound the regrets incurred in these two types of epochs.
For good epochs (App.~\ref{subsec:regret:good:epochs}), we 
are able to derive a tight upper bound on the regret $r_{t,i}=h(x_{t,i}^*)-h(x_{t,i})$ in each iteration $t$ by making use of 
%the definition of good epochs (i.e., the change of information between consecutive epochs is limited), 
the fact that the change of information in a good epoch
%between consecutive epochs 
is bounded,
% (due to the definition of good epochs),
and consequently obtain a tight upper bound on the total regrets in all good epochs.
%For bad epochs, we derive a relatively loose upper bound on $r_{t,i}$ and then apply the result from Appendix~\ref{subsubsec:definition:good:bad:epochs} which guarantees that the total number of bad epochs can be upper-bounded.
For bad epochs (App.~\ref{subsec:regret:bad:epochs}), we make use of the result from App.~\ref{subsubsec:definition:good:bad:epochs} which guarantees that the total number of bad epochs can be upper-bounded.
As a result, with an appropriate choice of $D=\mathcal{O}(T / (N \widetilde{d}))$, the growth rate of the total regret incurred in bad epochs is smaller than that in good epochs.
Lastly, the final regret upper bound 
%(Theorem~\ref{theorem:regret}) 
follows from adding up the total regrets from good and bad epochs (App.~\ref{subsec:regret:total}).

\vspace{-2.5mm}
\section{Experiments}
\label{sec:experiments}
\vspace{-2.5mm}
All figures in this section plot the average cumulative regret across all $N$ agents up to an iteration, which allows us to inspect the benefit that the federated setting brings to an agent (on average).
In all presented results, unless specified otherwise (by specifying a value of $D$), a communication round happens after each iteration.
All curves stand for the mean and standard error from 3 independent runs.
Some experimental details and results are deferred to App.~\ref{app:more:experimental:details} due to space limitation.
\vspace{-2mm}
\subsection{Synthetic Experiments}
\label{subsec:synthetic:experiments}
\vspace{-2mm}
We firstly use synthetic experiments to illustrate some interesting insights about our 
% \texttt{FN-UCB} 
algorithm. 
Similar to that of~\citet{zhou2020neural}, we adopt the synthetic functions of $h(x)=\cos(3 \langle a, x \rangle)$ and $h(x)=10(\langle a, x \rangle)^2$ which are referred to as the \texttt{cosine} and \texttt{square} functions, respectively.
We add a Gaussian observation noise with a standard deviation of $0.01$. 
The parameter $a$ is a $10$-dimensional vector randomly sampled from the unit hypersphere. In each iteration, every agent receives $K=4$ contexts (arms) which are randomly sampled from the unit hypersphere.
For fair comparisons, for all methods (including our \texttt{FN-UCB}, Neural UCB, and Neural TS), we use the same set of parameters of $\lambda=\nu_{TKN}=\nu_{TK}=0.1$ and use an NN with $1$ hidden layer and a width of $m=20$.
As suggested by our theoretical analysis (Sec.~\ref{subsec:weight:two:ucbs}),
% (Sec.~\ref{sec:fn_ucb}), 
we select an increasing sequence of $\alpha$ which is linearly increasing (to $1$) in the first 700 iterations, and let $\alpha=1$ afterwards.
To begin with, we compare 
% the performances of 
our main \texttt{FN-UCB} algorithm and its variant \texttt{FN-UCB} (\texttt{Less Comm.}) (Sec.~\ref{subsec:comm:cost}).
% using both \texttt{cosine} and \texttt{square}. 
The results (Figs.~\ref{fig:exp:synth:with:less:comm}a and~\ref{fig:exp:synth:with:less:comm}b in App.~\ref{app:more:experimental:details}) show that their empirical performances are very similar. So, for practical deployment, we recommend the use of \texttt{FN-UCB} (\texttt{Less Comm.}) as it is more communication-efficient and achieves a similar performance.
Accordingly, we will use the variant \texttt{FN-UCB} (\texttt{Less Comm.}) in all our subsequent experiments and refer to it as \texttt{FN-UCB} for simplicity.

Fig.~\ref{fig:exp:synth} presents the results.
Figs.~\ref{fig:exp:synth}a and~\ref{fig:exp:synth}b show that our \texttt{FN-UCB} with $N=1$ agent performs comparably with Neural UCB and Neural TS, and that the federation of a larger number $N$ of agents consistently improves the performance of our \texttt{FN-UCB}.
Note that the federation of $N=2$ agents can already provide significant improvements over non-federated algorithms. 
Fig.~\ref{fig:exp:synth}c gives an illustration of the importance of different components in our \texttt{FN-UCB}.
The red curve is obtained by removing $\text{UCB}^{b}_{t,i}$ (i.e., letting $\alpha=0$) and the green curve corresponds to removing $\text{UCB}^{a}_{t,i}$. 
The red curve shows that 
% despite achieving smaller regrets than the green curve initially due to its ability to exploit the observations from the other agents, 
relying solely on $\text{UCB}^{a}_{t,i}$ leads to significantly larger regrets in the long run due to its inability to utilize NNs to model the reward functions.
%sy: to model the reward functions accurately like NNs.
On the other hand, the green curve incurs larger regrets than the red curve initially; however, after more observations are collected (i.e., after the NNs are trained with enough data to accurately model the reward function), it quickly learns to achieve much smaller regrets.
These results 
% therefore
provide empirical justifications for our discussion on the weight between the two UCBs (Sec.~\ref{subsec:weight:two:ucbs}): It is reasonable to use an increasing sequence of $\alpha$ such that more weight is given to $\text{UCB}^{a}_{t,i}$ initially and then to $\text{UCB}^{b}_{t,i}$ later.
The yellow curve is obtained by removing the step of aggregating (i.e., averaging) the NN parameters (in line 4 of Algo.~\ref{algo:server}), i.e., when calculating $\text{UCB}^{b}_{t,i}$ (line 6 of Algo.~\ref{algo:agent}), 
% we use $\theta^{i}_{t}$ and $(V^{\text{local}}_{t,i})^{-1}$ to replace $\theta_{\text{sync,NN}}$ and $V_{\text{sync,NN}}^{-1}$.
we use $\theta^{i}_{t}$ to replace $\theta_{\text{sync,NN}}$.
The results show that the aggregation of the NN parameters significantly improves the performance of \texttt{FN-UCB} (i.e., the blue curve has much smaller regrets than the yellow one) and is hence an indispensable part of our \texttt{FN-UCB}.
% algorithm.
Lastly, Fig.~\ref{fig:exp:synth}d shows that more frequent communications (i.e., smaller values of $D$ which make it easier to initiate a communication round; see line 11 of Algo.~\ref{algo:agent}) lead to smaller regrets.
% and the performance gradually approaches that of single-agent \texttt{FN-UCB} ($N=1$) as $D$ increases.

\begin{figure}
\vspace{-6mm}
     \centering
     \begin{tabular}{cccc}
         \includegraphics[width=0.25\linewidth]{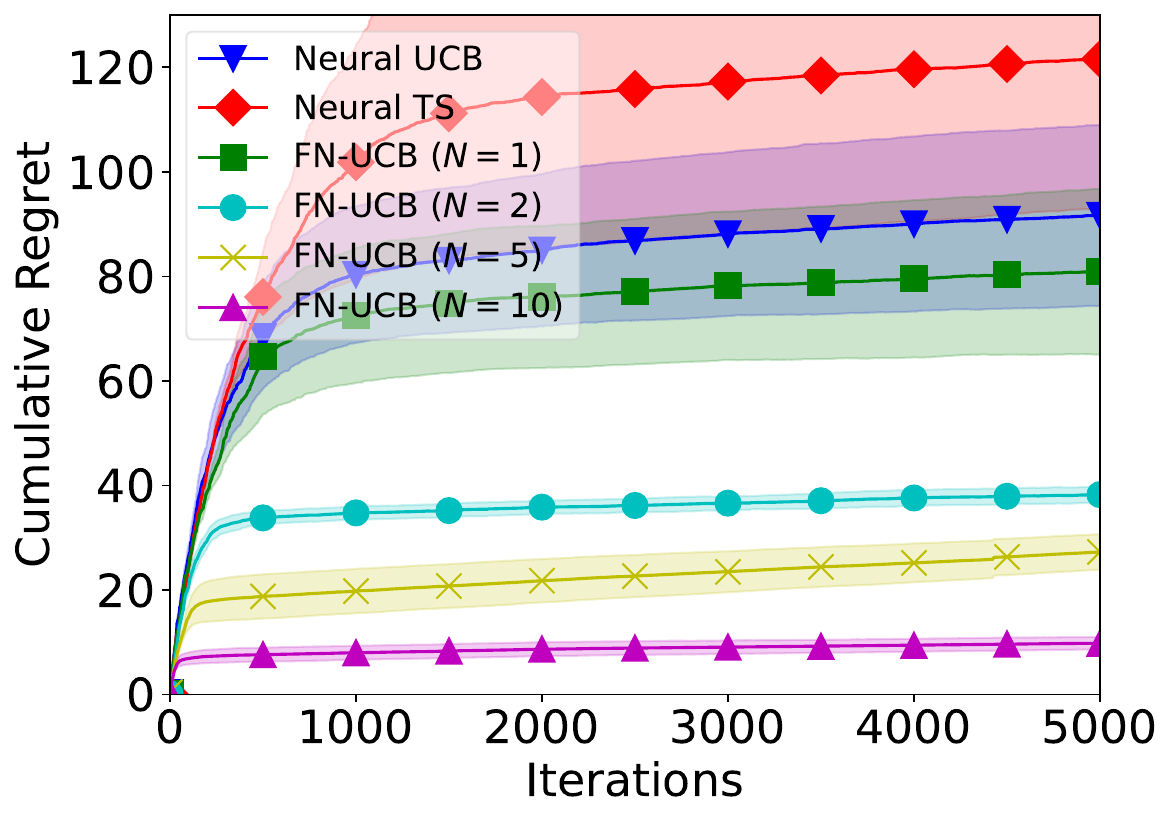} & \hspace{-6mm} 
         \includegraphics[width=0.25\linewidth]{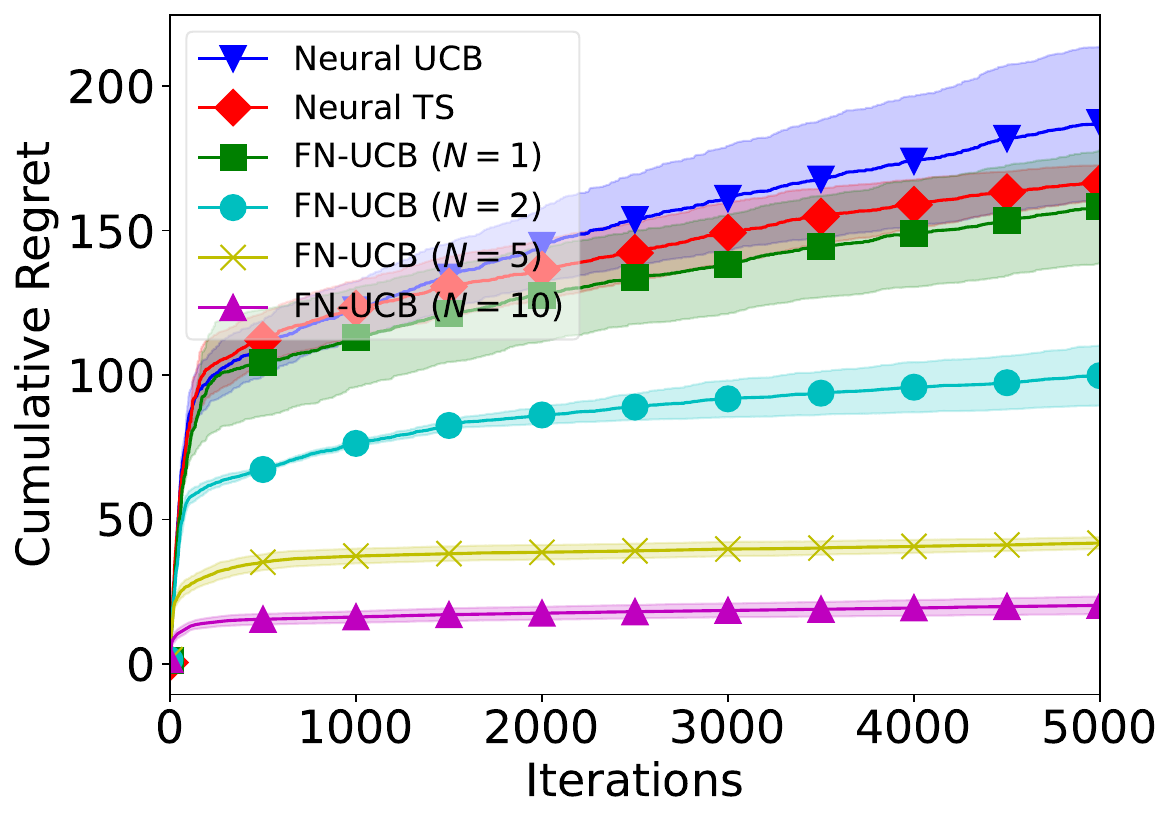}& \hspace{-6mm}
         \includegraphics[width=0.25\linewidth]{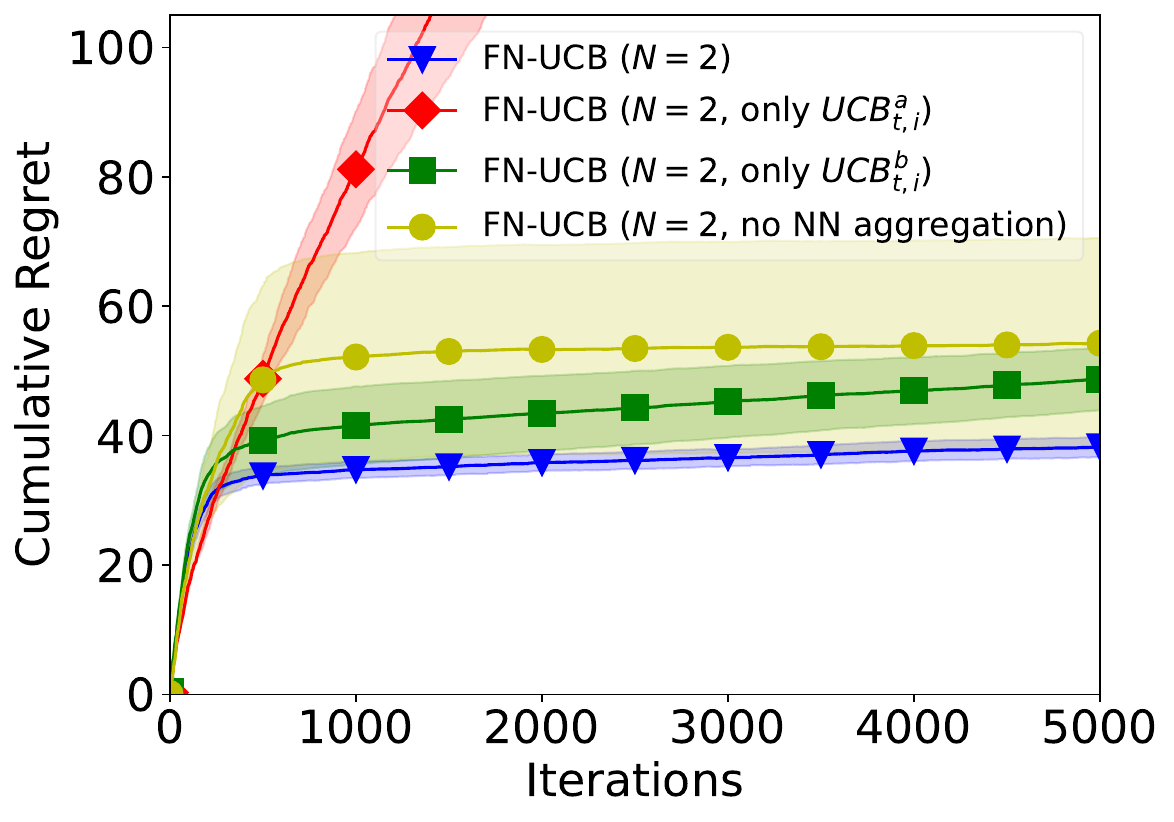}& \hspace{-6mm}
         \includegraphics[width=0.25\linewidth]{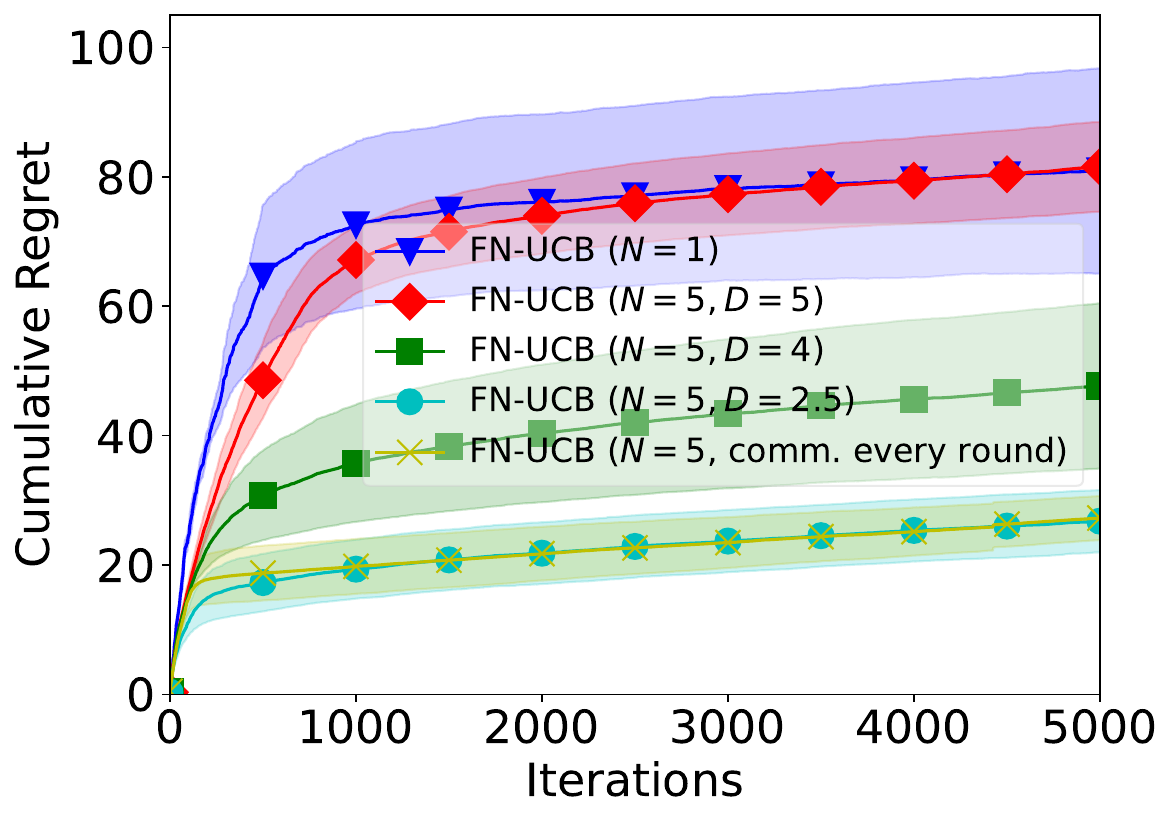}\\
         {\hspace{-2mm} (a) \texttt{cosine}} & {\hspace{-5mm} (b) \texttt{square}} & {\hspace{-6mm} (c) \texttt{cosine}} & {\hspace{-7mm} (d) \texttt{cosine}}
     \end{tabular}
\vspace{-2.8mm}
     \caption{
     Cumulative regret with varying number of agents for the (a) \texttt{cosine} function and (b) \texttt{square} function.
     (c) Illustration of the importance of different components of our \texttt{FN-UCB} algorithm (\texttt{cosine} function).
     (d) Performances with different values of $D$  (\texttt{cosine} function). The average number of rounds of communications are $348.0,380.0,456.7$ for $D=5, 4, 2.5$, respectively.
     }
     \label{fig:exp:synth}
\vspace{-1.4mm}
\end{figure}

\begin{figure}
\vspace{-1.4mm}
     \centering
     \begin{tabular}{cccc}
        \hspace{-4mm} \includegraphics[width=0.265\linewidth]{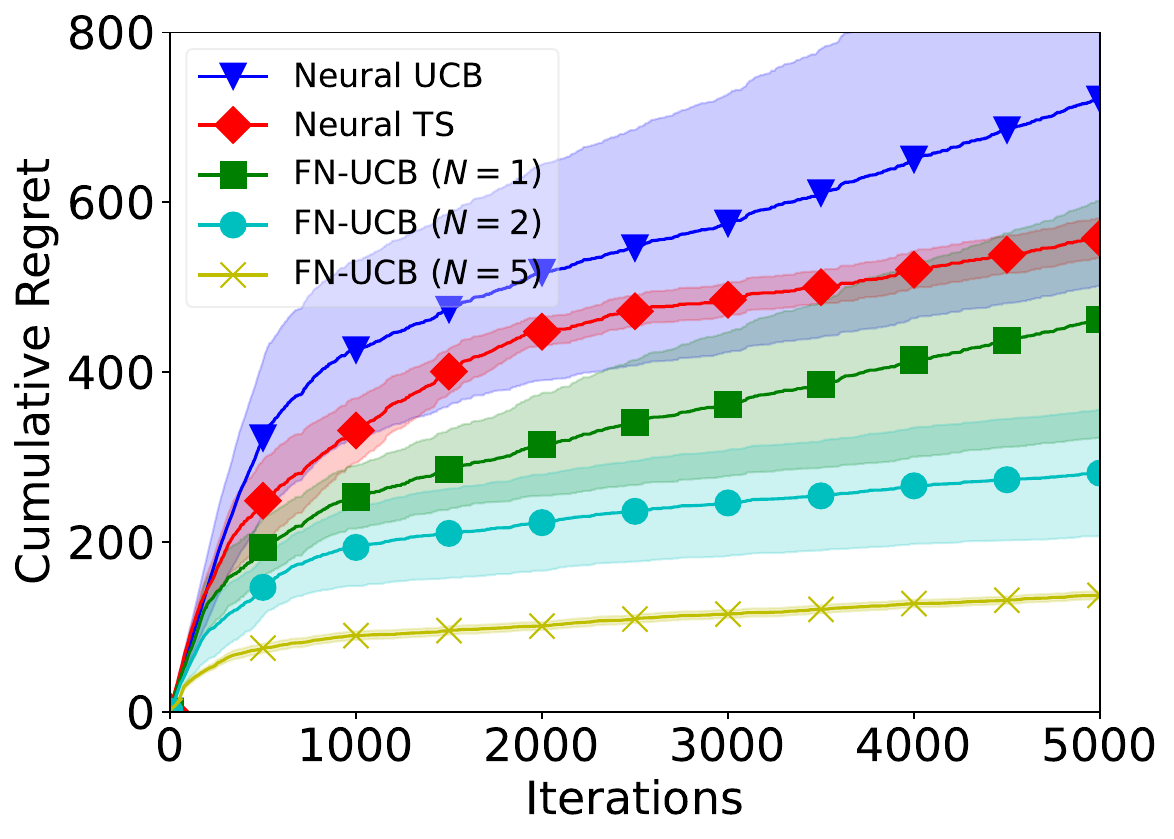} & \hspace{-7.5mm} 
         \includegraphics[width=0.265\linewidth]{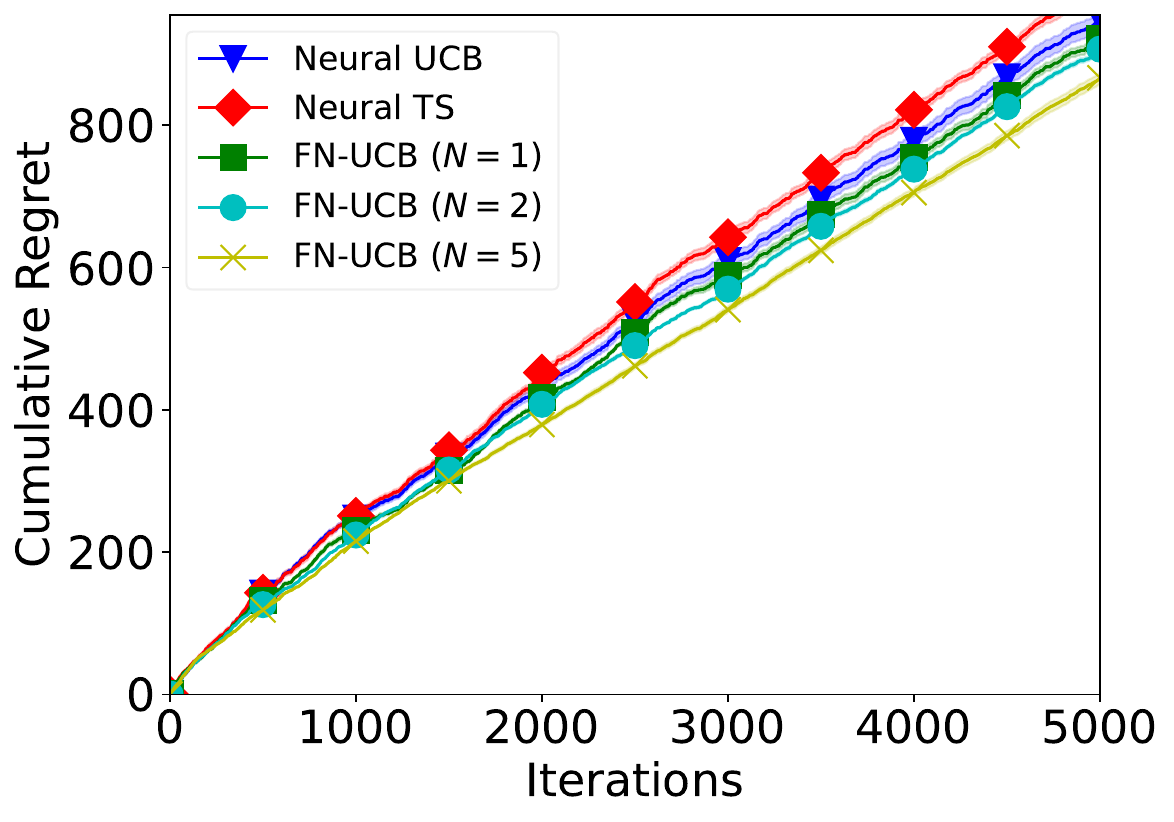}& \hspace{-7.5mm}
         \includegraphics[width=0.265\linewidth]{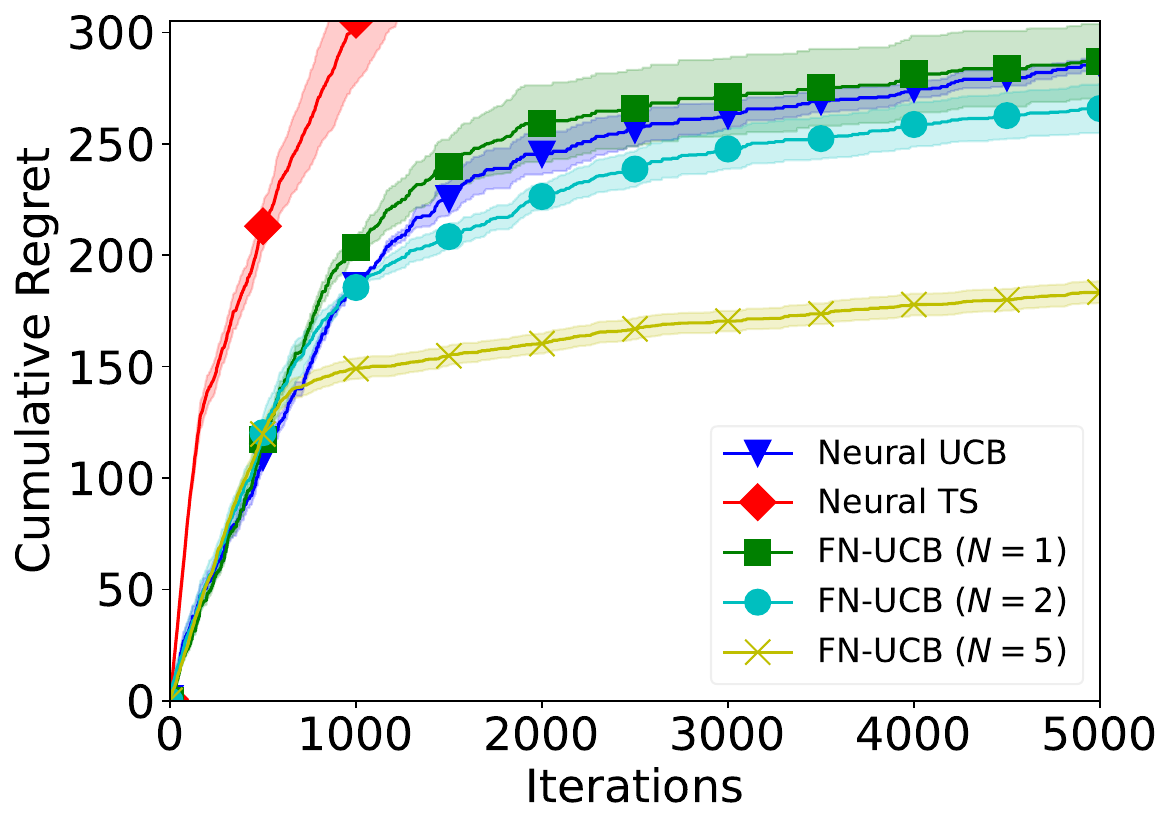}& \hspace{-7.5mm}
         \includegraphics[width=0.265\linewidth]{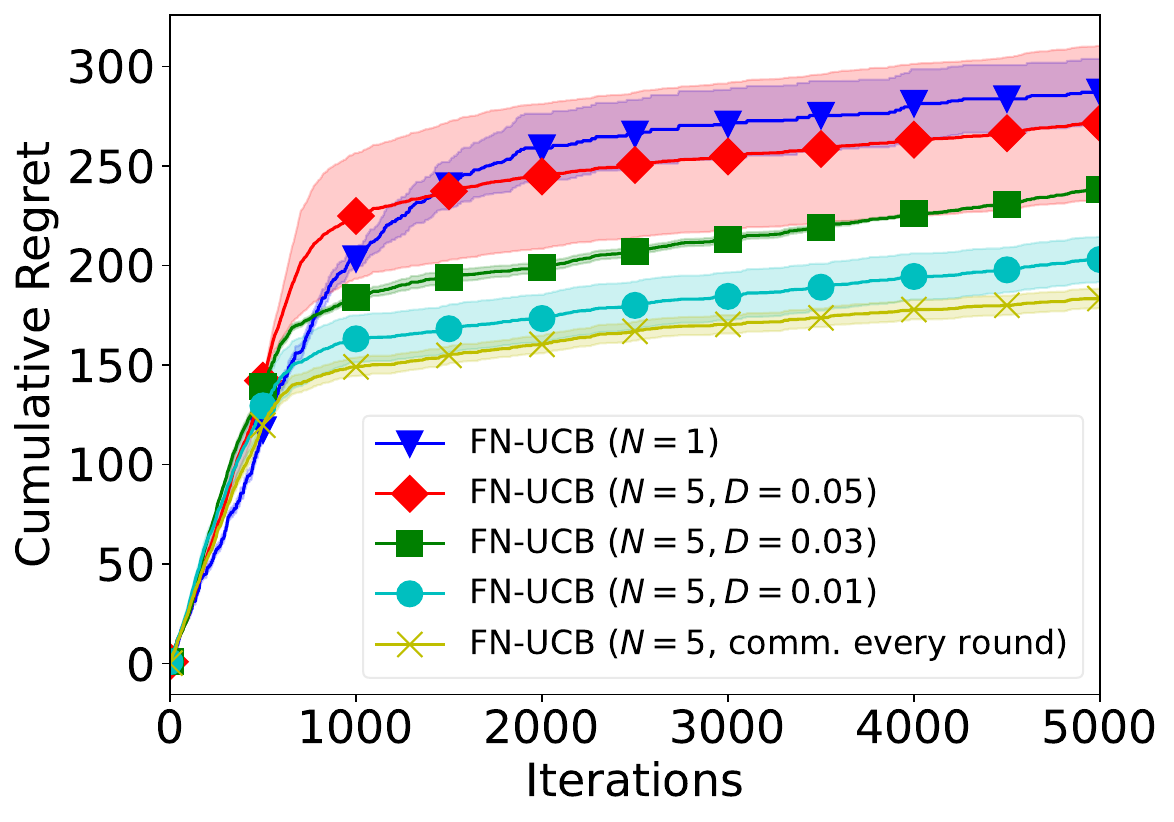}\\
         {\hspace{-3mm}(a) \texttt{shuttle}} & {\hspace{-5mm}(b) \texttt{magic telescope}} & {\hspace{-4mm}(c) \texttt{shuttle} (diag.)} & {\hspace{-3mm}(d) \texttt{shuttle} (diag.)}
     \end{tabular}
\vspace{-2.8mm}
     \caption{
     Results ($m=20$) for (a) \texttt{shuttle} and (b) \texttt{magic}.
     (c) Results for \texttt{shuttle} with diagonal approximation ($m=50$).
     (d) Results for \texttt{shuttle} with different values of $D$. The average number of communication rounds are $3850.7,4442.7,4906.3$ for $D=0.05, 0.03, 0.01$, respectively.
     }
     \label{fig:exp:real}
\vspace{-4mm}
\end{figure}

\vspace{-2mm}
\subsection{Real-world Experiments}
\label{subsec:real:experiments}
\vspace{-2mm}
We adopt the \texttt{shuttle} and \texttt{magic telescope} datasets from the UCI machine learning repository \citep{Dua:2019} and construct the 
%contextual bandit 
experiments following a widely used protocol in previous works \citep{li2010contextual,zhang2020neural,zhou2020neural}.
A $K$-class classification problem can be converted into a $K$-armed contextual bandit problem. 
In each iteration, an input $\mathbf{x}$ is randomly drawn from the dataset and is then used to construct $K$ 
%($K\times d$-dimensional) 
context feature vectors  $\mathbf{x}_1=[\mathbf{x};\mathbf{0}_d;\ldots;\mathbf{0}_d],\mathbf{x}_2=[\mathbf{0}_d;\mathbf{x};\ldots;\mathbf{0}_d],\ldots,\mathbf{x}_K=[\mathbf{0}_d;\ldots;\mathbf{0}_d;\mathbf{x}]$ which correspond to the $K$ classes.
The reward is $1$ if the arm with the correct class is pulled, and $0$ otherwise.
%We use an NN with $L=1$ hidden layer and $m=20$ hidden nodes for both datasets.
%We additionally use an NN ($L=1,m=50$) with diagonal approximations for the shuttle dataset, in order to verify the effectiveness of our \texttt{FN-UCB} with the diagonal approximation, which improves the communication and computational cost of our \texttt{FN-UCB} (Sec.~\ref{sec:fn_ucb}).
%The results (Fig.~\ref{fig:exp:real}) show that 
For fair comparisons, we use the same set of parameters of $\lambda=10$, $\nu_{TKN}=0.1$, and $\nu_{TK}=0.01$ for all methods.
Figs.~\ref{fig:exp:real}a and~\ref{fig:exp:real}b present the results for the two datasets ($1$ hidden layer, $m=20$) and show that our \texttt{FN-UCB} with $N=2$ agents consistently outperforms standard Neural UCB and Neural TS, and its performance also improves with the federation of more agents.
Fig.~\ref{fig:exp:real}c shows the results for \texttt{shuttle} when diagonal approximation (Sec.~\ref{subsec:comm:cost}) is applied to the NNs ($1$ hidden layer, $m=50$); the corresponding results are consistent with those in Fig.~\ref{fig:exp:real}a.\footnote{Since diagonalization increases the scale of the first term in $\text{UCB}^{a}_{t,i}$, we use a heuristic to rescale the values of this term for all contexts 
% to be between $[0,1]$, i.e., 
such that 
the max and min values (among all contexts) are $0$ and $1$ after rescaling.}
Moreover, the regrets in Fig.~\ref{fig:exp:real}c are in general smaller than those in Fig.~\ref{fig:exp:real}a.
This may suggest that in practice, a wider NN with diagonal approximation may be preferable to a narrower NN without diagonal approximation since it not only improves the performance but also reduces the computational and communication costs (Sec.~\ref{subsec:comm:cost}).
Fig.~\ref{fig:exp:real}d plots the regrets of \texttt{shuttle} (with diagonal approximation) for different values of $D$ and shows that more frequent communications lead to better performances and are hence consistent with 
% the results shown in 
that in Fig.~\ref{fig:exp:synth}d.
For completeness, we also compare their performance with that of linear and kernelized contextual bandit algorithms (for the experiments in both Secs.~\ref{subsec:synthetic:experiments} and \ref{subsec:real:experiments}), and the results (Fig.~\ref{fig:exp:with:linear:kernel}, App.~\ref{app:more:experimental:details}) show that they are outperformed by neural contextual bandit algorithms.

\vspace{-2mm}
\section{Conclusion}
\vspace{-2mm}
\label{sec:conclusion}
This paper describes the first federated neural contextual bandit algorithm called \texttt{FN-UCB}.
We use a weighted combination of two UCBs
%, which, respectively, helps accelerate exploration by using observations from other agents and uses an NN with aggregated parameters for reward prediction.
and the choice of this weight required by our theoretical analysis has an interesting interpretation emphasizing accelerated exploration initially and accurate prediction of the aggregated NN later.
We derive upper bounds on the regret and communication complexity of \texttt{FN-UCB}, and verify its effectiveness using empirical experiments.
Our algorithm is not equipped with privacy guarantees,
% and does not consider heterogeneous reward functions, 
which may be a potential limitation and will be tackled in future work.
% A potential negative societal impact is that our paper may further promote the use of NNs in more applications, which may increase energy consumption and contribute to the greenhouse effect.

%to incorporate differential privacy into our \texttt{FN-UCB}, potentially following similar techniques as \cite{dubey2020differentially}.

\newpage
\section*{Reproducibility Statement}
We have included the necessary details to ensure the reproducibility of our theoretical and empirical results.
For our theoretical results, we have stated all our assumptions in Sec.~\ref{sec:background}, added a proof sketch in Sec.~\ref{subsec:proof:sketch}, and included the complete proofs in App.~\ref{app:proof:regret} and App.~\ref{sec:communication:complexity}.
Our detailed experimental settings have been described in Sec.~\ref{subsec:synthetic:experiments}, Sec.~\ref{subsec:real:experiments}, and App.~\ref{app:more:experimental:details}.
Our code has been submitted as  supplementary material.

\subsubsection*{Acknowledgments}
This research/project is supported by A*STAR under its RIE$2020$ Advanced Manufacturing and Engineering (AME) Industry Alignment Fund – Pre Positioning (IAF-PP) (Award A$19$E$4$a$0101$).

\bibliography{iclr2023_conference}
\bibliographystyle{iclr2023_conference}

\newpage
\appendix

%\setcounter{page}{13}

%\section{Appendix}

% \vspace{-1.5mm}
\section{Related Works}
\label{sec:related:works}
% \vspace{-1.5mm}
\paragraph{Federated Bandits.}
Federated learning (FL) has received enormous attention in recent years \citep{kairouz2019advances,li2021survey,li2019federated,mcmahan2016communication}.
A number of recent works have extended the classic $K$-armed bandits (i.e., the arms are not associated with feature vectors) to the federated setting.
%The works of 
\citet{li2022privacy} and \citet{li2020federated} focused on incorporating privacy guarantees into federated $K$-armed bandits in both centralized and decentralized settings.
%The work of~
\citet{shi2021federated} proposed a setting where the goal is to minimize the regret of a global bandit whose reward of an arm is the average of the rewards of the corresponding arm from all agents, which was later extended by adding personalization such that every agent aims to maximize a weighted combination between the global and local rewards \citep{shi2021federatedwithpersonalization}.
Subsequent works on federated $K$-armed bandits have focused on other important aspects such as decentralized communication via the gossip algorithm~\citep{zhu2021federated}, the security aspect via cryptographic techniques~\citep{ciucanu2022samba}, uncoordinated exploration~\citep{yan2022federated}, and robustness against Byzantine attacks~\citep{demirel2022federated}.
Regarding federated linear contextual bandits, \citet{wang2019distributed} proposed a distributed linear contextual bandit algorithm which allows every agent to use the observations from the other agents by only exchanging the sufficient statistics to calculate the Linear UCB policy.
Subsequently,
\citet{dubey2020differentially} extended the method from \citet{wang2019distributed} to consider differential privacy and decentralized communication,
\citet{huang2021federated} considered a setting where every agent is associated with a unique context vector, 
\citet{li2021asynchronous} focused on asynchronous communication, 
and \citet{jadbabaie2022byzantine} considered the robustness against Byzantine attacks.
Federated kernelized/GP bandits (also named federated Bayesian optimization) have been explored by \citet{dai2020federated,dai2021differentially}, which focused on the practical problem of hyperparameter tuning in the federated setting.
The recent works of \citet{li2022communicationKernelized,li2022communicationGeneralized} have, respectively, focused on deriving communication-efficient algorithms for federated kernelized and generalized linear bandits.
In addition to federated bandits, other similar sequential decision-making problems have also been extended to the federated setting, such as federated reinforcement learning \citep{fan2021fault,zhuo2019federated} and federated hyperparameter tuning \citep{holly2021evaluation,khodak2021federated,zhou2021flora}.

\paragraph{Neural Bandits.}
Since the pioneering works of \citet{zhou2020neural} and \citet{zhang2020neural} which, respectively, introduced Neural UCB and Neural TS, a number of recent works have focused on different aspects of neural contextual bandits.
%For example, 
\citet{xu2020neural} reduced the computational cost of Neural UCB by using an NN as a feature extractor and applying Linear UCB only to the last layer of the learned NN, 
\citet{kassraie2021neural} analyzed 
%the growth rate of 
the maximum information gain of the NTK and hence derived no-regret 
%neural contextual bandit 
algorithms,
\citet{gu2021batched} focused on the batch setting in which the policy is only updated at a small number of time steps, 
\citet{nabati2021online} aimed to reduce the memory requirement of Neural UCB, 
\citet{lisicki2021empirical} performed an empirical investigation of neural bandit algorithms to verify their practical effectiveness,
\citet{ban2021ee} adopted a separate NN for exploration in neural contextual bandits,
\citet{ban2021convolutional} applied the convolutional NTK,
\citet{jia2021learning} used perturbed rewards to train the NN to remove the need for explicit exploration,
\citet{nguyen2021offline} incorporated offline policy learning into neural contextual bandits,
\citet{zhu2021pure} studied pure exploration in kernel and neural bandits,
\citet{kassraie2022graph} applied graph NNs in neural bandits to handle graph-structured data,
\citet{salgia2022provably} extended neural bandits beyond the ReLU activation to consider smoother activation functions,
and
\citet{dai2022sample} introduced a scalable batch Neural TS algorithm through sample-then-optimize optimization.

\section{More Background}
\label{app:more:background}
In this section, we give more details on some of the technical background mentioned in Sec.~\ref{sec:background}.
The details in this section all follow the works of \citet{zhang2020neural,zhou2020neural}, and we present them here for completeness.

\paragraph{Definition of the NN $f(x;\theta)$.}
Let $\mathbf{W}_1\in\mathbb{R}^{m\times d}$, $\mathbf{W}_l\in\mathbb{R}^{m\times m},\forall l=2,\ldots,L-1$, and $\mathbf{W}_L\in\mathbb{R}^{m\times1}$, then the NN $f(x;\theta)$ is defined as
\begin{equation*}
\begin{split}
&f_1=\mathbf{W}_1 x, \\
&f_l = \mathbf{W}_l \text{ReLU}(f_{l-1}), \forall l=2\ldots,L,\\
&f(x;\theta) = \sqrt{m} f_L,
\end{split}
\end{equation*}
in which $\text{ReLU}(z)=\max(z,0)$ denotes the rectified linear unit (ReLU) activation function and is applied to each element of $f_{l-1}$.
With this definition of the NN, $\theta$ denotes the collection of all parameters of the NN: $\theta = (\text{vec}(\mathbf{W}_1),\ldots,\text{vec}(\mathbf{W}_L)) \in \mathbb{R}^{p_0}$.

\paragraph{Details of the Initialization Scheme $\theta_0\sim\text{init}(\cdot)$.}
To obtain the initial parameters $\theta_0$, for each $l=1,\ldots,L-1$, let 
%$\mathbf{W}_l=(\mathbf{W},\mathbf{0};\mathbf{0},\mathbf{W})$ 
$\mathbf{W}_l=\left(
\begin{array}{cc} 
  \mathbf{W} & \mathbf{0} \\ 
  \mathbf{0} & \mathbf{W} 
\end{array} 
\right)$
where each entry of $\mathbf{W}$ is independently sampled from $\mathcal{N}(0, 4/m)$, and let $\mathbf{W}_L=(\mathbf{w}^{\top},-\mathbf{w}^{\top})$ where each entry of $\mathbf{w}$ is independently sampled from $\mathcal{N}(0,2/m)$.
This initialization scheme is the same as that used by the works of \citet{zhang2020neural,zhou2020neural}.

\paragraph{Definitions of the NTK Matrices $\mathbf{H}$ and $\mathbf{H}_i$'s.}
To simplify the exposition here, we use $\{x^j\}_{j=1,\ldots,TKN}$ to denote the set of all contexts from all iterations, all arms and \emph{all agents}: $\{x^k_{t,i}\}_{t\in[T],k\in[K],i\in[N]}$. We can then define
\begin{equation*}
\begin{split}
&\widetilde{\mathbf{H}}^{(1)}_{i,j}=\mathbf{\Sigma}^{(1)}_{i,j} =\langle x^i, x^j \rangle, \mathbf{A}^{(l)}_{i,j} = \left(
\begin{array}{cc} 
  \mathbf{\Sigma}^{(l)}_{i,i} & \mathbf{\Sigma}^{(l)}_{i,j} \\ 
  \mathbf{\Sigma}^{(l)}_{i,j} & \mathbf{\Sigma}^{(l)}_{j,j} 
\end{array} 
\right),\\
&\mathbf{\Sigma}^{(l+1)}_{i,j} = 2 \mathbb{E}_{(u,v)\sim \mathcal{N}(\mathbf{0},\mathbf{A}^{(l)}_{i,j})} \max(u,0) \max(v,0),\\
&\widetilde{\mathbf{H}}^{(l+1)}_{i,j} = 2\widetilde{\mathbf{H}}^{(l)}_{i,j}\mathbb{E}_{(u,v)\sim \mathcal{N}(\mathbf{0},\mathbf{A}^{(l)}_{i,j})} \mathbbm{1}(u>0) \mathbbm{1}(v>0) + \mathbf{\Sigma}^{(l+1)}_{i,j}.
\end{split}
\end{equation*}
With these definitions, the NTK matrix is defined as $\mathbf{H}=(\widetilde{\mathbf{H}}^{(L)} + \mathbf{\Sigma}^{(L)}) / 2$.
Similarly, $\mathbf{H}_i$ can be obtained in the same way by only using all contexts from agent $i$ in the definitions above, i.e., now we use $\{x^j\}_{j=1,\ldots,TK}$ to denote $\{x^k_{t,i}\}_{t\in[T],k\in[K]}$ and plug these $TK$ contexts into the definitions above to obtain $\mathbf{H}_i$.

%\begin{equation*}
%\left(
%\begin{array}{cc} 
%  A & B \\ 
%  O & C 
%\end{array} 
%\right)
%\end{equation*}

\section{Proof of Regret Upper Bound (Theorem~\ref{theorem:regret})}
\label{app:proof:regret}
We use $p$ to index different epochs and denote by $P$ the total number of epochs. 
We use $t_p$ to denote the first iteration of epoch $p$, and use $E_p$ to represent the length (i.e., number of iterations) of epoch $p$.
Throughout our theoretical analysis, we will denote different error probabilities as $\delta_1,\ldots,\delta_6$, which we will combine via a union bound at the end of the proof to ensure that our final regret upper bound holds with probability of at least $1-\delta$.

%We analyze two versions of our algorithm. The first version of our algorithm, named \emph{Federated Neural UCB} (FN-UCB), is the standard algorithm; the second version is a reduced variant of FN-UCB named FN-UCB-Lite, in which we use $\text{UCB}^{a}_{t,i}$ only in the first iteration after every synchronization round. 
%That is, for FN-UCB-Lite, we only let $\alpha_t>0$ for iterations $t_p,\forall p\in[P]$, i.e., only for the first iteration in every epoch.

\subsection{Conditions on the Width $m$ of the Neural Networks}
\label{app:conditions:on:m}
We list here the detailed conditions on the width $m$ of the NN that are needed by our theoretical analysis. These include two types of conditions, some of them (conditions 1-4) are required for our regret upper bound to hold (i.e., they are used during the proof to derive the regret upper bound), whereas the others (conditions 5-6) are 
used after the final regret upper bound is derived to ensure that the final regret upper bound is small (see App.~\ref{subsec:regret:total}).
%These conditions are similar to the conditions on $m$ required by the work of~\cite{zhang2020neural,zhou2020neural} (e.g., see condition 4.1 of \cite{zhang2020neural}).

When presenting our detailed proofs starting from the next subsection, we will refer to each of these conditions whenever they are used by the corresponding lemmas. Different lemmas may use different leading constants in their required condition (i.e., lower bound) on $m$, but here we use the same constant $C>0$ for all lower bounds for simplicity, which can be considered as simply taking the maximum among all these different leading constants for different lemmas.
\begin{enumerate}
\item $m\geq CT^6 K^6 N^6 L^6 \log(TKNL/\delta)$,
\item $m \geq C T^4 K^4 N^4 L^6 \log(T^2 K^2 N^2 L / \delta) / \lambda_0^4$,
\item $m \geq C \sqrt{\lambda}L^{-3/2} [\log(TKNL^2/\delta)]^{3/2}$,
\item $m (\log m)^{-3} \geq CTL^{12}\lambda^{-1} + C T^7 \lambda^{-8}L^{18}(\lambda+LT)^6 + CL^{21}T^7 \lambda^{-7}(1+\sqrt{T/\lambda})^6$.
%\end{enumerate}
%The following conditions are required to ensure that our resulting regret upper bound is small (see Appendix \ref{subsec:regret:total}):
%\begin{enumerate}
%\setcounter{enumi}{4}
\item $m(\log m)^{-3} \geq C T^{10}N^6 \lambda^{-4} L^{18}$,
\item $m(\log m)^{-3} \geq C T^{16} N^6 L^{24} \lambda^{-10}(1+\sqrt{T/\lambda})^6$.
\end{enumerate}
Some of these conditions above can be combined, but we leave them as separate conditions to make it easier to refer to the corresponding place in the proof where a particular condition is needed.

Furthermore, to achieve a small upper bound on the cumulative regret, we also need to place some conditions on the learning rate $\eta$ and number of iterations $J$ for the gradient descent training (line 14 of Algo.~\ref{algo:agent}). Specifically, we need to choose the learning rate as 
\begin{equation}
\eta=C_4(m\lambda + mTL)^{-1},
\label{eq:condition:on:eta}
\end{equation}
in which $C_4>0$ is an absolute constant such that $C_4 \leq 1+TL$, and choose 
\begin{equation}
J=\frac{1}{C_4}\left(1+\frac{TL}{\lambda}\right) \log\left( \frac{1}{3C_2 N} \sqrt{\frac{\lambda}{T^3 L}} \right)=\widetilde{O}\left(TL/(\lambda C_4) \right).
\label{eq:condition:on:J}
\end{equation}

\subsection{Definition of Good and Bad Epochs}
\label{subsubsec:definition:good:bad:epochs}
%We use $p$ to index different epochs and denote by $P$ the total number of epochs. 
Denote the matrix $V_{\text{last}}$ (see line 18 of Algo.~\ref{algo:agent}) after epoch $p$ as $V_p$.
As a result, the matrix $V_P$ is calculated using all selected inputs from all agents: $V_P=\sum^{T}_{t=1}\sum^N_{i=1}g(x_{t,i};\theta_0) g(x_{t,i};\theta_0)^{\top} /m + \lambda I$.
Define $V_{0}\triangleq \lambda I$.
Imagine that we have a hypothetical agent which chooses all $T \times N$ queries $\{x_{t,i}\}_{t\in[T],i\in[N]}$ sequentially in a round-robin fashion (i.e., the hypothetical agent chooses $x_{1,1},x_{1,2},\ldots,x_{2,1},x_{2,2},\ldots,x_{T,N}$), and denote the corresponding hypothetical covariance matrix as $\widetilde{V}_{t,i}=\sum^{t-1}_{\tau=1}\sum^{N}_{j=1}g(x_{\tau,j};\theta_0) g(x_{\tau,j};\theta_0)^{\top} / m + \sum^{i}_{j=1} g(x_{t,j};\theta_0) g(x_{t,j};\theta_0)^{\top} / m + \lambda I$.
We represent the indices of this hypothetical agent by $t'\in[TN]$ to distinguish it from our original multi-agent setting. 
%Define $\widetilde{V}_{0}\triangleq \lambda I$.
Define $J_{TN}\triangleq [g(x_{t'};\theta_0)]_{t'\in[TN]}$ which is a $p_0\times (TN)$ matrix,
%This immediately implies that $V_P = J_{TN} J_{TN}^{\top} / m+\lambda I$. 
and define $\mathbf{K}_{TN}\triangleq J_{TN}^{\top} J_{TN} / m$, which is a $(TN) \times (TN)$ matrix.
According to thes definitions, we have that
%Lemma B.7 of~\cite{zhang2020neural} tells us that 
\begin{lemma}[Lemma B.7 of~\cite{zhang2020neural}]
\label{lemma:bound:error:NTK:matrix}
Let $\delta_1\in(0,1)$.
If $m\geq CT^6 N^6 K^6 L^6 \log(TNKL/\delta_1)$, we have with probability of at least $1-\delta_1$ that
\[
\log\det(I + \lambda^{-1}\mathbf{K}_{t'}) \leq \log\det(I + \lambda^{-1}\mathbf{H}) + 1,\forall t'\in[TN].
\]
\end{lemma}

The condition on $m$ given in Lemma \ref{lemma:bound:error:NTK:matrix} corresponds to condition 1 listed in App.~\ref{app:conditions:on:m}, except that $\delta_1$ is used here instead of $\delta$ in App.~\ref{app:conditions:on:m}.
%The next inequality will be important to justify our definition of good and bad epochs.
Lemma \ref{lemma:bound:error:NTK:matrix} allows us to derive the following equation, which we will use (at the end of this section) to justify that the total number of "bad" epochs is not too large.
\begin{equation}
\begin{split}
\sum^{P-1}_{p=0} \log \frac{\text{det} V_{p+1}}{\text{det} V_p }  = \log \frac{\text{det} V_P}{\text{det} V_0} &\stackrel{(a)}{=} \log \frac{\text{det} \left(J_{TN} J_{TN}^{\top} /m + \lambda I \right)}{\text{det} V_0}\\
&= \log \frac{\text{det} \left(\lambda \left(\lambda^{-1} J_{TN} J_{TN}^{\top} /m + I \right) \right)}{\text{det} V_0}\\
&\stackrel{(b)}{=} \log \frac{\lambda^{p_0} \text{det}  \left(\lambda^{-1} J_{TN} J_{TN}^{\top} /m + I \right)}{\lambda^{p_0}}\\
&= \log\det \left(\lambda^{-1} J_{TN} J_{TN}^{\top} /m + I \right)\\
&\stackrel{(c)}{=} \log\det \left(\lambda^{-1} J_{TN}^{\top} J_{TN} / m + I \right)\\
&= \log\det \left(\lambda^{-1} \mathbf{K}_{TN} + I \right)\\
&\stackrel{(d)}{\leq} \log\det \left(\lambda^{-1} \mathbf{H} + I \right) + 1\\
&\stackrel{(e)}{=} \widetilde{d}\log(1+TKN/\lambda) + 1 \triangleq R'.
%&\leq 2\gamma_{TNK} + 1 \triangleq R'
\end{split}
\label{eq:sum:of:log:det}
\end{equation}
Step $(a)$ is because $V_P = J_{TN} J_{TN}^{\top} / m+\lambda I$ according to our definition of $J_{TN}$ above.
Step $(b)$ follows from our definition of $V_0=\lambda I$ above, as well as some standard properties of matrix determinant.
Step $(c)$ follows because: $\text{det}(\mathbf{A} \mathbf{A}^{\top} + I)=\text{det}(\mathbf{A}^{\top} \mathbf{A} + I)$.
Step $(d)$ has made use of Lemma~\ref{lemma:bound:error:NTK:matrix} above, which suggests that \eqref{eq:sum:of:log:det} holds with probability of at least $1-\delta_1$.
Step $(e)$ follows from the definition of $\widetilde{d}\triangleq \frac{\log\det (I + \mathbf{H}/\lambda)}{\log(1+TKN/\lambda)}$ (Sec.~\ref{sec:background}).
In the last step, we have defined $R' \triangleq \widetilde{d}\log(1+TKN/\lambda) + 1$.
%Furthermore, Lemma 3 of~\cite{chowdhury2017kernelized} tells us that 
%In addition, also note that because $\widetilde{d}\log(1+TKN/\lambda) \leq 2\gamma_{TKN}$ (Sec.~\ref{sec:background}), we can alternatively define  $2\gamma_{TNK} + 1 \triangleq R'$.
We further define $\overline{R} \triangleq \lceil R' \rceil$, in which $\lceil \cdot \rceil$ denotes the ceiling operator.

%This allows us to show that
%\begin{equation}
%\begin{split}
%\sum^{P-1}_{p=0} \log \frac{\text{det} V_{p+1}}{\text{det} V_p }  = \log \frac{\text{det} V_P}{\text{det} V_0} &\leq \sum^{TN}_{t'=1} \norm{g(x_{t'};\theta_0) / \sqrt{m}}_{\widetilde{V}_{t'-1}^{-1}}^2\\
%&\leq \sum^{TN}_{t'=1} \frac{1}{\lambda} \widetilde{\sigma}^2_{t'-1,\text{approx}}(x_{t'})\\
%&\leq \sum^{TN}_{t'=1}  \widetilde{\sigma}^2_{t'-1,\text{exact}}(x_{t'}) + TN \varepsilon_{\text{NTK}}(m,TN)\\
%&\leq C_1 \gamma_{TN} + TN \varepsilon_{\text{NTK}}(m,TN) \triangleq R',
%\end{split}
%\label{eq:sum:of:log:det}
%\end{equation}
%in which we have defined $C_1 \triangleq \frac{2}{\log(1+\lambda^{-1})}$, and $\gamma_{TN}$ denotes the maximum information gain of the NTK kernel given any set of any $TN$ observations.
%The first inequality follows by making use of Lemma 11 of~\cite{abbasi2011improved} and by considering all inputs involved in the calculation of $V_P$ as being sequentially selected by our hypothetical agent introduced above.

Now we define all epochs $p$'s which satisfy the following condition as "good epochs":
%separate all epochs into good epochs and bad epochs according to the following criteria:
\begin{equation}
1 \leq \frac{\text{det}V_p}{\text{det} V_{p-1}} \leq e,
\label{eq:good:epoch:condition}
\end{equation}
%\begin{equation}
%1 \leq \frac{\text{det}V_p}{\text{det} V_{p-1}} \leq 2,
%\label{eq:good:epoch:condition}
%\end{equation}
and define all other epochs as "bad epochs".
The first inequality trivially holds for all epochs according to the way in which the matrices are constructed. It is easy to verify that the second inequality holds for at least $R$ epochs (with probability of at least $1-\delta_1$). This is because if the second inequality is violated for more than $\overline{R}$ epochs (i.e., if $\log\frac{\text{det}V_p}{\text{det} V_{p-1}} > 1$ for more than $\overline{R}$ epochs), then $\sum^{P-1}_{p=0} \log \frac{\text{det} V_{p+1}}{\text{det} V_p } > \overline{R}$, which violates \eqref{eq:sum:of:log:det}. 
This suggests that \emph{there are no more than $\overline{R}$ bad epochs} (with probability of at least $1-\delta_1$).
From here onwards, we will denote the set of good epochs by $\mathcal{E}^{\text{good}}$ and the set of bad epochs by $\mathcal{E}^{\text{bad}}$.
%From here onwards, we will call those epochs $p$ where equation~\eqref{eq:good:epoch:condition} is satisfied as "good" epochs and denote it by $\mathcal{E}^{\text{good}}$, and the other epochs as "bad" epochs and denote it by $\mathcal{E}^{\text{bad}}$.
%Note that we have implicitly assumed that the base of the logarithmic function is $2$.

\subsection{Validity of the Upper Confidence Bound}
\label{subsec:validity:of:ucbs}
In this section, we prove that the upper confidence bound used in our algorithm, $(1- \alpha_t) \text{UCB}^{a}_{t,i}(x) + \alpha_t \text{UCB}^{b}_{t,i}(x)$ (used in line 7 of Algo.~\ref{algo:agent}), is a valid high-probability upper bound on the reward function $h$.
We will achieve this by separately proving that $\text{UCB}^{a}_{t,i}$ and $\text{UCB}^{b}_{t,i}$ are valid high-probability upper bounds on $h$ in the next two sections.
Note that for both UCBs, unlike Neural UCB \citep{zhou2020neural} and Neural TS \citep{zhang2020neural} which use $\theta_t$ (the parameters of trained NNs) to calculate the exploration term (the second terms of $\text{UCB}^{a}_{t,i}$ and $\text{UCB}^{b}_{t,i}$), we instead use $\theta_0$. This is consistent with \citet{kassraie2021neural} who have shown that the use of $\theta_0$ gives accurate uncertainty estimation.

\subsubsection{Validity of $\text{UCB}^{a}_{t,i}$ as A High-Probability Upper Bound on $h$:}
\label{subsec:validity:of:ucb:2}
To begin with, we will need the following lemma from \cite{zhang2020neural}.
\begin{lemma}[Lemma B.3 of~\cite{zhang2020neural}]
\label{lemma:approx:h:by:linear:func}
Let $\delta_2\in(0,1)$.
There exists a constant $C>0$ such that 
%for any $\delta\in(0,1)$, 
if
\begin{equation*}
m \geq C T^4 K^4 N^4 L^6 \log(T^2 K^2 N^2 L / \delta_2) / \lambda_0^4,
\end{equation*}
then with probability of $\geq 1-\delta_2$ over random initializations of $\theta_0$, there exists a $\theta^*\in\mathbb{R}^{p_0}$ such that
\begin{equation}
h(x) = \langle g(x;\theta_0), \theta^* - \theta_0 \rangle, \quad \sqrt{m} \norm{\theta^* - \theta_0}_{2} \leq \sqrt{2 \mathbf{h}^{\top} \mathbf{H}^{-1} \mathbf{h}} \leq B, \forall x\in\mathcal{X}_{t,i}, t\in[T],i\in[N].
\end{equation}
\end{lemma}
%\paragraph{Validity of $\text{UCB}^{b}_{t,i}$:}
%Recall that we have defined in the main text that $\nu_{t} = B + R\sqrt{2(\log(1/\delta_{\text{UCB}_2}) + 1 + \gamma_{t})}$.
%Following similar steps of proof as Appendix B.2 in the work of~\cite{zhang2020neural}, we can obtain 
The condition on $m$ required by Lemma \ref{lemma:approx:h:by:linear:func} corresponds to condition 2 listed in App.~\ref{app:conditions:on:m}, except that $\delta_2$ is used here instead of $\delta$ as in App.~\ref{app:conditions:on:m}.
The following lemma formally guarantees the validity of $\text{UCB}^{a}_{t,i}$ as a high-probability upper-bound on $h$.
\begin{lemma}
%[\cite{zhang2020neural}]
\label{lemma:confidence:bound:ucb:2}
Let $\delta_3 \in(0,1)$ and $\nu_{TKN} = B + R\sqrt{2(\log(1/\delta_3) + 1) + \widetilde{d}\log(1+TKN/\lambda)}$.
We have with probability of at least $1-\delta_1 - \delta_2 - \delta_3$ for all $t\in[T],i\in[N]$, that
\begin{equation*}
\begin{split}
|h(x) - \langle g(x;\theta_0) /\sqrt{m}, \overline{\theta}_{t,i} \rangle| &\leq  \nu_{TKN} \sqrt{\lambda} \norm{g(x;\theta_0) / \sqrt{m}}_{\overline{V}_{t,i}^{-1}}, \forall x\in \mathcal{X}_{t,i}
\end{split}
\end{equation*}
\end{lemma}
%Lemma~\ref{lemma:confidence:bound:ucb:2} immediately implies the validity of the high-probability upper bound of $\text{UCB}^{b}_{t,i}$ (Algorithm~\ref{algo:agent}). Next, we prove that $\text{UCB}^{a}_{t,i}$ is also a valid upper bound (with high probability).

\begin{proof}
Lemma~\ref{lemma:confidence:bound:ucb:2} can be proved by following similar steps as the proof of Lemma 4.3 in the work of~\cite{zhang2020neural}.
%, as well as using the inequality $\log\det(I + \mathbf{H}/\lambda) \leq 2\gamma_{TKN}$.
Specifically, the proof of Lemma~\ref{lemma:confidence:bound:ucb:2} requires Lemmas B.3, B.6 and B.7 of~\cite{zhang2020neural} (after being adapted for our setting).
The adapted versions of Lemmas B.3 and B.7 of~\cite{zhang2020neural} have been presented in our Lemma~\ref{lemma:approx:h:by:linear:func} and Lemma~\ref{lemma:bound:error:NTK:matrix}, respectively.
%Specifically, see the lemma below for Lemma B.3 of~\cite{zhang2020neural} (adapted for our setting).
%On the other hand, Lemma B.7 of~\cite{zhang2020neural} (adapted for our setting) has been presented earlier in our Lemma~\ref{lemma:bound:error:NTK:matrix}.
Of note, Lemma~\ref{lemma:bound:error:NTK:matrix} and Lemma~\ref{lemma:approx:h:by:linear:func} require some conditions on the width $m$ of the NN,
%Of note, the validity of Lemmas B.3 and B.7 requires some conditions on the width $m$ of the NN. 
%The required conditions on $m$ from Lemma~\ref{lemma:bound:error:NTK:matrix} and Lemma~\ref{lemma:approx:h:by:linear:func} have been 
which have been listed as conditions 1 and 2 in Appendix~\ref{app:conditions:on:m}.
Lastly, Lemma B.6 of~\cite{zhang2020neural}, which makes use of Theorem 1 of~\cite{chowdhury2017kernelized}, can be directly applied in our setting and introduces an error probability of $\delta_3$ (which appears in the expression of $\nu_{TKN}$).
As a result, Lemma~\ref{lemma:confidence:bound:ucb:2} holds with probability of at least $1-\delta_1 - \delta_2 - \delta_3$, in which the error probabilities come from Lemma~\ref{lemma:bound:error:NTK:matrix} ($\delta_1$), Lemma~\ref{lemma:approx:h:by:linear:func} ($\delta_2$) and the application of Lemma B.6 of~\cite{zhang2020neural} ($\delta_3$).

%Lemma~\ref{lemma:confidence:bound:ucb:2} holds with high probability, in which the error probabilities come from the use of Lemmas B.3, B.6 and B.7 of~\cite{zhang2020neural}. 
%In our setting, these error probabilities correspond to $\delta_{\text{linear,func}}$, $\delta_{\text{UCB}_2}$ and $\delta_{\text{ntk,error}}$, respectively.
%Therefore, Lemma \ref{lemma:confidence:bound:ucb:2} holds with probability of at least $1-\delta_{\text{linear,func}} - \delta_{\text{UCB}_2} - \delta_{\text{ntk,error}}$.

%\noindent\fbox{
%    \parbox{\textwidth}{
%Setting all three of them to $\delta_{\text{UCB}_2}/3$ ensures that Lemma~\ref{lemma:confidence:bound:ucb:2} holds with probability of at least $1-\delta_{\text{UCB}_2}$; as a result, doing this will add an additional factor of $3$ within the $\log$ in the first two conditions on $m$ (Appendix~\ref{app:conditions:on:m}) which can be absorbed by the constant $C$, and will also introduce a factor of $3$ within the $\log$ in the expression of $\nu_{TKN}$ (Lemma~\ref{lemma:confidence:bound:ucb:2}) which we have ignored for simplicity.
%    }%
%}

\end{proof}

\subsubsection{Validity of $\text{UCB}^{b}_{t,i}$ as A High-Probability Upper Bound on $h$:}
\label{app:subsec:validity:of:ucb:b}
%\paragraph{Validity of $\text{UCB}^{a}_{t,i}$:}
Note that $\text{UCB}^{b}_{t,i}$ is updated only in every communication round. We denote the set of iterations after which $\text{UCB}^{b}_{t,i}$ is updated (i.e., the last iteration in every epoch) as $\mathcal{T}_{-1} \triangleq \{t_p-1\}_{p=2,\ldots,P-1}$, which immediately implies that $\mathcal{T}_{-1}\subset [T]$ and hence $|\mathcal{T}_{-1}| \leq T$.
\begin{lemma}
\label{lemma:confidence:bound:ucb:1}
Let $\delta_4,\delta_5\in(0,1)$, and $\nu_{TK} = B + R\sqrt{2(\log(N/\delta_4) + 1) + \widetilde{d}_{\max}\log(1+TK/\lambda)}$
Suppose the width $m$ of the NN satisfies $m \geq C \sqrt{\lambda}L^{-3/2} [\log(TKNL^2/\delta_5)]^{3/2}$ for some constant $C>0$, as well as condition 4 in Appendix~\ref{app:conditions:on:m}.
Suppose the learning rate $\eta$ and number of iterations $J$ of the gradient descent training satisfy the conditions in \eqref{eq:condition:on:eta} and \eqref{eq:condition:on:J} (App.~\ref{app:conditions:on:m}), respectively.
We have with probability of at least $1-\delta_4 - \delta_5$ for all $t\in\mathcal{T}_{-1},i\in[N]$, that
\[
|h(x) - f(x;\theta_{\text{sync,NN}})| \leq  \nu_{TK}  \sqrt{\lambda} \frac{1}{N}\sum^N_{i=1} \norm{g(x;\theta_0) / \sqrt{m}}_{(V^{\text{local}}_{i})^{-1}} + \varepsilon_{\text{linear}}(m,T), \forall x\in \mathcal{X}_{t,i}.
\]
\end{lemma}
\begin{proof}
%To begin with, 
Note that the condition on $m$ listed in the lemma, $m \geq C \sqrt{\lambda}L^{-3/2} [\log(TKNL^2/\delta_5)]^{3/2}$, corresponds to condition 3 listed in App.~\ref{app:conditions:on:m} except that $\delta_5$ is used here instead of $\delta$. Therefore, the validity of Lemma \ref{lemma:confidence:bound:ucb:1} requires conditions 3 and 4 on $m$ (App.~\ref{app:conditions:on:m}) to be satisfied.
For ease of exposition, we separate our proof into three steps.

\textbf{Step 1: NN Output $f(x; \theta_{\text{sync,NN}})$ Is Close to (Averaged) Linear Prediction}

Based on Lemma C.2 of~\cite{zhang2020neural}, 
%if the conditions 3.~and 4.~on $m$ (Appendix~\ref{app:conditions:on:m}) are satisfied, 
if the conditions on $m$ listed in Lemma \ref{lemma:confidence:bound:ucb:1} is satisfied,
then for any $\widetilde{\theta}$ such that $\norm{\widetilde{\theta} - \theta_0}_2 \leq 2\sqrt{t / (m\lambda)}$, there exists a constant $C_1>0$ such that we have with probability of at least $1-\delta_5$ over random initializations $\theta_0$ that 
\begin{equation}
\begin{split}
|f(x; \widetilde{\theta}) - \langle g(x;\theta_0), \widetilde{\theta} - \theta_0 \rangle | &\leq C_1 t^{2/3} m^{-1/6} \lambda^{-2/3} L^3\sqrt{\log m} \\
&\leq C_1 T^{2/3} m^{-1/6} \lambda^{-2/3} L^3\sqrt{\log m}\\
&\triangleq \varepsilon_{\text{linear}, 1}(m,T),
\label{eq:linearization:nn}
\end{split}
\end{equation}
which holds $\forall x\in\mathcal{X}_{t,i}, t\in[T], i\in[N]$.
%The addition term of $N$ in condition 3 in Appendix~\ref{app:conditions:on:m} comes from our requirement that \eqref{eq:linearization:nn} holds for (the contexts of) all $N$ agents.

Also note that according to Lemma C.1 of ~\cite{zhang2020neural}, 
%if (\emph{a}) the conditions 3 and 4 on $m$ (Appendix~\ref{app:conditions:on:m}) are satisfied and (\emph{b}) the learning rate $\eta$ for gradient descent satisfies $\eta\leq C_3(m\lambda + tmL)^{-1}$ for some constant $C_3$ (note that the result here holds for any iteration $j\in[J]$ of gradient descent, so it does not need any requirement on $J$), 
if conditions 3 and 4 on $m$ listed in App.~\ref{app:conditions:on:m}, as well as the condition on $\eta$ \eqref{eq:condition:on:eta}, are satisfied,
then we have with probability of at least $1-\delta_5$ over random initializations $\theta_0$ that $\norm{\theta^{i}_t - \theta_0}_2 \leq 2\sqrt{t / m\lambda},\forall i\in[N]$. 
An immediate implication is that the aggregated NN parameters $\theta_{\text{sync,NN}}=\frac{1}{N}\sum^N_{i=1}  \theta^{i}_t$ also satisfies: 
%$\norm{\theta_{\text{sync,NN}} - \theta_0}_2 \leq 2\sqrt{t / m\lambda}$.
\[
\norm{\theta_{\text{sync,NN}} - \theta_0}_2 = \norm{\frac{1}{N}\sum^N_{i=1}  \theta^{i}_t - \theta_0}_2 \leq \frac{1}{N}\sum^N_{i=1} \norm{\theta^{i}_t - \theta_0}_2 \leq 2\sqrt{t / m\lambda}.
\]
This implies that \eqref{eq:linearization:nn} holds for $\theta_{\text{sync,NN}}$ with probability of at least $1-2\delta_5$:
%all $\theta^{(i)}_t$'s and hence also holds 
%for $\theta_{\text{sync,NN}}$:
\begin{equation}
|f(x; \theta_{\text{sync,NN}}) - \langle g(x;\theta_0), \theta_{\text{sync,NN}} - \theta_0 \rangle | \leq \varepsilon_{\text{linear}, 1}(m,T).
\label{eq:sync:nn:linear}
\end{equation}
%That is, $\norm{\theta_{\text{sync,NN}} - \theta_0}_2 \leq 2\sqrt{t / m\lambda}$ with probability of at least $1-\delta/N$.
%Note that Lemma C.1 of ~\cite{zhang2020neural} places an upper bound on the learning rate $\eta$ of the gradient descent algorithm with which the $\theta^{(i)}_t$'s are obtained.
%We need to use a union bound over all $N$ agents to ensure that it holds for all agents. This is the reason for one of the additional factor of $\log N$ in the conditions on $m$ (Appendix~\ref{app:conditions:on:m}).

Next, note that the $\theta^{i}_t$ is obtained by training only using agent $i$'s local observations (line 14 of Algo.~\ref{algo:agent}).
Define $\theta_{t,i}^{\text{local}} = (V^{\text{local}}_{t,i})^{-1}(\sum^{t}_{\tau=1} y_{\tau,i} g(x_{\tau,i};\theta_0) / \sqrt{m} )$.
Note that $\theta_{t,i}^{\text{local}}$ is calculated in the same way as $\overline{\theta}_{t,i}$ (line 4 of Algo.~\ref{algo:agent}), except that its calculation only involves agent $i$'s local observations.
% and $V^{\text{local}}_{t,i}$. 
Next, making use of Lemmas C.1 and C.4 of~\cite{zhang2020neural}, we can follow similar steps as equation C.3 of~\cite{zhang2020neural} (in Appendix C.2 of~\cite{zhang2020neural}) to show that there exists constants $C_2>0$ and $C_3>0$ such that we have $\forall x\in\mathcal{X}_{t,i}, t\in[T], i\in[N]$ that
%with probability of at least $1-\delta_{\text{nn}}$ that 
\begin{equation}
\begin{split}
|\langle g(x;\theta_0), \theta^{i}_t &- \theta_0 \rangle - \langle g(x;\theta_0) / \sqrt{m}, \theta_{t,i}^{\text{local}} \rangle  | \\
&\leq C_2(1-\eta m \lambda)^{J}\sqrt{tL/\lambda} + C_3 m^{-1/6} \sqrt{\log m} L^4 t^{5/3} \lambda^{-5/3} (1+\sqrt{t/\lambda})\\
&\leq C_2(1-\eta m \lambda)^{J}\sqrt{TL/\lambda} +  C_3 m^{-1/6} \sqrt{\log m} L^4 T^{5/3} \lambda^{-5/3} (1+\sqrt{T/\lambda})\\
&\triangleq \varepsilon_{\eta,J} +  \varepsilon_{\text{linear}, 2}(m,T).
\end{split}
\label{eq:diff:between:linear:and:gp:mean}
\end{equation}
We refer to $\langle g(x;\theta_0) / \sqrt{m}, \theta_{t,i}^{\text{local}} \rangle$ as the \textbf{linear prediction} because it is the prediction of the linear model with the neural tangent features $g(x;\theta_0) / \sqrt{m}$ as the input features, conditioned on the local observations of agent $i$.
Note that similar to~\eqref{eq:sync:nn:linear} which also relies on Lemma C.1 of~\cite{zhang2020neural}, \eqref{eq:diff:between:linear:and:gp:mean} also requires conditions 3 and 4 on $m$, as well as the condition on $\eta$, in App.~\ref{app:conditions:on:m} to be satisfied.
%(\emph{a}) the conditions 3.~and 4.~on $m$ (Appendix~\ref{app:conditions:on:m}) are satisfied and (\emph{b}) the learning rate $\eta$ for gradient descent satisfies $\eta\leq C_3(m\lambda + tmL)^{-1}$ for some constant $C_3$. 
\eqref{eq:diff:between:linear:and:gp:mean} holds with probability of at least $1-2\delta_5$, where the error probabilities come from the use of Lemmas C.1 and C.4 of~\cite{zhang2020neural}.
%By choosing the learning rate $\eta$ and the number of iterations $J$ for gradient descent in a similar way as~\cite{zhang2020neural}, we can make sure that the first error term $\varepsilon_{\eta,J}$ is upper-bounded by a constant (e.g., it can be upper-bounded by $1/3$ in~\cite{zhang2020neural}).
%Also note that $\varepsilon_{\text{linear}, 2}(m,T)$ decreases as $m$ increases and goes to $0$ as $m\rightarrow \infty$.

%Note that $\theta_{t,i}^{\text{local}}$ is calculated in the same way as $\overline{\theta}_{t,i}$, except that its calculation only involves agent $i$'s local observations and $V^{\text{local}}_{t,i}$. 

Next, we can bound the difference between $f(x;\theta_{\text{sync,NN}})$ (i.e., the prediction of the NN with the aggregated parameters) and the averaged linear predictions of all agents calculated using their local observations:
\begin{equation}
\begin{split}
|f(x;\theta_{\text{sync,NN}}) &- \frac{1}{N}\sum^N_{i=1} \langle g(x;\theta_0) / \sqrt{m}, \theta_{t,i}^{\text{local}} \rangle|
\leq |f(x;\theta_{\text{sync,NN}}) - \frac{1}{N}\sum^N_{i=1} \langle g(x;\theta_0) , \theta^{i}_t - \theta_0 \rangle | \\
&\quad + |\frac{1}{N}\sum^N_{i=1} \langle g(x;\theta_0), \theta^{i}_t - \theta_0 \rangle 
- \frac{1}{N}\sum^N_{i=1} \langle g(x;\theta_0) / \sqrt{m}, \theta_{t,i}^{\text{local}} \rangle|\\
&\leq |f(x; \theta_{\text{sync,NN}}) - \langle g(x;\theta_0), \theta_{\text{sync,NN}} - \theta_0 \rangle |  \\
&\quad  + \frac{1}{N}\sum^N_{i=1} | \langle g(x;\theta_0), \theta^{i}_t - \theta_0 \rangle 
- \langle g(x;\theta_0) / \sqrt{m}, \theta_{t,i}^{\text{local}} \rangle|\\
&\leq \varepsilon_{\text{linear}, 1}(m,T) + \frac{1}{N}\sum^N_{i=1} (\varepsilon_{\eta,J} + \varepsilon_{\text{linear}, 2}(m,T)) \\
&\leq \varepsilon_{\text{linear}, 1}(m,T) + \varepsilon_{\eta,J} + \varepsilon_{\text{linear}, 2}(m,T) \\
&\triangleq \varepsilon_{\text{linear}}(m,T).
\end{split}
\label{eq:agg:nn:is:close:to:gp:mean}
\end{equation}
In the second inequality, we plugged in the definition of $\theta_{\text{sync,NN}}=\frac{1}{N}\sum^N_{i=1}  \theta^{i}_t$.
In the third inequality, we have made use of \eqref{eq:sync:nn:linear} and \eqref{eq:diff:between:linear:and:gp:mean};
\eqref{eq:agg:nn:is:close:to:gp:mean} holds with probability of at least $1-4\delta_5$, where the error probabilities come from \eqref{eq:sync:nn:linear} ($2\delta_5$) and \eqref{eq:diff:between:linear:and:gp:mean} ($2\delta_5$), respectively. Now we replace $\delta_5$ by $\delta_5/4$, which ensures that \eqref{eq:agg:nn:is:close:to:gp:mean} holds with probability of at least $1-\delta_5$.
This will only introduce a factor of $4$ within the $\log$ of condition 3 on $m$ (App.~\ref{app:conditions:on:m}), which is ignored since it can be absorbed by the constant $C$.

%Now we replace $\delta_{\text{nn}}$ by $\delta_{\text{nn}} / 3$, which ensures that equation~\eqref{eq:agg:nn:is:close:to:gp:mean} holds with probability of at least $1-\delta_{\text{nn}}$, and it will only introduce a factor of $3$ within the $\log$ of condition 3.~on $m$ (Appendix~\ref{app:conditions:on:m}) which can be absorbed by the constant $C$.
%Note that $\varepsilon_{\text{linear}}(m,T)$ can also be made arbitrarily small by making $m$ sufficiently large. That is, as long as the width $m$ of the NN is wide enough, the output of the NN with the aggregated parameters can be well-approximated by a weighted average of the GP posterior mean with the neural tangent features.

%We need to take a union bound over the two events above, however, we omit it for simplicity, because the resulting constants will only be reflected in the conditions on $m$ (Appendix~\ref{app:conditions:on:m}) and can be absorbed into the corresponding constants.

\textbf{Step 2: Linear Prediction Is Close to the Reward Function $h(x)$}

In the proof in this section, we will also need a "local" variant of the confidence bound of Lemma~\ref{lemma:confidence:bound:ucb:2}, i.e., the confidence bound of Lemma~\ref{lemma:confidence:bound:ucb:2} calculated only using the local observations of an agent $i$:
%Define $\nu_{TK} = B + R\sqrt{2(\log(N/\delta_{\text{UCB}_1}) + 1) + \widetilde{d}_{\max}\log(1+TK/\lambda) }$, we have that
\begin{lemma}[\cite{zhang2020neural}]
\label{lemma:confidence:bound:ucb:1:local}
We have with probability of at least $1-\delta_4$ for all $t\in \mathcal{T}_{-1} \subset [T],i\in[N]$, that
\[
|h(x) - \langle g(x;\theta_0) / \sqrt{m}, \theta^{\text{local}}_{t,i} \rangle| \leq  \nu_{TK} \sqrt{\lambda} \norm{g(x;\theta_0) / \sqrt{m}}_{(V_{t,i}^{\text{local}})^{-1}}, \forall x\in \mathcal{X}_{t,i}.
\]
\end{lemma}
\begin{proof}
Similar to the proof of Lemma~\ref{lemma:confidence:bound:ucb:2} (App.~\ref{subsec:validity:of:ucb:2}), the proof of Lemma \ref{lemma:confidence:bound:ucb:1:local} also requires of Lemmas B.3, B.6 and B.7 from~\cite{zhang2020neural}, 
%which are exactly the same as those corresponding lemmas from~\cite{zhang2020neural} except that we need an additional union bound over all $N$ agents.
which, in this case, can be directly applied to our setting (except that we need an additional union bound over all $N$ agents).
The implication of the additional union bound on the error probabilities is taken care of by the additional term of $N$ within the $\log$ in the expression of $\nu_{TK}$ (Lemma \ref{lemma:confidence:bound:ucb:1}), and in conditions 1 and 2 on $m$ (see App.~\ref{app:conditions:on:m}, and also Lemmas \ref{lemma:approx:h:by:linear:func} and~\ref{lemma:bound:error:NTK:matrix}).
The required lower bounds on $m$ by the local variants of Lemmas B.3 and B.7 (required in the proof here) are smaller than those given in Lemmas \ref{lemma:approx:h:by:linear:func} and~\ref{lemma:bound:error:NTK:matrix} and hence do not need to appear in the conditions in Appendix~\ref{app:conditions:on:m}.
By letting the sum of the three error probabilities (resulting from the applications of Lemmas B.3, B.6 and B.7 of~\cite{zhang2020neural}) be $\delta_4$, 
we can ensure that Lemma \ref{lemma:confidence:bound:ucb:1:local} holds with probability of at least $1-\delta_4$.
For simplicity, we let the error probability for Lemma B.6 be $\delta_4$, which leads to the cleaner expression of $\nu_{TK}$ in Lemma~\ref{lemma:confidence:bound:ucb:1}. 
This means that the error probabilities for Lemmas B.3 and B.7 are very small, which can be accounted for by simply increasing the value of the absolute constant $C$ in conditions 1 and 2 on $m$ (App.~\ref{app:conditions:on:m}) and hence does not affect our main theoretical analysis.
\end{proof}

\textbf{Step 3: Combining Results from Step 1 and Step 2}

Next, we are ready to prove the validity of $\text{UCB}^{b}_{t,i}$ by using the averaged linear prediction $\frac{1}{N}\sum^N_{i=1} \langle g(x;\theta_0) / \sqrt{m}, \theta_{t,i}^{\text{local}} \rangle$ as an intermediate term:
\begin{equation}
\begin{split}
|f(x;&\theta_{\text{sync,NN}}) - h(x)| \\
&= |f(x;\theta_{\text{sync,NN}}) - \frac{1}{N}\sum^N_{i=1} \langle g(x;\theta_0) / \sqrt{m}, \theta_{t,i}^{\text{local}} \rangle
+ \frac{1}{N}\sum^N_{i=1} \langle g(x;\theta_0) / \sqrt{m}, \theta_{t,i}^{\text{local}} \rangle
- h(x)|\\
&\leq \frac{1}{N}\sum^N_{i=1} | \langle g(x;\theta_0) / \sqrt{m}, \thetaß†_{t,i}^{\text{local}} \rangle - h(x)| + \varepsilon_{\text{linear}}(m,T)\\
&\leq \frac{1}{N}\sum^N_{i=1} \nu_{TK} \sqrt{\lambda} \norm{g(x;\theta_0) / \sqrt{m}}_{(V^{\text{local}}_{t,i})^{-1}} + \varepsilon_{\text{linear}}(m,T)\\
&= \frac{1}{N}\sum^N_{i=1} \nu_{TK} \sqrt{\lambda} \norm{g(x;\theta_0) / \sqrt{m}}_{(V^{\text{local}}_{i})^{-1}} + \varepsilon_{\text{linear}}(m,T)\\
% &=  \nu_{TK} \frac{1}{N}\sum^N_{i=1} \sqrt{ \lambda g(x;\theta_0)^{\top} (V^{\text{local}}_{t,i})^{-1} g(x;\theta_0) / m} + \varepsilon_{\text{linear}}(m,T)\\
% &\leq \nu_{TK}  \sqrt{ \frac{1}{N}\sum^N_{i=1} \lambda g(x;\theta_0)^{\top} (V^{\text{local}}_{t,i})^{-1} g(x;\theta_0) / m} + \varepsilon_{\text{linear}}(m,T)\\
% &= \nu_{TK}  \sqrt{  \lambda g(x;\theta_0)^{\top} \left( \frac{1}{N}\sum^N_{i=1} (V^{\text{local}}_{t,i})^{-1} \right) g(x;\theta_0) / m} + \varepsilon_{\text{linear}}(m,T)\\
% &= \nu_{TK}  \sqrt{  \lambda g(x;\theta_0)^{\top} \left( V_{\text{sync,NN}}^{-1} \right) g(x;\theta_0) / m} + \varepsilon_{\text{linear}}(m,T)\\
% &= \nu_{TK}  \sqrt{\lambda} \norm{g(x;\theta_0) / \sqrt{m}}_{V_{\text{sync,NN}}^{-1}} + \varepsilon_{\text{linear}}(m,T).
\end{split}
\label{eq:last:eq:ucb:1}
\end{equation}
The second inequality has made use of \eqref{eq:agg:nn:is:close:to:gp:mean}, and the third inequality follows from Lemma~\ref{lemma:confidence:bound:ucb:1:local}.
% , the fourth inequality results from the concavity of the square root function.
% In the second last equality, we have plugged in the definition of $V_{\text{sync,NN}}^{-1} = \frac{1}{N}\sum^{N}_{i=1} (V^{\text{local}}_{t,i})^{-1}$.
In the last inequality, we have made the substitution of $(V^{\text{local}}_{i})^{-1}=(V^{\text{local}}_{t,i})^{-1}$. This is because in the proof of Lemma \ref{lemma:confidence:bound:ucb:1} here, we only consider the iterations of $t\in\mathcal{T}_{-1} \triangleq \{t_p-1\}_{p=2,\ldots,P-1}$, i.e., \emph{the last iteration of every epoch}. As a result, this ensures that $(V^{\text{local}}_{i})^{-1}=(V^{\text{local}}_{t,i})^{-1}$ because every time the central server obtains $(V^{\text{local}}_{i})^{-1}$ through $(V^{\text{local}}_{i})^{-1}=(V^{\text{local}}_{t,i})^{-1},\forall i\in[N]$ (line 4 of Algo.~\ref{algo:server}), we have that the current iteration $t$ is \emph{the last iteration of the previous epoch}.
As a results, \eqref{eq:last:eq:ucb:1} holds with probability of at least $1-\delta_4 - \delta_5$, in which the error probabilities come from \eqref{eq:agg:nn:is:close:to:gp:mean} ($\delta_5$) and Lemma~\ref{lemma:confidence:bound:ucb:1:local} ($\delta_4$).
In other words, Lemma~\ref{lemma:confidence:bound:ucb:1} (i.e., the validity of $\text{UCB}^{b}_{t,i}$) holds with probability of at least $1-\delta_4 - \delta_5$.

%\noindent\fbox{
%    \parbox{\textwidth}{
%Now letting both $\delta_{\text{nn}}$ and $\delta_{\text{UCB}_1}$ be equal to $\delta_{\text{UCB}_1}/2$, we have that Equation~\eqref{eq:last:eq:ucb:1}, as well as Lemma~\ref{lemma:confidence:bound:ucb:1} (i.e., the validity of $\text{UCB}^{b}_{t,i}$) holds with probability of at least $1-\delta_{\text{UCB}_1}$.
%Again, doing this will introduce another factor of $2$ within the $\log$ in condition 3.~on $m$ (Appendix~\ref{app:conditions:on:m}) which can be absorbed by the constant $C$, and will introduce a factor of $2$ into the expression of $\nu_{TK}$ which we have ignored for simplicity.
%    }%
%}

\end{proof}

\subsection{Regret Upper Bound for Good Epochs}
\label{subsec:regret:good:epochs}
In this section, we derive an upper bound on the total regrets incurred in all good epochs $\mathcal{E}^{\text{good}}$ (defined in App.~\ref{subsubsec:definition:good:bad:epochs}).

% \subsubsection{Auxiliary Inequalities}
% We firstly derive two auxiliary results which will be used in the proofs later. 
\subsubsection{Auxiliary Inequality}
We firstly derive an auxiliary result which will be used in the proofs later. 

% To begin with, 
For agent $i$ and iteration $t$ in a good epoch $p \in \mathcal{E}^{\text{good}}$, we have that
\begin{equation}
\begin{split}
\sqrt{\lambda} \norm{g(x;\theta_0) / \sqrt{m}}_{\overline{V}_{t,i}^{-1}} &= \sqrt{\lambda g(x;\theta_0)^{\top} \overline{V}_{t,i}^{-1} g(x;\theta_0) / m}\\
&\leq \sqrt{\lambda g(x;\theta_0)^{\top} \widetilde{V}_{t,i}^{-1} g(x;\theta_0)  / m \frac{\text{det}\widetilde{V}_{t,i}}{\text{det}\overline{V}_{t,i}} }\\
&\leq \sqrt{\lambda g(x;\theta_0)^{\top} \widetilde{V}_{t,i}^{-1} g(x;\theta_0)  / m \frac{\text{det}V_{p}}{\text{det}V_{p-1}} }\\
&\leq \sqrt{e \lambda g(x;\theta_0)^{\top} \widetilde{V}_{t,i}^{-1} g(x;\theta_0)  / m }\\
&=\sqrt{e \lambda} \norm{g(x;\theta_0) / \sqrt{m}}_{\widetilde{V}_{t,i}^{-1}}.
\end{split}
\label{eq:change:det:for:good:epochs}
\end{equation}
Recall that $\overline{V}_{t,i}$ (line 4 of Algo.~\ref{algo:agent}) is used by agent $i$ in iteration $t$ to select $x_{t,i}$ (via $\text{UCB}^{a}_{t,i}$),
and that the matrix $\widetilde{V}_{t,i}$ is defined for the hypothetical agent which sequentially chooses all $TN$ queries $\{x_{t,i}\}_{t\in[T],i\in[N]}$ in a round-robin fashion (App.~\ref{subsubsec:definition:good:bad:epochs}).
The first inequality in \eqref{eq:change:det:for:good:epochs} above follows from Lemma 12 of~\cite{abbasi2011improved}.
The second inequality is because $V_p$ contains more information than $\widetilde{V}_{t,i}$ (since $V_p$ is calculated using all the inputs selected \emph{after} epoch $p$), and $V_{p-1}$ contains less information than $\overline{V}_{t,i}$ (because compared with $V_{p-1}$, $\overline{V}_{t,i}$ additionally contains the local inputs selected by agent $i$ in the current epoch $p$).
In the last inequality, we have made use of the definition of good epochs, i.e., $(\text{det}V_{p}) / (\text{det}V_{p-1}) \leq e$ (App.~\ref{subsubsec:definition:good:bad:epochs}).

\subsubsection{Upper Bound on the Instantaneous Regret $r_{t,i}$}
Here we assume that both $\text{UCB}^{a}_{t,i}$ and $\text{UCB}^{b}_{t,i}$ hold (hence we ignore the error probabilities here), which we have proved in App.~\ref{subsec:validity:of:ucbs}.
We now derive an upper bound on the instantaneous regret $r_{t,i} = h(x^*_{t,i}) - h(x_{t,i})$ for agent $i$ and iteration $t$ in a good epoch $p \in \mathcal{E}^{\text{good}}$:
%by making use of the confidence bounds from Lemma~\ref{lemma:confidence:bound:ucb:1} and Lemma~\ref{lemma:confidence:bound:ucb:2}:
\begin{equation}
\begin{split}
r_{t,i} &= h(x^*_{t,i}) - h(x_{t,i}) \\
&= \alpha h(x^*_{t,i}) + (1-\alpha) h(x^*_{t,i}) - h(x_{t,i})\\
&\leq \alpha \text{UCB}^{b}_{t,i}(x^*_{t,i}) + \alpha \varepsilon_{\text{linear}}(m,T) + (1-\alpha) \text{UCB}^{a}_{t,i}(x^*_{t,i}) - h(x_{t,i})\\
&\leq \alpha \text{UCB}^{b}_{t,i}(x_{t,i}) + (1-\alpha) \text{UCB}^{a}_{t,i}(x_{t,i}) + \alpha \varepsilon_{\text{linear}}(m,T) - h(x_{t,i})\\
&= \alpha \left( \text{UCB}^{b}_{t,i}(x_{t,i}) - h(x_{t,i}) \right) + (1-\alpha) \left( \text{UCB}^{a}_{t,i}(x_{t,i}) - h(x_{t,i})\right) + \alpha \varepsilon_{\text{linear}}(m,T)\\
&\leq \alpha \Big(  2 \nu_{TK}  \frac{1}{N} \sum^N_{j=1} \sqrt{\lambda} \norm{g(x_{t,i};\theta_0) / \sqrt{m}}_{(V^{\text{local}}_{j})^{-1}} + \varepsilon_{\text{linear}}(m,T) \Big) + \\
&\quad (1-\alpha) \Big( 2\nu_{TKN} \sqrt{\lambda} \norm{g(x_{t,i};\theta_0) / \sqrt{m}}_{\overline{V}_{t,i}^{-1}}\Big) + \alpha \varepsilon_{\text{linear}}(m,T)\\
&\leq \alpha \Big(  2 \nu_{TK}  \frac{1}{N} \sum^N_{j=1} \sqrt{\lambda} \norm{g(x_{t,i};\theta_0) / \sqrt{m}}_{(V^{\text{local}}_{j})^{-1}} + \varepsilon_{\text{linear}}(m,T) \Big) + \\
&\quad (1-\alpha) \Big( 2\nu_{TKN} \sqrt{e \lambda} \norm{g(x_{t,i};\theta_0) / \sqrt{m}}_{\widetilde{V}_{t,i}^{-1}}\Big) + \alpha \varepsilon_{\text{linear}}(m,T)\\
&= \alpha  2 \nu_{TK}  \frac{1}{N} \sum^N_{j=1} \sqrt{\lambda} \norm{g(x_{t,i};\theta_0) / \sqrt{m}}_{(V^{\text{local}}_{j})^{-1}} + \\
&\quad (1-\alpha)  2\nu_{TKN} \sqrt{e \lambda} \norm{g(x_{t,i};\theta_0) / \sqrt{m}}_{\widetilde{V}_{t,i}^{-1}} + 2 \alpha \varepsilon_{\text{linear}}(m,T)\\
%&\triangleq \alpha_t  2 \nu_{TK}  \frac{1}{\sqrt{N}} \sum^N_{j=1} \overline{\sigma}^{\text{local}}_{t_p-1,j}(x_{t,i}) + (1-\alpha_t)  2\nu_{TKN} \sqrt{e} \widetilde{\sigma}_{t,i}(x_{t,i}) + 2 \alpha_t \varepsilon_{\text{linear}}(m,T).
&\triangleq (1-\alpha)  2\nu_{TKN} \sqrt{e} \widetilde{\sigma}_{t,i}(x_{t,i}) + \alpha 2 \nu_{TK}  \frac{1}{N} \sum^N_{j=1} \widetilde{\sigma}^{\text{local}}_{t_p-1,j}(x_{t,i}) + 2 \alpha \varepsilon_{\text{linear}}(m,T).
\end{split}
\label{eq:upper:bound:inst:regret}
\end{equation}
The first inequality makes use of Lemma \ref{lemma:confidence:bound:ucb:2} (i.e., the validity of $\text{UCB}^{a}_{t,i}$) and Lemma \ref{lemma:confidence:bound:ucb:1} (i.e., the validity of $\text{UCB}^{b}_{t,i}$).
The second inequality follows from the way in which $x_{t,i}$ is selected (line 7 of Algo.~\ref{algo:agent}): $x_{t,i} = {\arg\max}_{x\in\mathcal{X}_{t,i}} (1- \alpha) \text{UCB}^{a}_{t,i}(x) + \alpha \text{UCB}^{b}_{t,i}(x)$.
The third inequality again makes use of Lemma \ref{lemma:confidence:bound:ucb:2} and Lemma \ref{lemma:confidence:bound:ucb:1}, as well as the expressions of $\text{UCB}^{a}_{t,i}$ and $\text{UCB}^{b}_{t,i}$.
In the fourth inequality, we have made used of the auxiliary inequality of \eqref{eq:change:det:for:good:epochs} 
% and \eqref{eq:good:epoch:unpack:V:sync} 
we derived in the last section.
Recall that we have discussed that $(V^{\text{local}}_{i})^{-1}=(V^{\text{local}}_{t,i})^{-1}$ for all $t=t_p-1$ at the end of App.~\ref{app:subsec:validity:of:ucb:b}.
Therefore, in the last step, we have defined $\widetilde{\sigma}^{\text{local}}_{t_p-1,j}(x_{t,i}) \triangleq \sqrt{\lambda} \norm{g(x_{t,i};\theta_0) / \sqrt{m}}_{(V^{\text{local}}_{t_p-1,j})^{-1}} = \sqrt{\lambda} \norm{g(x_{t,i};\theta_0) / \sqrt{m}}_{(V^{\text{local}}_{j})^{-1}}$ which represents the GP posterior standard deviation (using the kernel of $\widetilde{k}(x,x')=g(x;\theta_0)^{\top} g(x';\theta_0) / m$) conditioned on all agent $j$'s local observations before iteration $t_p$.
Note that $\widetilde{\sigma}^{\text{local}}_{t_p-1,j}(x_{t,i})$ is the same as the one defined in 
% Sec.~\ref{sec:fn_ucb} 
Sec.~\ref{subsec:weight:two:ucbs}
of the main text where we explain the weight between the two UCBs.
Similarly, we have also defined $\widetilde{\sigma}_{t,i}(x_{t,i}) \triangleq \sqrt{\lambda} \norm{g(x_{t,i};\theta_0) / \sqrt{m}}_{\widetilde{V}_{t,i}^{-1}}$, which represents the GP posterior standard deviation conditioned on the observations of the hypothetical agent before $x_{t,i}$ is selected (App.~\ref{subsubsec:definition:good:bad:epochs}).
% (and it hence is the same as the term that appears in the analysis of standard GP-UCB).

%Next, as a result of our assumption that $|h(x)| \leq 1$, we further have that
%\begin{equation}
%\begin{split}
%r_{t,i} &\leq \min\{ \alpha_t  2 \nu_t  \frac{1}{\sqrt{N}} \sum^N_{j=1} \overline{\sigma}^{\text{local}}_{t_p-1,j}(x_{t,i}) + (1-\alpha_t)  2\nu_{TN} \sqrt{2} \widetilde{\sigma}_{t,i}(x_{t,i}) + 2 \alpha_t \varepsilon_{\text{linear}}(m,T), 2\}\\
%&\leq \min\{ \alpha_t  2 \nu_t  \frac{1}{\sqrt{N}} \sum^N_{j=1} \overline{\sigma}^{\text{local}}_{t_p-1,j}(x_{t,i}), 2\} + \min\{ (1-\alpha_t)  2\nu_{TN} \sqrt{2} \widetilde{\sigma}_{t,i}(x_{t,i}), 2\} + \\
%&\quad \min\{ 2 \alpha_t \varepsilon_{\text{linear}}(m,T), 2\}\\
%&\leq \min\{ \alpha_t  2 \nu_t  \frac{1}{\sqrt{N}} \sum^N_{j=1} \overline{\sigma}^{\text{local}}_{t_p-1,j}(x_{t,i}), 2\} + \min\{ (1-\alpha_t)  2\nu_{TN} \sqrt{2} \widetilde{\sigma}_{t,i}(x_{t,i}), 2\} + \\
%&\quad 2 \alpha_t \varepsilon_{\text{linear}}(m,T)\\
%\end{split}
%\label{eq:upper:bound:inst:regret}
%\end{equation}

%Next, we will separately sum up (across all good epochs and all agents) the first and second terms of the upper bound from equation \eqref{eq:upper:bound:inst:regret}.
Next, we will separately derive upper bounds on the summation (across all good epochs and all agents) of the first and second terms of the upper bound from \eqref{eq:upper:bound:inst:regret}.

\subsubsection{Upper Bound on the Sum of the First term of \eqref{eq:upper:bound:inst:regret}}
\label{app:subsubsec:upper:bound:sum:of:first:term:good:epochs}
Here, similar to~\cite{kassraie2021neural}, we denote as $\kappa_0$ an upper bound on the value of the NTK function for any input: $\langle g(x;\theta_0) / \sqrt{m}, g(x;\theta_0) / \sqrt{m} \rangle \leq \kappa_0,\forall x\in\mathcal{X}_{t,i},t\in[T],i\in[N]$.
As a result, we can use it to show that both $\widetilde{\sigma}_{t,i}(x)$ and $\widetilde{\sigma}^{\text{local}}_{t_p-1,j}(x)$ can be upper-bounded: $\widetilde{\sigma}_{t,i}(x) \leq \sqrt{\kappa_0}$ and $\widetilde{\sigma}^{\text{local}}_{t_p-1,j}(x) \leq \sqrt{\kappa_0}$.
%As a result, we can show that 
%%it is easy to verify (using the closed-form equations for the GP posterior variance) that 
%the GP posterior variance at any input $x$ is upper-bounded by $\kappa_0$: $\widetilde{\sigma}^2_{t,i}(x) \leq \kappa_0$ and $(\widetilde{\sigma}^{\text{local}}_{t_p-1,j}(x))^2 \leq \kappa_0$.
To show this, following the notations of Appendix \ref{subsubsec:definition:good:bad:epochs}, we denote $\widetilde{V}_{t,i}= J_{t,i} J^{\top}_{t,i} + \lambda I$ where $J_{t,i} = \left[\left[g(x_{\tau,j};\theta_0)\right]_{\tau \in [t-1], j\in[N]}, \left[g(x_{t,j};\theta_0)\right]_{j\in[i]}\right]$ which is a $p_0\times [(t-1)N + i]$ matrix.
Then we have
\begin{equation}
\begin{split}
\widetilde{\sigma}^2_{t,i}(x) &= \lambda \norm{g(x;\theta_0) / \sqrt{m}}_{\widetilde{V}_{t,i}^{-1}}^2\\
&=\lambda g(x;\theta_0)^{\top}  (J_{t,i} J^{\top}_{t,i} + \lambda I)^{-1}  g(x;\theta_0) / m\\
&=\lambda g(x;\theta_0)^{\top}
\Big( \frac{1}{\lambda} I - \frac{1}{\lambda}J_{t,i} \big( I + J_{t,i}^{\top} \frac{1}{\lambda} J_{t,i} \big)^{-1} J_{t,i}^{\top} \frac{1}{\lambda} \Big)
g(x;\theta_0) / m\\
&= g(x;\theta_0)^{\top} g(x;\theta_0) / m - (g(x;\theta_0)^{\top} / \sqrt{m}) J_{t,i} \big( \lambda I + J_{t,i}^{\top} J_{t,i} \big)^{-1} J_{t,i}^{\top} (g(x;\theta_0) / \sqrt{m})\\
&= g(x;\theta_0)^{\top} g(x;\theta_0) / m - \norm{(g(x;\theta_0)^{\top} / \sqrt{m}) J_{t,i}}^2_{\left( \lambda I + J_{t,i}^{\top} J_{t,i} \right)^{-1}}\\
&\leq g(x;\theta_0)^{\top} g(x;\theta_0) / m \leq \kappa_0
\end{split}
\label{eq:proof:using:gp:post:var}
\end{equation}
where we used the matrix inversion lemma in the third equality.
Using similar derivations also allows us to show that $(\widetilde{\sigma}^{\text{local}}_{t_p-1,j}(x))^2 \leq \kappa_0$.
Therefore, we have that $\widetilde{\sigma}_{t,i}(x) \leq \sqrt{\kappa_0}$ and $\widetilde{\sigma}^{\text{local}}_{t_p-1,j}(x) \leq \sqrt{\kappa_0}$.

Denoting the set of iterations from all good epochs as $\mathcal{T}^{\text{good}}$, we can derive an upper bound the first term of \eqref{eq:upper:bound:inst:regret}, summed across all agents $i\in[N]$ and all iteration in good epochs $\mathcal{T}^{\text{good}}$:
\begin{equation}
\begin{split}
\sum^N_{i=1} \sum_{t \in \mathcal{T}^{\text{good}}} (1-\alpha) 2 \nu_{TKN} &\sqrt{e} \widetilde{\sigma}_{t,i}(x_{t,i}) \stackrel{(a)}{\leq} 2\sqrt{e}\nu_{TKN} \sum^N_{i=1} \sum^T_{t=1} \widetilde{\sigma}_{t,i}(x_{t,i}) \\
&\stackrel{(b)}{=} 2\sqrt{e}\nu_{TKN} \sum^N_{i=1} \sum^T_{t=1} \min\{ \widetilde{\sigma}_{t,i}(x_{t,i}), \sqrt{\kappa_0}\}\\
&\stackrel{(c)}{\leq} 2\sqrt{e}\nu_{TKN} \sum^N_{i=1} \sum^T_{t=1} \min\{ \sqrt{\kappa_0} \widetilde{\sigma}_{t,i}(x_{t,i}), \sqrt{\kappa_0}\}\\
&\leq 2\sqrt{e}\nu_{TKN} \sqrt{\kappa_0} \sum^N_{i=1} \sum^T_{t=1} \min\{ \widetilde{\sigma}_{t,i}(x_{t,i}), 1\}\\
&\stackrel{(d)}{\leq} 2\sqrt{e}\nu_{TKN} \sqrt{\kappa_0} \sqrt{TN \sum^N_{i=1} \sum^T_{t=1} \min\{ \widetilde{\sigma}_{t,i}^2(x_{t,i}), 1\} }\\
&\stackrel{(e)}{\leq} 2\sqrt{e}\nu_{TKN} \sqrt{\kappa_0} \sqrt{TN [ 2\lambda \log\det (\lambda^{-1}\mathbf{K}_{TN} + I) ] }\\
&\stackrel{(f)}{\leq} 2\sqrt{2e}\nu_{TKN} \sqrt{\kappa_0} \sqrt{TN\lambda [\log\det (\lambda^{-1}\mathbf{H} + I) + 1] }\\
&= 2\sqrt{e}\nu_{TKN} \sqrt{\kappa_0} \sqrt{TN\lambda \left[\widetilde{d} \log(1+TNK/\lambda) + 1\right] }\\
\end{split}
\label{eq:proof:upper:bound:first:term:good:epochs}
\end{equation}
%\begin{equation}
%\begin{split}
%\sum^N_{i=1} \sum_{t \in \mathcal{T}^{\text{good}}} (1-\alpha_t) 2 \nu_{TN} &\sqrt{2} \widetilde{\sigma}_{t,i}(x_{t,i}) \leq 2\sqrt{2}\nu_{TN} \sum^N_{i=1} \sum^T_{t=1} \widetilde{\sigma}_{t,i}(x_{t,i}) \\
%&\leq 2\sqrt{2}\nu_{TN} \sum^N_{i=1} \sum^T_{t=1} \widetilde{\sigma}_{t,i,\text{exact}}(x_{t,i}) + 2\sqrt{2}\nu_{TN} NT \varepsilon_{\text{NTK}}(m, TN) \\
%&\leq 2\sqrt{2}\nu_{TN} \sqrt{TN \sum^N_{i=1}\sum^T_{t=1} \widetilde{\sigma}^2_{t,i,\text{exact}}(x_{t,i}) }  + 2 \sqrt{2}\nu_{TN} NT \varepsilon_{\text{NTK}}(m, TN)\\
%&\leq 2\sqrt{2}\nu_{TN} \sqrt{TN C_1 \gamma_{TN} }  + 2\sqrt{2}\nu_{TN} NT \varepsilon_{\text{NTK}}(m, TN)
%\end{split}
%\end{equation}
%This equation holds with probability of at least $1-\delta_1$.
Step $(a)$ follows from $\alpha \leq 1,\forall t\geq1$ and summing across all iterations $[T]$ instead of only those iterations $\mathcal{T}^{\text{good}}$ in good epochs. 
Step $(b)$ follows because $\widetilde{\sigma}_{t,i}(x) \leq \sqrt{\kappa_0}$ as discussed above.
In step $(c)$, we have assumed that $\kappa_0\geq1$; however, if $\kappa_0 < 1$, the proof still goes through since we can directly upper-bound $\min\{ \widetilde{\sigma}_{t,i}(x_{t,i}), \sqrt{\kappa_0}\}$ by $\min\{ \widetilde{\sigma}_{t,i}(x_{t,i}), 1\}$,
%$\sqrt{\kappa_0}$ by $1$, 
after which the only modification we need to make to the equation above is to remove the dependency on multiplicative term of $\sqrt{\kappa_0}$.
Step $(d)$ results from the Cauchy–Schwarz inequality.
%We have made use of the definition of $C_1 \triangleq \frac{2}{\log(1+\lambda^{-1})}$.
Step $(e)$ can be derived following the proof of Lemma 4.8 of~\cite{zhang2020neural} (in Appendix B.7 of \cite{zhang2020neural}).
Step $(f)$ follows from Lemma~\ref{lemma:bound:error:NTK:matrix} and hence holds with probability of at least $1-\delta_1$.
The last equality simply plugs in the definition of the effective dimension $\widetilde{d}$ (Sec.~\ref{sec:background}).

\subsubsection{Upper Bound on the Sum of the Second term of \eqref{eq:upper:bound:inst:regret}}
\label{app:subsubsec:upper:bound:sum:of:second:term:good:epochs}
In this subsection, we derive an upper bound on the sum of the second term in equation~\eqref{eq:upper:bound:inst:regret} across all good epochs and all agents.
%\emph{across a particular good epoch}. 
%We denote by $t_p$ the starting index of the good epoch $p$, and by $E_p$ the length of this good epoch.

For the proof here, we need a "local" version of Lemma~\ref{lemma:bound:error:NTK:matrix}, i.e., a version of Lemma~\ref{lemma:bound:error:NTK:matrix} which only makes use of the contexts of an agent $i$.
Define $\mathbf{K}_{t,i}$ as the local counterpart to $\mathbf{K}_{t'}$ (from Lemma~\ref{lemma:bound:error:NTK:matrix}), i.e., $\mathbf{K}_{t,i}$ is the $t \times t$ matrix calculated using only agent $j$'s local contexts up to (and including) iteration $t$. 
Specifically, define $J_{t,i}\triangleq [g(x_{\tau,i};\theta_0)]_{\tau\in[t]}$ which is a $p_0\times t$ matrix, then $\mathbf{K}_{t,i}$ is defined as $\mathbf{K}_{t,i}\triangleq J_{t,i}^{\top} J_{t,i} / m$, which is a $t \times t$ matrix.
Also recall that in the main text, we have defined $\mathbf{H}_i$ as the local counterpart of $\mathbf{H}$ for agent $i$ (Sec.~\ref{sec:background}).
The next lemma gives our desired local version of Lemma~\ref{lemma:bound:error:NTK:matrix}.
\begin{lemma}[Lemma B.7 of~\cite{zhang2020neural}]
\label{lemma:bound:error:NTK:matrix:local}
If $m\geq CT^6 K^6 L^6 \log(TNKL/\delta_6)$, we have with probability of at least $1-\delta_6$ that
\[
\log\det(I + \lambda^{-1}\mathbf{K}_{t,i}) \leq \log\det(I + \lambda^{-1}\mathbf{H}_i) + 1,
\]
for all $t\in[T],i\in[N]$.
\end{lemma}
We needed to take a union bound over all $N$ agents, which explains the factor of $N$ within the $\log$ in the lower bound on $m$ given in Lemma \ref{lemma:bound:error:NTK:matrix:local}.
Note that the required lower bound on $m$ by Lemma \ref{lemma:bound:error:NTK:matrix:local} is smaller than that of Lemma \ref{lemma:bound:error:NTK:matrix} (by a factor of $N^6$), therefore, the condition on $m$ in Lemma \ref{lemma:bound:error:NTK:matrix:local} is ignored in the conditions listed in App.~\ref{app:conditions:on:m}.

Of note, throughout the entire epoch $p$, $\widetilde{\sigma}^{\text{local}}_{t_p-1,j}(x_{t,i})$ is calculated \emph{conditioned on all the local observations of agent $j$ before iteration $t_p$}.
%Also note that because of the way in which $\alpha_t$ is chosen, we have that $\alpha_t \leq \widetilde{\sigma}^{\text{local}}_{t,i,\min} / \widetilde{\sigma}^{\text{local}}_{t,i,\max} \leq \widetilde{\sigma}^{\text{local}}_{t,j}(x_{t,j}) / \widetilde{\sigma}^{\text{local}}_{t,j}(x_{t,i})$.
%We use $\mathcal{T}_{1} \triangleq \{t_p\}_{p=2,\ldots,P}$
%to denote the collection of the indices of the first iteration $t$ after every communication round, which immediately implies that $\mathcal{T}_1 \subset [T]$.
Denote by $\mathcal{T}^{(p)}$ the iteration indices in epoch $p$: $\mathcal{T}^{(p)}=\{t_p,\ldots,t_p+E_p-1\}$.
In the proof in this section, as we have discussed in the first paragraph of Sec.~\ref{subsec:theoretical:results},
we analyze a simpler variant of our algorithm where we only set $\alpha>0$ in the first iteration after a communication round, 
%i.e., $\alpha_t > 0,\forall t\in \mathcal{T}_1$ and $\alpha_t=0,\forall t\in[T]\setminus \mathcal{T}_1$.
i.e., $\alpha > 0,\forall t\in \{t_p\}_{p\in[P]}$ and $\alpha=0,\forall t\in[T]\setminus \{t_p\}_{p\in[P]}$.
Now we are ready to derive an upper bound on the second term in \eqref{eq:upper:bound:inst:regret}, summed over all agents and all good epochs:
\begin{equation}
\begin{split}
\sum^N_{i=1} \sum_{p \in \mathcal{E}^{\text{good}}} \sum_{t\in\mathcal{T}^{(p)}} \alpha  2 \nu_{TK}  \frac{1}{N} &\sum^N_{j=1} \widetilde{\sigma}^{\text{local}}_{t_p-1,j}(x_{t,i}) \stackrel{(a)}{\leq} 2 \nu_{TK}  \frac{1}{N}  \sum^N_{i=1} \sum_{p \in [P]} \sum_{t\in\mathcal{T}^{(p)}} \alpha  \sum^N_{j=1} \widetilde{\sigma}^{\text{local}}_{t_p-1,j}(x_{t,i})\\
&\stackrel{(b)}{\leq} 2 \nu_{TK}  \frac{1}{N}  \sum^N_{i=1} \sum^N_{j=1} \sum_{p \in [P]}  \alpha \widetilde{\sigma}^{\text{local}}_{t_p-1,j}(x_{t_p,i})\\
&\stackrel{(c)}{\leq} 2 \nu_{TK}  \frac{1}{N}  \sum^N_{i=1} \sum^N_{j=1} \sum_{p \in [P]}  \widetilde{\sigma}^{\text{local}}_{t_p-1,j}(x_{t_p,j})\\
&\stackrel{(d)}{\leq} 2 \nu_{TK}  \frac{1}{N}  \sum^N_{i=1} \sum^N_{j=1} \sum^T_{t=1}  \widetilde{\sigma}^{\text{local}}_{t-1,j}(x_{t,j})\\
%&\stackrel{(e)}{\leq} 2 \nu_{TK} \sqrt{\kappa_0} \frac{1}{\sqrt{N}}  \sum^N_{i=1} \sum^N_{j=1} \sum^T_{t=1} \min\{ \widetilde{\sigma}^{\text{local}}_{t-1,j}(x_{t,j}), 1\}\\
%&\stackrel{(f)}{\leq} 2 \nu_{TK} \sqrt{\kappa_0} \frac{1}{\sqrt{N}}  \sum^N_{i=1} \sum^N_{j=1} \sqrt{ T \sum^T_{t=1} \min\{ \left(\widetilde{\sigma}^{\text{local}}_{t-1,j}(x_{t,j})\right)^2, 1\} }\\
%&\stackrel{(g)}{\leq} 2 \nu_{TK} \sqrt{\kappa_0} \frac{1}{\sqrt{N}}  \sum^N_{i=1} \sum^N_{j=1} \sqrt{ T [2\lambda \log\det(\lambda^{-1} \mathbf{K}_{T,j} + I)] }\\
%&\stackrel{(h)}{\leq} 2\sqrt{2} \nu_{TK} \sqrt{\kappa_0} \frac{1}{\sqrt{N}}  \sum^N_{i=1} \sum^N_{j=1} \sqrt{ T \lambda [\log\det(\lambda^{-1} \mathbf{H}_{j} + I) + 1] }\\
%&= 2\sqrt{2} \nu_{TK} \sqrt{\kappa_0} \sqrt{N} \sum^N_{j=1} \sqrt{ T \lambda \left[\widetilde{d}_j \log(1 + TK/\lambda) ) + 1\right] }.
\end{split}
\label{eq:good:epochs:second:term:eq:1}
\end{equation}
The inequality in step $(a)$ results from summing across all epoch $p\in[P]$ instead of only good epochs $p\in\mathcal{E}^{\text{good}}$.
%The inequality in step $(a)$ results from summing across all iterations $t\in[T]$ (or equivalently, $t\in\mathcal{T}^{(p)},p\in[P]$) instead of only those iterations in good epochs: $t\in\mathcal{T}^{(p)},p\in\mathcal{E}^{\text{good}}$.
Step $(b)$ follows since $\alpha_t=0,\forall t\in[T]\setminus \{t_p\}_{p\in[P]}$ as we discussed above, therefore, for every epoch $p$, we only need to keep the first term of $t=t_p$ in the summation of $t\in\mathcal{T}^{(p)}$.
To understand step $(c)$, recall that in the main text 
% (Sec.~\ref{sec:fn_ucb}, the paragraph "The Weight between the Two UCBs"), 
(Sec.~\ref{subsec:weight:two:ucbs}),
we have defined: $\widetilde{\sigma}^{\text{local}}_{t,i,\min} \triangleq \min_{x\in\mathcal{X}} \widetilde{\sigma}^{\text{local}}_{t,i}(x)$ and $\widetilde{\sigma}^{\text{local}}_{t,i,\max} \triangleq \max_{x\in\mathcal{X}} \widetilde{\sigma}^{\text{local}}_{t,i}(x),\forall i\in[N]$. 
Next, note that our algorithm selects $\alpha$ by: $\alpha=\min_{i\in[N]} \alpha_{t,i}$ (line 4 of Algo.~\ref{algo:server}) where $\alpha_{t,i} = \widetilde{\sigma}^{\text{local}}_{t,i,\min} / \widetilde{\sigma}^{\text{local}}_{t,i,\max}$ (line 15 of Algo.~\ref{algo:agent}) and $t=t_p-1$ since $\alpha_{t,i}$ is calculated only in the last iteration of every epoch.
As a result, we have that 
\[
\alpha =\min_{i\in[N]} \alpha_{t_p-1,i} = \min_{i\in[N]} \frac{ \widetilde{\sigma}^{\text{local}}_{t_p-1,i,\min} }{ \widetilde{\sigma}^{\text{local}}_{t_p-1,i,\max}}
 \leq \frac{\widetilde{\sigma}^{\text{local}}_{t_p-1,j,\min}}{\widetilde{\sigma}^{\text{local}}_{t_p-1,j,\max}} \leq \frac{\widetilde{\sigma}^{\text{local}}_{t_p-1,j}(x_{t_p,j}) }{ \widetilde{\sigma}^{\text{local}}_{t_p-1,j}(x_{t_p,i})},
\]
which tells us that $\alpha \widetilde{\sigma}^{\text{local}}_{t_p-1,j}(x_{t_p,i}) \leq \widetilde{\sigma}^{\text{local}}_{t_p-1,j}(x_{t_p,j})$ and hence leads to step $(c)$.
Step $(d)$ results from summing across all iterations $[T]$ instead of only the first iteration of every epoch.
% : $t\in\{t_p\}_{p\in[P]}$.

Next, we can derive an upper bound on the inner summation over $t=1,\ldots,T$ from \eqref{eq:good:epochs:second:term:eq:1}:
\begin{equation}
\begin{split}
%\sum^N_{i=1} \sum_{p \in \mathcal{E}^{\text{good}}} \sum_{t\in\mathcal{T}^{(p)}} \alpha_t  2 \nu_{TK}  \frac{1}{\sqrt{N}} &\sum^N_{j=1} \widetilde{\sigma}^{\text{local}}_{t_p-1,j}(x_{t,i}) \stackrel{(a)}{\leq} 2 \nu_{TK}  \frac{1}{\sqrt{N}}  \sum^N_{i=1} \sum_{p \in [P]} \sum_{t\in\mathcal{T}^{(p)}} \alpha_t  \sum^N_{j=1} \widetilde{\sigma}^{\text{local}}_{t_p-1,j}(x_{t,i})\\
%&\stackrel{(b)}{\leq} 2 \nu_{TK}  \frac{1}{\sqrt{N}}  \sum^N_{i=1} \sum^N_{j=1} \sum_{p \in [P]}  \alpha_t \widetilde{\sigma}^{\text{local}}_{t_p-1,j}(x_{t_p,i})\\
%&\stackrel{(c)}{\leq} 2 \nu_{TK}  \frac{1}{\sqrt{N}}  \sum^N_{i=1} \sum^N_{j=1} \sum_{p \in [P]}  \widetilde{\sigma}^{\text{local}}_{t_p-1,j}(x_{t_p,j})\\
\sum^T_{t=1}  \widetilde{\sigma}^{\text{local}}_{t-1,j}(x_{t,j}) &\stackrel{(a)}{\leq} \sqrt{\kappa_0}  \sum^T_{t=1} \min\{ \widetilde{\sigma}^{\text{local}}_{t-1,j}(x_{t,j}), 1\}\\
&\stackrel{(b)}{\leq} \sqrt{\kappa_0}  \sqrt{ T \sum^T_{t=1} \min\{ \left(\widetilde{\sigma}^{\text{local}}_{t-1,j}(x_{t,j})\right)^2, 1\} }\\
&\stackrel{(c)}{\leq}  \sqrt{\kappa_0}  \sqrt{ T [2\lambda \log\det(\lambda^{-1} \mathbf{K}_{T,j} + I)] }\\
&\stackrel{(d)}{\leq} \sqrt{2} \sqrt{\kappa_0} \sqrt{ T \lambda [\log\det(\lambda^{-1} \mathbf{H}_{j} + I) + 1] }\\
&= \sqrt{2} \sqrt{\kappa_0} \sqrt{ T \lambda \left[\widetilde{d}_j \log(1 + TK/\lambda) ) + 1\right] }.
\end{split}
\label{eq:good:epochs:second:term:eq:2}
\end{equation}
Step $(a)$ is obtained in the same way as steps $(b)$ and $(c)$ in \eqref{eq:proof:upper:bound:first:term:good:epochs} (App.~\ref{app:subsubsec:upper:bound:sum:of:first:term:good:epochs}), i.e., we have made use of $\widetilde{\sigma}^{\text{local}}_{t_p-1,j}(x) \leq \sqrt{\kappa_0}$ and assumed that $\kappa_0\geq1$. Again note that if $\kappa_0 < 1$, then the proof still goes through since $\widetilde{\sigma}^{\text{local}}_{t-1,j}(x_{t,j}) \leq \min\{ \widetilde{\sigma}^{\text{local}}_{t-1,j}(x_{t,j}), \sqrt{\kappa_0}\} \leq \min\{ \widetilde{\sigma}^{\text{local}}_{t-1,j}(x_{t,j}), 1\}$, after which the only modification we need to make to the equation above is to remove the dependency on multiplicative term of $\sqrt{\kappa_0}$.
Step $(b)$ makes use of the Cauchy–Schwarz inequality.
Step $(c)$, similar to step $(e)$ of \eqref{eq:proof:upper:bound:first:term:good:epochs}, is derived following the proof of Lemma 4.8 of~\cite{zhang2020neural} (in Appendix B.7 of \cite{zhang2020neural}).
Step $(d)$ follows from Lemma \ref{lemma:bound:error:NTK:matrix:local} and hence holds with probability of at least $1-\delta_6$.
In the last equality, we have simply plugged in the definition of $\widetilde{d}_j$ (Sec.~\ref{sec:background}).

Now we can plug \eqref{eq:good:epochs:second:term:eq:2} into \eqref{eq:good:epochs:second:term:eq:1} to obtain
\begin{equation}
\begin{split}
\sum^N_{i=1} \sum_{p \in \mathcal{E}^{\text{good}}} \sum_{t\in\mathcal{T}^{(p)}} &\alpha 2 \nu_{TK}  \frac{1}{N} \sum^N_{j=1} \widetilde{\sigma}^{\text{local}}_{t_p-1,j}(x_{t,i}) \\
&\leq 2 \nu_{TK}  \frac{1}{N}  \sum^N_{i=1} \sum^N_{j=1} \sqrt{2} \sqrt{\kappa_0} \sqrt{ T \lambda \left[\widetilde{d}_j \log(1 + TK/\lambda) ) + 1\right] }\\
&= 2\sqrt{2} \nu_{TK} \sqrt{\kappa_0} \sum^N_{j=1} \sqrt{ T \lambda \left[\widetilde{d}_j \log(1 + TK/\lambda) ) + 1\right] }.
\end{split}
\label{eq:final:upper:bound:on:second:term:eq:13}
\end{equation}

\subsubsection{Putting Things Together}
\label{app:subsubsec:good:epochs:putting:things:together}
Finally, recall that our derived upper bound on $r_{t,i}$ in \eqref{eq:upper:bound:inst:regret} contains three terms (the third term is simply an error term), and now we can make use of our derived upper bound on the first term (App.~\ref{app:subsubsec:upper:bound:sum:of:first:term:good:epochs}) and the second term (App.~\ref{app:subsubsec:upper:bound:sum:of:second:term:good:epochs}), summed over all agents and all good epochs, to obtain an upper bound on the total regrets incurred in all good epochs:
\begin{equation}
\begin{split}
R_T^{\text{good}} &= \sum^N_{i=1} \sum_{t \in \mathcal{T}^{\text{good}}} r_{t,i} \\
&\leq 2\sqrt{e}\nu_{TKN} \sqrt{\kappa_0} \sqrt{TN\lambda \left[\widetilde{d} \log(1+TNK/\lambda) + 1\right] } + \\
&\quad 2\sqrt{2} \nu_{TK} \sqrt{\kappa_0} \sum^N_{j=1} \sqrt{ T \lambda \left[\widetilde{d}_j \log(1 + TK/\lambda) ) + 1\right] } + TN\varepsilon_{\text{linear}}(m,T)\\
&=\widetilde{O}\Big(\sqrt{\widetilde{d}} \sqrt{TN\widetilde{d}} + \sqrt{\widetilde{d}}_{\max} N \sqrt{T \widetilde{d}_{\max}} + TN\varepsilon_{\text{linear}}(m,T)
\Big)\\
&=\widetilde{O}\left(\widetilde{d} \sqrt{TN} + \widetilde{d}_{\max} N \sqrt{T} + TN\varepsilon_{\text{linear}}(m,T)
\right).
\end{split}
\label{eq:final:regret:upper:bound:in:good:epochs}
\end{equation}
In the second last equality, we have used $\nu_{TKN}=\widetilde{O}(\sqrt{\widetilde{d}})$ and $\nu_{TK}=\widetilde{O}(\sqrt{\widetilde{d}_{\max}})$.

\subsection{Regret Upper Bound for Bad Epochs}
\label{subsec:regret:bad:epochs}
In this section, we derive an upper bound on the total regrets from all bad epochs.
% for FN-UCB-Lite, which we show is tighter than that of FN-UCB shown above.
%Here we assume that $\text{UCB}^{a}_{t,i}$ holds.
To begin with, we firstly derive an upper bound on the total regrets of any bad epoch $p$ denoted as $R^{[p]}$:
%, and then use the fact that there are not so many bad epochs.
\begin{equation}
\begin{split}
R^{[p]} = \sum^N_{i=1}\sum^{t_p + E_p - 1}_{t=t_p} r_{t,i} &\stackrel{(a)}{\leq}  \sum^N_{i=1}\Big(2 + 2 + \sum^{t_p + E_p - 2}_{t=t_p+1}r_{t,i}\Big)\\
&\stackrel{(b)}{\leq} \sum^N_{i=1}\Big[4 + \sum^{t_p + E_p - 2}_{t=t_p+1} \big( \text{UCB}^{a}_{t,i}(x^*_{t,i}) - h(x_{t,i}) \big)\Big]\\
&\stackrel{(c)}{\leq} \sum^N_{i=1}\Big[4 + \sum^{t_p + E_p - 2}_{t=t_p+1} \big( \text{UCB}^{a}_{t,i}(x_{t,i}) - h(x_{t,i}) \big)\Big]\\
&\stackrel{(d)}{\leq} \sum^N_{i=1}\Big[4 + \sum^{t_p + E_p - 2}_{t=t_p+1} 2\nu_{TKN} \sqrt{\lambda} \norm{g(x_{t,i};\theta_0) / \sqrt{m}}_{\overline{V}_{t,i}^{-1}} \Big]\\
%&\leq \sum^N_{i=1}\left(4 + \sum^{t_p + E_p - 2}_{t=t_p} 2\nu_{TN} \min\{  \sqrt{\lambda} \norm{g(x_{t,i};\theta_0) / \sqrt{m}}_{\overline{V}_{t,i}^{-1}}, 1 \}\right)\\
&\stackrel{(e)}{\leq} \sum^N_{i=1}\Big(4 + 2\nu_{TKN} \sqrt{\kappa_0} \sum^{t_p + E_p - 2}_{t=t_p} \min\{ \sqrt{\lambda} \norm{g(x_{t,i};\theta_0) / \sqrt{m}}_{\overline{V}_{t,i}^{-1}}, 1 \}\Big)\\
&\stackrel{(f)}{\leq} \sum^N_{i=1}\Big(4 + 2\nu_{TKN} \sqrt{\kappa_0\lambda} \sum^{t_p + E_p - 2}_{t=t_p} \min\{  \norm{g(x_{t,i};\theta_0) / \sqrt{m}}_{\overline{V}_{t,i}^{-1}}, 1 \}\Big)
\end{split}
\label{eq:regret:bad:epochs:eq:1}
\end{equation}
Step $(a)$ follows from simply upper-bounding the regrets of the first and last iteration within this epoch by $2$.
Step $(b)$ makes use of the validity of $\text{UCB}^{a}_{t,i}$ (Lemma~\ref{lemma:confidence:bound:ucb:2}).
Step $(c)$ follows because $\alpha=0,\forall t\in[T]\setminus \{t_p\}_{p\in[P]}$ (i.e., we set $\alpha=0$ except for the first iteration of all epochs), which implies that after the first iteration of an epoch, $x_{t,i}$ is selected by only maximizing $\text{UCB}^{a}_{t,i}$ (line 7 of Algo.~\ref{algo:agent}).
Step $(d)$ again uses Lemma~\ref{lemma:confidence:bound:ucb:2}, as well as the expression of $\text{UCB}^{a}_{t,i}$.
Step $(e)$ is obtained in the same way as steps $(b)$ and $(c)$ in \eqref{eq:proof:upper:bound:first:term:good:epochs} (App.~\ref{app:subsubsec:upper:bound:sum:of:first:term:good:epochs}).
Specifically, since $\langle g(x;\theta_0), g(x;\theta_0) \rangle \leq \kappa_0$ (App.~\ref{app:subsubsec:upper:bound:sum:of:first:term:good:epochs}), therefore, $\sqrt{\lambda} \norm{g(x_{t,i};\theta_0) / \sqrt{m}}_{\overline{V}_{t,i}^{-1}} \leq \sqrt{\kappa_0}$, 
which can be proved by following the same steps as \eqref{eq:proof:using:gp:post:var}.
%which can be verified using the expressions of the GP posterior variance. 
As a result, if we assume that $\kappa \geq 1$, then $\sqrt{\lambda} \norm{g(x_{t,i};\theta_0) / \sqrt{m}}_{\overline{V}_{t,i}^{-1}} = \min \{ \sqrt{\lambda} \norm{g(x_{t,i};\theta_0) / \sqrt{m}}_{\overline{V}_{t,i}^{-1}},  \sqrt{\kappa_0} \} \leq \sqrt{\kappa_0} \min \{ \sqrt{\lambda} \norm{g(x_{t,i};\theta_0) / \sqrt{m}}_{\overline{V}_{t,i}^{-1}},  1\}$; in the other case where $\kappa_0 < 1$, then $\sqrt{\lambda} \norm{g(x_{t,i};\theta_0) / \sqrt{m}}_{\overline{V}_{t,i}^{-1}} = \min \{ \sqrt{\lambda} \norm{g(x_{t,i};\theta_0) / \sqrt{m}}_{\overline{V}_{t,i}^{-1}},  \sqrt{\kappa_0} \} \leq \min \{ \sqrt{\lambda} \norm{g(x_{t,i};\theta_0) / \sqrt{m}}_{\overline{V}_{t,i}^{-1}},  1\}$.
Here we have assumed $\kappa_0 \geq 1$ for simplicity, since when $\kappa_0<1$, the equation above still holds except that we can remove the dependency on $\sqrt{\kappa_0}$.
Step $(f)$ follows because $\lambda =1+2/T >1$.

Next, we derive an upper bound on the inner summation in \eqref{eq:regret:bad:epochs:eq:1}.
\begin{equation}
\begin{split}
&\sum^{t_p + E_p - 2}_{t=t_p} \min\{  \norm{g(x_{t,i};\theta_0) / \sqrt{m}}_{\overline{V}_{t,i}^{-1}}, 1 \}\\
&\stackrel{(a)}{\leq}  \sqrt{ (E_p-1) \sum^{t_p + E_p - 2}_{t=t_p} \min\{  \norm{g(x_{t,i};\theta_0) / \sqrt{m}}_{\overline{V}_{t,i}^{-1}}^2, 1 \} } \\
&\stackrel{(b)}{\leq} \sqrt{ (E_p-1) 2 \log\frac{\text{det}\overline{V}_{t_p+E_p-2,i} }{\text{det}\overline{V}_{t_p,i} } } \\
&\stackrel{(c)}{\leq} \sqrt{ 2 ((t_p+E_p-2) - t_{\text{last}}) \log\frac{\text{det}V_{t_p+E_p-2,i} }{\text{det}V_{\text{last}} } } \\
&\stackrel{(d)}{\leq} \sqrt{2 D}.
%&= \left(4 + 2\nu_{TKN} \sqrt{2 \kappa_0\lambda D} \right) N\\
\end{split}
\label{eq:regret:bad:epochs:eq:2}
\end{equation}
Step $(a)$ follows from the Cauchy–Schwarz inequality.
Step $(b)$ makes use of Lemma 11 of~\cite{abbasi2011improved}.
In step $(c)$, we used the notations of $t_{\text{last}}=t_p-1$, $\overline{V}_{t_p,i}=V_{\text{last}}$ (this is because in the first iteration $t_p$ of an epoch, $W_{\text{new},i}=\mathbf{0}_{p_0\times p_0}$ and hence $\overline{V}_{t_p,i}=V_{\text{last}}=W_{\text{sync}}+\lambda I$), and $V_{t_p+E_p-2,i} = \overline{V}_{t_p+E_p-2,i} + g(x_{t,i};\theta_0) g(x_{t,i};\theta_0)^{\top} / m$, and also used $\text{det}\overline{V}_{t_p+E_p-2,i} \leq \text{det}V_{t_p+E_p-2,i}$.
To understand step $(d)$, note that the term in step $(c)$: $((t_p+E_p-2) - t_{\text{last}}) \log\frac{\text{det}V_{t_p+E_p-2,i} }{\text{det}V_{\text{last}} }$ is exactly the criterion we use to check whether to start a communication round in iteration $t=t_p+E_p-2$ (line 11 of Algo.~\ref{algo:agent}).
Since $t=t_p+E_p-2$ is not the last iteration in this epoch (i.e., we did not start a communication round after checking this criterion in iteration $t=t_p+E_p-2$), therefore, this criterion is not satisfied in iteration $t=t_p+E_p-2$, i.e., $((t_p+E_p-2) - t_{\text{last}}) \log\frac{\text{det}V_{t_p+E_p-2,i} }{\text{det}V_{\text{last}} } \leq D$, which explains step $(d)$.

Next, we can plug \eqref{eq:regret:bad:epochs:eq:2} into \eqref{eq:regret:bad:epochs:eq:1} to obtain:
\begin{equation}
\begin{split}
R^{[p]} = \sum^N_{i=1}\sum^{t_p + E_p - 1}_{t=t_p} r_{t,i} \leq \sum^N_{i=1}\left(4 + 2\nu_{TKN} \sqrt{\kappa_0\lambda}   \sqrt{2D}  \right) = \left(4 + 2\nu_{TKN} \sqrt{2 \kappa_0\lambda D} \right) N,
\end{split}
\label{eq:regret:bad:epochs:eq:3}
\end{equation}
which gives an upper bound on the total regret from \emph{any} bad epoch.
%The first inequality follows from simply upper-bounding the regrets of the first and last iteration within this epoch by $2$.
%The second inequality makes use of the validity of $\text{UCB}^{a}_{t,i}$ (Lemma~\ref{lemma:confidence:bound:ucb:2}).
%In the third inequality, we have made use our algorithmic choice that we let $\alpha_t=0$ except for the first iteration after a synchronization round, so that we don't need to worry about the impact of $\text{UCB}^{b}_{t,i}$ because $x_{t,i}$ is selected by simply maximizing $\text{UCB}^{a}_{t,i}$.
%The seventh inequality makes use of Lemma 11 of~\cite{abbasi2011improved}.
%In the eighth inequality, we replaced some notations: $t_{\text{last}}=t_p-1$ and $\overline{V}_{t_p,i}=V_{\text{last}}$, and used $\text{det}\overline{V}_{t_p+E_p-2,i} \leq \text{det}V_{t_p+E_p-2,i}$.
%Of note, the term in the eighth line: $((t_p+E_p-2) - t_{\text{last}}) \log\frac{\text{det}V_{t_p+E_p-2,i} }{\text{det}V_{\text{last}} }$ is exactly the criterion we use to check whether to start a communication round in iteration $t=t_p+E_p-2$, and because this is not the last iteration in this epoch (i.e., we did not start a communication round after checking this criterion in this iteration), we have that this term is $<D$. This explains the next inequality.
%We have used $\lambda =1+2/T >1$.
Now recall that as we have discussed in App.~\ref{subsubsec:definition:good:bad:epochs}, there are no more than $\overline{R}$ bad epochs (with probability of at least $1-\delta_1$).
Therefore, the total regret of \emph{all} bad epochs can be upper-bounded by:
\begin{equation}
\begin{split}
R_T^{\text{bad}} &\leq \overline{R} \left(4 + 2\nu_{TKN} \sqrt{2 \kappa_0\lambda D} \right) N\\
&\leq \left( \widetilde{d}\log(1+TKN/\lambda) + 1\right) \left(4 + 2\nu_{TKN} \sqrt{2 \kappa_0\lambda D} \right) N\\
&=\widetilde{O}\Big(\widetilde{d} \sqrt{\widetilde{d}} \sqrt{D} N  \Big)\\
&=\widetilde{O}\Big((\widetilde{d})^{3/2} \sqrt{D} N  \Big).
\end{split}
\label{eq:final:regret:upper:bound:in:bad:epochs}
\end{equation}
In the second last equality, we have used $\nu_{TKN}=\widetilde{O}(\sqrt{\widetilde{d}})$.
%This upper bound holds with probability of at least $1-\delta_1$, because \eqref{eq:sum:of:log:det} holds with probability of at least $1-\delta_1$.
%\begin{equation}
%\begin{split}
%R_T^{\text{bad}} &\leq R \left(4 + 2\nu_{TN} \sqrt{ 2 \lambda D} \right) N\\
%&\leq \left(C_1 \gamma_{TN} + TN \varepsilon_{\text{NTK}}(m,TN) + 1\right) \left(4 + 2\nu_{TN} \sqrt{ 2 \lambda D} \right) N\\
%&= \mathcal{O}\left( \nu_{TN} \sqrt{ D}  N \left(\gamma_{TN} + TN \varepsilon_{\text{NTK}}(m, TN)\right) \right)\\
%&= \mathcal{O}\left( \sqrt{ D} N \sqrt{\gamma_{TN}}  \left(\gamma_{TN} + TN \varepsilon_{\text{NTK}}(m, TN)\right) \right)\\
%\end{split}
%\end{equation}
%In the last step, we have used $\nu_{TN}=\mathcal{O}(\sqrt{\gamma_{TN}})$.
By choosing $D=\mathcal{O}(\frac{T}{N \widetilde{d}})$ (line 1 of Algo.~\ref{algo:agent}), we can further express the above upper bound on the total regrets from all bad epochs as:
\begin{equation}
\begin{split}
R_T^{\text{bad}} &= \mathcal{O}\Big( \sqrt{\frac{T}{N \widetilde{d}}}    (\widetilde{d})^{3/2} N \Big)\\
&=\mathcal{O}\left( \widetilde{d}\sqrt{TN} \right).
\end{split}
\label{eq:final:final:regret:upper:bound:in:bad:epochs}
\end{equation}
%\begin{equation}
%\begin{split}
%R_T^{\text{bad}} &= \mathcal{O}\left( \sqrt{\frac{T}{N \gamma_{TN}}} N \sqrt{\gamma_{TN}}  \left(\gamma_{TN} + TN \varepsilon_{\text{NTK}}(m, TN)\right) \right)\\
%&=\mathcal{O}\left( \sqrt{TN}  \left(\gamma_{TN} + TN \varepsilon_{\text{NTK}}(m, TN)\right) \right).
%\end{split}
%\end{equation}

\subsection{Final Regret Upper Bound}
\label{subsec:regret:total}
Here we derive an upper bound on the total cumulative regret by adding up the regrets resulting from all good epochs (App.~\ref{subsec:regret:good:epochs}) and all bad epochs (App.~\ref{subsec:regret:bad:epochs}):
\begin{equation}
\begin{split}
R_T &= R^{\text{good}}_T+R^{\text{bad}}_T \\
&= \widetilde{O}\Big(\widetilde{d} \sqrt{TN} + \widetilde{d}_{\max} N \sqrt{T} + TN\varepsilon_{\text{linear}}(m,T) + \widetilde{d}\sqrt{TN}
\Big)\\
&=\widetilde{O}\Big(\widetilde{d} \sqrt{TN} + \widetilde{d}_{\max} N \sqrt{T} + TN\varepsilon_{\text{linear}}(m,T)
\Big).
\end{split}
\label{app:eq:final:regret:upper:bound:eq:1}
\end{equation}
This regret upper bound holds with probability of at least $1-\delta_1-\delta_2 - \delta_3 - \delta_4 - \delta_5-\delta_6$.
We let $\delta_3=\delta_4=\delta/3$, which leads to the expressions of $\nu_{TKN}$ and $\nu_{TK}$ given in the main paper (Sec.~\ref{sec:fn_ucb}).
We let $\delta_1=\delta_2=\delta_5=\delta_6=\delta/12$, and this will only introduce an additional factor of $\log 12$ in the first three conditions on $m$ in App.~\ref{app:conditions:on:m} which can be absorbed by the constant $C$.

Next, the last term from the upper bound in \eqref{app:eq:final:regret:upper:bound:eq:1} can be further written as:
\begin{equation}
\begin{split}
TN&\varepsilon_{\text{linear}}(m,T) = TN\Big(\varepsilon_{\text{linear}, 1}(m,T) + \varepsilon_{\text{linear}, 2}(m,T) + \varepsilon_{\eta,J}
\Big)\\
&=TN C_1 T^{2/3} m^{-1/6} \lambda^{-2/3} L^3\sqrt{\log m}  + TN C_3 m^{-1/6} \sqrt{\log m} L^4 T^{5/3} \lambda^{-5/3} (1+\sqrt{T/\lambda}) \\
&\quad + TN C_2(1-\eta m \lambda)^{J}\sqrt{TL/\lambda}.
\end{split}
\label{eq:error:term:final:regret:bound}
\end{equation}
It can be easily verified that as long as $m(\log m)^{-3} \geq 3^6 C_1^6 T^{10}N^6 \lambda^{-4} L^{18}$ and $m(\log m)^{-3} \geq 3^6 C_3^6 T^{16} N^6 L^{24} \lambda^{-10}(1+\sqrt{T/\lambda})^6$ (which are ensured by conditions 5 and 6 on $m$ in App.~\ref{app:conditions:on:m}), then the first and second terms in~\eqref{eq:error:term:final:regret:bound} can both be upper-bounded by $1/3$.
Moreover, if the conditions on $\eta$ and $J$ presented in App.~\ref{app:conditions:on:m} are satisfied, i.e., if we choose the learning rate as $\eta=C_4(m\lambda + mTL)^{-1}$ in which $C_4>0$ is an absolute constant such that $C_4 \leq 1+TL$, and choose $J=\frac{1}{C_4}\Big(1+\frac{TL}{\lambda}\Big) \log\Big( \frac{1}{3C_2 N} \sqrt{\frac{\lambda}{T^3 L}} \Big)=\widetilde{O}\left(TL/(\lambda C_4) \right)$, then the third term in~\eqref{eq:error:term:final:regret:bound} can also be upper-bounded by $1/3$.

As a result, the last term from the upper bound in \eqref{app:eq:final:regret:upper:bound:eq:1} can be upper-bounded by $1$, and hence the regret upper bound becomes:
\begin{equation}
\begin{split}
R_T = \widetilde{O}\Big(\widetilde{d} \sqrt{TN} + \widetilde{d}_{\max} N \sqrt{T}
\Big).
\end{split}
\label{eq:final:regret:upper:bound:added:rebuttal}
\end{equation}

\paragraph{Worst-Case Regret Upper Bound in Terms of the Maximum Information Gain $\gamma$.}
Next, we perform some further analysis of the final regret upper bound derived above, which allows us to inspect the order of growth of our regret upper bound in the worst-case scenario (i.e., without assuming that the effective dimensions are upper-bounded by constants).
We have defined in Sec.~\ref{sec:background} that $\widetilde{d} \leq  2\gamma_{TKN} / \log(1+TKN/\lambda)$, $\widetilde{d}_i \leq 2\gamma_{TK} / \log(1+TK/\lambda),\forall i\in[N]$ and $\widetilde{d}_{\max}=\max_{i\in[N]}\widetilde{d}$.
As a result, in our derivations in \eqref{eq:final:regret:upper:bound:in:good:epochs} and \eqref{eq:final:regret:upper:bound:in:bad:epochs}, we can replace $\widetilde{d} \log(1+TKN/\lambda)$ by $2\gamma_{TKN}$ and replace $\widetilde{d}_{j} \log(1+TK/\lambda)$ by $2\gamma_{TK}$, after which the regret upper bound becomes
\begin{equation}
\begin{split}
R_T = \widetilde{O}\Big( \gamma_{TKN} \sqrt{TN} + \gamma_{TK} N \sqrt{T}
\Big).
\end{split}
\end{equation}
% \begin{equation}
% \begin{split}
% R_T = \widetilde{O}\Big( \gamma_{TKN} \sqrt{TN} + \gamma_{TK} N^{3/2} \sqrt{T}
% \Big).
% \end{split}
% \end{equation}
The growth rate of the maximum information gain of NTK has been characterized by previous works: $\gamma_{T}=\widetilde{\mathcal{O}}(T^{\frac{d-1}{d}})$~\citep{kassraie2021neural,vakili2021uniform}.
This implies that our regret upper bound can be further expressed as 
\[
R_T=\widetilde{O}\left(K^{\frac{(d-1)}{d}} (TN)^{\frac{3d-2}{2d}} + K^{\frac{(d-1)}{d}} T^{\frac{3d-2}{2d}} N \right) = \widetilde{O}\left(K^{\frac{(d-1)}{d}}T^{\frac{3d-2}{2d}} N \right).
\]

\subsection{Regret Upper Bound for \texttt{FN-UCB} (\texttt{Less Comm.})}
\label{app:proof:regret:fn:ucb:less:comm}
Here we explain how the proof above can be modified to derive a regret upper bound \texttt{FN-UCB} (\texttt{Less Comm.}).
To begin with, note that in terms of the regret analysis, the only difference between \texttt{FN-UCB} (\texttt{Less Comm.}) and \texttt{FN-UCB} is that $\text{UCB}^{b}_{t,i}$ of every agent $i$ is now modified to be: $\text{UCB}^{b}_{t,i}(x)=f(x;\theta_{\text{sync,NN}}) + \nu_{TK} \sqrt{\lambda} \norm{g(x;\theta_0) / \sqrt{m}}_{V_{\text{sync,NN}}^{-1}}$, in which the matrix $V_{\text{sync,NN}}^{-1}$ is obtained by: $V_{\text{sync,NN}}^{-1} = \frac{1}{N}\sum^{N}_{i=1} (V^{\text{local}}_{t,i})^{-1}$.
Note that every time the matrix $V_{\text{sync,NN}}^{-1}$ is calculated, we have that $t=t_p-1$.

\textbf{Firstly}, we prove that the modified $\text{UCB}^{b}_{t,i}$ is also a valid high-probability upper bound on the reward function $f$. To achieve this, all we need to do is to add a few steps to \eqref{eq:last:eq:ucb:1} in \textbf{Step 3} of the proof of.
Specifically, we can further analyze \eqref{eq:last:eq:ucb:1} by:
\begin{equation}
\begin{split}
|f(x;&\theta_{\text{sync,NN}}) - h(x)| \\
% &\leq |f(x;\theta_{\text{sync,NN}}) - \frac{1}{N}\sum^N_{i=1} \langle g(x;\theta_0) / \sqrt{m}, \theta_{t,i}^{\text{local}} \rangle
% + \frac{1}{N}\sum^N_{i=1} \langle g(x;\theta_0) / \sqrt{m}, \theta_{t,i}^{\text{local}} \rangle
% - h(x)|\\
% &\leq \frac{1}{N}\sum^N_{i=1} | \langle g(x;\theta_0) / \sqrt{m}, \thetaß†_{t,i}^{\text{local}} \rangle - h(x)| + \varepsilon_{\text{linear}}(m,T)\\
&\leq \frac{1}{N}\sum^N_{i=1} \nu_{TK} \sqrt{\lambda} \norm{g(x;\theta_0) / \sqrt{m}}_{(V^{\text{local}}_{i})^{-1}} + \varepsilon_{\text{linear}}(m,T)\\
&=  \nu_{TK} \frac{1}{N}\sum^N_{i=1} \sqrt{ \lambda g(x;\theta_0)^{\top} (V^{\text{local}}_{i})^{-1} g(x;\theta_0) / m} + \varepsilon_{\text{linear}}(m,T)\\
&\leq \nu_{TK}  \sqrt{ \frac{1}{N}\sum^N_{i=1} \lambda g(x;\theta_0)^{\top} (V^{\text{local}}_{i})^{-1} g(x;\theta_0) / m} + \varepsilon_{\text{linear}}(m,T)\\
&= \nu_{TK}  \sqrt{  \lambda g(x;\theta_0)^{\top} \left( \frac{1}{N}\sum^N_{i=1} (V^{\text{local}}_{i})^{-1} \right) g(x;\theta_0) / m} + \varepsilon_{\text{linear}}(m,T)\\
&= \nu_{TK}  \sqrt{  \lambda g(x;\theta_0)^{\top} \left( V_{\text{sync,NN}}^{-1} \right) g(x;\theta_0) / m} + \varepsilon_{\text{linear}}(m,T)\\
&= \nu_{TK}  \sqrt{\lambda} \norm{g(x;\theta_0) / \sqrt{m}}_{V_{\text{sync,NN}}^{-1}} + \varepsilon_{\text{linear}}(m,T).
\end{split}
\label{eq:last:eq:ucb:1:new:less:comm}
\end{equation}
The first inequality directly follows from \eqref{eq:last:eq:ucb:1}, and the second inequality results from the concavity of the square root function.
In the second last equality, we have plugged in the definition of $V_{\text{sync,NN}}^{-1} = \frac{1}{N}\sum^{N}_{i=1} (V^{\text{local}}_{t,i})^{-1}$.
% As a results, \eqref{eq:last:eq:ucb:1} holds with probability of at least $1-\delta_4 - \delta_5$, in which the error probabilities come from Equation~\eqref{eq:agg:nn:is:close:to:gp:mean} ($\delta_5$) and Lemma~\ref{lemma:confidence:bound:ucb:1:local} ($\delta_4$).
% In other words, Lemma~\ref{lemma:confidence:bound:ucb:1} (i.e., the validity of $\text{UCB}^{b}_{t,i}$) holds with probability of at least $1-\delta_4 - \delta_5$.
As a result, Lemma \ref{lemma:confidence:bound:ucb:1} which guarantees the validity of $\text{UCB}^{b}_{t,i}$ can be modified to be:
\begin{equation}
|h(x) - f(x;\theta_{\text{sync,NN}})| \leq  \nu_{TK}  \sqrt{\lambda} \norm{g(x;\theta_0) / \sqrt{m}}_{V_{\text{sync,NN}}^{-1}} + \varepsilon_{\text{linear}}(m,T), \forall x\in \mathcal{X}_{t,i}.
\label{eq:modified:ucb:b}
\end{equation}

\textbf{Secondly}, we will need the following auxiliary inequality 
% Next, we also need the following auxiliary result 
for agent $i$ and iteration $t$ in a good epoch $p \in \mathcal{E}^{\text{good}}$:
\begin{equation}
\begin{split}
\sqrt{\lambda} \norm{g(x_{t,i};\theta_0) / \sqrt{m}}_{V_{\text{sync,NN}}^{-1}} &= \sqrt{ \lambda g(x_{t,i};\theta_0)^{\top} V_{\text{sync,NN}}^{-1} g(x_{t,i};\theta_0) / m }\\
& = \sqrt{ \lambda g(x_{t,i};\theta_0)^{\top} \Big( \frac{1}{N}\sum^N_{j=1} (V^{\text{local}}_{j})^{-1} \Big) g(x_{t,i};\theta_0) / m }\\
& = \sqrt{ \frac{1}{N}\sum^N_{j=1} \lambda g(x_{t,i};\theta_0)^{\top}  (V^{\text{local}}_{j})^{-1} g(x_{t,i};\theta_0) / m }\\
& \leq \frac{1}{\sqrt{N}} \sum^N_{j=1} \sqrt{\lambda g(x_{t,i};\theta_0)^{\top}  (V^{\text{local}}_{j})^{-1} g(x_{t,i};\theta_0) / m }\\
&\leq \frac{1}{\sqrt{N}} \sum^N_{j=1} \sqrt{\lambda} \norm{g(x_{t,i};\theta_0) / \sqrt{m}}_{(V^{\text{local}}_{j})^{-1}}.
\end{split}
\label{eq:good:epoch:unpack:V:sync:new:less:comm}
\end{equation}
The first inequality is because $\sqrt{a+b} \leq \sqrt{a} + \sqrt{b}$. 
% Note that in the equation above, $V^{\text{local}}_{t_p,j}$ is indexed by $t_p$ because in epoch $p$, every $V^{\text{local}}_{t_p,j}$ used in the aggregation to obtain $V_{\text{sync,NN}}^{-1}$ is calculated using agent $j$'s local observations before iteration $t_p$, i.e., before the first iteration of epoch $p$.

\textbf{Thirdly}, we need to modify the proof of the regret upper bound for good epochs (App.~\ref{subsec:regret:good:epochs}).
Specifically, we can derive an upper bound on the instantaneous regret $r_{t,i} = h(x^*_{t,i}) - h(x_{t,i})$ for agent $i$ and iteration $t$ in a good epoch $p\in\mathcal{E}^{\text{good}}$ (in a similar way to \eqref{eq:upper:bound:inst:regret}):
\begin{equation}
\begin{split}
r_{t,i} &= h(x^*_{t,i}) - h(x_{t,i}) \\
% &= \alpha_t h(x^*_{t,i}) + (1-\alpha_t) h(x^*_{t,i}) - h(x_{t,i})\\
% &\leq \alpha_t \text{UCB}^{b}_{t,i}(x^*_{t,i}) + \alpha_t \varepsilon_{\text{linear}}(m,T) + (1-\alpha_t) \text{UCB}^{a}_{t,i}(x^*_{t,i}) - h(x_{t,i})\\
% &\leq \alpha_t \text{UCB}^{b}_{t,i}(x_{t,i}) + (1-\alpha_t) \text{UCB}^{a}_{t,i}(x_{t,i}) + \alpha_t \varepsilon_{\text{linear}}(m,T) - h(x_{t,i})\\
% &= \alpha_t \left( \text{UCB}^{b}_{t,i}(x_{t,i}) - h(x_{t,i}) \right) + (1-\alpha_t) \left( \text{UCB}^{a}_{t,i}(x_{t,i}) - h(x_{t,i})\right) + \alpha_t \varepsilon_{\text{linear}}(m,T)\\
&\leq \alpha \Big(  2 \nu_{TK}  \sqrt{\lambda} \norm{g(x_{t,i};\theta_0) / \sqrt{m}}_{V_{\text{sync,NN}}^{-1}} + \varepsilon_{\text{linear}}(m,T) \Big) + \\
&\quad (1-\alpha) \Big( 2\nu_{TKN} \sqrt{\lambda} \norm{g(x_{t,i};\theta_0) / \sqrt{m}}_{\overline{V}_{t,i}^{-1}}\Big) + \alpha \varepsilon_{\text{linear}}(m,T)\\
&\leq \alpha \Big(  2 \nu_{TK}  \frac{1}{\sqrt{N}} \sum^N_{j=1} \sqrt{\lambda} \norm{g(x_{t,i};\theta_0) / \sqrt{m}}_{(V^{\text{local}}_{j})^{-1}} + \varepsilon_{\text{linear}}(m,T) \Big) + \\
&\quad (1-\alpha) \Big( 2\nu_{TKN} \sqrt{e \lambda} \norm{g(x_{t,i};\theta_0) / \sqrt{m}}_{\widetilde{V}_{t,i}^{-1}}\Big) + \alpha \varepsilon_{\text{linear}}(m,T)\\
&= \alpha  2 \nu_{TK}  \frac{1}{\sqrt{N}} \sum^N_{j=1} \sqrt{\lambda} \norm{g(x_{t,i};\theta_0) / \sqrt{m}}_{(V^{\text{local}}_{j})^{-1}} + \\
&\quad (1-\alpha)  2\nu_{TKN} \sqrt{e \lambda} \norm{g(x_{t,i};\theta_0) / \sqrt{m}}_{\widetilde{V}_{t,i}^{-1}} + 2 \alpha \varepsilon_{\text{linear}}(m,T)\\
%&\triangleq \alpha_t  2 \nu_{TK}  \frac{1}{\sqrt{N}} \sum^N_{j=1} \overline{\sigma}^{\text{local}}_{t_p-1,j}(x_{t,i}) + (1-\alpha_t)  2\nu_{TKN} \sqrt{e} \widetilde{\sigma}_{t,i}(x_{t,i}) + 2 \alpha_t \varepsilon_{\text{linear}}(m,T).
&\triangleq (1-\alpha)  2\nu_{TKN} \sqrt{e} \widetilde{\sigma}_{t,i}(x_{t,i}) + \alpha  2 \nu_{TK}  \frac{1}{\sqrt{N}} \sum^N_{j=1} \widetilde{\sigma}^{\text{local}}_{t_p-1,j}(x_{t,i}) + 2 \alpha \varepsilon_{\text{linear}}(m,T).
\end{split}
\label{eq:upper:bound:inst:regret:new:less:comm}
\end{equation}
In the first inequality, we have made use of \eqref{eq:modified:ucb:b} which ensures the validity of the modified $\text{UCB}^{b}_{t,i}$ as a high probability upper bound on $h$. The second inequality follows from \eqref{eq:good:epoch:unpack:V:sync:new:less:comm}.
In the last equality, we have defined $\widetilde{\sigma}^{\text{local}}_{t_p-1,j}(x_{t,i})$ in the same way as \eqref{eq:upper:bound:inst:regret}.
The steps regarding the term involving $(1-\alpha)$ are the same as those from \eqref{eq:upper:bound:inst:regret}.
As a result, the only change we have made to instantaneous regret upper bound from \eqref{eq:upper:bound:inst:regret} is that in the second term, we have replaced $\frac{1}{N}$ by $\frac{1}{\sqrt{N}}$.
Further propagating this change through the proof for the regret upper bound for all good epochs (App.~\ref{app:subsubsec:upper:bound:sum:of:second:term:good:epochs} and App.~\ref{app:subsubsec:good:epochs:putting:things:together}), we have that:
% \begin{equation}
% \begin{split}
% \sum^N_{i=1} \sum_{p \in \mathcal{E}^{\text{good}}} \sum_{t\in\mathcal{T}^{(p)}} &\alpha 2 \nu_{TK}  \frac{1}{\sqrt{N}} \sum^N_{j=1} \widetilde{\sigma}^{\text{local}}_{t_p-1,j}(x_{t,i}) \\
% &\leq 2 \nu_{TK}  \frac{1}{\sqrt{N}}  \sum^N_{i=1} \sum^N_{j=1} \sqrt{2} \sqrt{\kappa_0} \sqrt{ T \lambda \left[\widetilde{d}_j \log(1 + TK/\lambda) ) + 1\right] }\\
% &= 2\sqrt{2} \nu_{TK} \sqrt{\kappa_0} \sqrt{N} \sum^N_{j=1} \sqrt{ T \lambda \left[\widetilde{d}_j \log(1 + TK/\lambda) ) + 1\right] }.
% \end{split}
% \end{equation}
\begin{equation}
\begin{split}
R_T^{\text{good}} = \sum^N_{i=1} \sum_{t \in \mathcal{T}^{\text{good}}} r_{t,i} =\widetilde{O}\left(\widetilde{d} \sqrt{TN} + \widetilde{d}_{\max} N^{3/2} \sqrt{T} + TN\varepsilon_{\text{linear}}(m,T)
\right).
\end{split}
\label{eq:final:regret:upper:bound:in:good:epochs:new:less:comm}
\end{equation}
\textbf{Lastly}, also note that the regret upper bound for the bad epochs (i.e., the proof in App.~\ref{subsec:regret:bad:epochs}) remains unchanged.
Therefore, the final regret upper bound for \texttt{FN-UCB} (\texttt{Less Comm.}) is
\begin{equation}
\begin{split}
R_T &= R^{\text{good}}_T+R^{\text{bad}}_T \\
&= \widetilde{O}\Big(\widetilde{d} \sqrt{TN} + \widetilde{d}_{\max} N^{3/2} \sqrt{T} + TN\varepsilon_{\text{linear}}(m,T) + \widetilde{d}\sqrt{TN}
\Big)\\
&=\widetilde{O}\Big(\widetilde{d} \sqrt{TN} + \widetilde{d}_{\max} N^{3/2} \sqrt{T} + TN\varepsilon_{\text{linear}}(m,T)
\Big)\\
&=\widetilde{O}\Big(\widetilde{d} \sqrt{TN} + \widetilde{d}_{\max} N^{3/2} \sqrt{T}
\Big).
\end{split}
\label{app:eq:final:regret:upper:bound:eq:1:new:less:comm}
\end{equation}

\section{Proof of Upper Bound on Communication Complexity (Theorem~\ref{theorem:communication})}
\label{sec:communication:complexity}
In this section, we derive an upper bound on the communication complexity (i.e., the total number of communication rounds) of our \texttt{FN-UCB} algorithm (including its variant \texttt{FN-UCB} (\texttt{Less Comm.})).
Define $\zeta \triangleq \sqrt{D T / R}$. An immediate implication is that there can be at most $\lceil T/\zeta \rceil$ epochs whose length is larger than $\zeta$. Next, we try to derive an upper bound on the number of epochs whose length is smaller than $\zeta$.

Note that if an epoch $p$ contains less than $\zeta$ iterations, then because of our criterion to start a communication round (line 10 of Algo.~\ref{algo:agent}), we have that $\log\frac{\text{det} V_{p}}{\text{det} V_{p-1}} > \frac{D}{\zeta}$. Also recall that equation~\eqref{eq:sum:of:log:det} (Appendix \ref{subsubsec:definition:good:bad:epochs}) tells us that:
\begin{equation}
\begin{split}
\sum^{P-1}_{p=0} \log \frac{\text{det} V_{p+1}}{\text{det} V_p } \leq R'  \leq  \overline{R},
\end{split}
\end{equation}
with probability of at least $1-\delta_1 \geq 1-\delta$.
Therefore, there can be at most $\lceil \frac{\overline{R}}{D / \zeta} \rceil = \lceil \frac{\overline{R}\zeta}{D} \rceil$ such epochs whose length is smaller than $\zeta$.
As a result, the total number of epochs can be upper-bounded by:
\begin{equation}
\lceil T/\zeta \rceil + \lceil \frac{\overline{R}\zeta}{D} \rceil = \mathcal{O}\Big(\sqrt{\frac{T\overline{R}}{D}}\Big).
\end{equation}

Recall that $\overline{R}=\widetilde{\mathcal{O}}( \widetilde{d} )$ (App.~\ref{subsubsec:definition:good:bad:epochs}). 
%If we let the width $m$ of the NN be large enough, then we can make the second term arbitrarily small, i.e., smaller than any constant we specify. So $R$ can be written as $R=\mathcal{O}(\gamma_{TN})$, and hence 
Therefore, with probability of at least $1-\delta_1\geq1-\delta$, the total number of epochs can be upper-bounded by $\widetilde{\mathcal{O}}(\sqrt{\frac{T\widetilde{d}}{D}})$.

% Since we have chosen $D=\widetilde{\mathcal{O}}(\frac{T}{N \widetilde{d}})$ (line 1 of Algo.~\ref{algo:agent}), therefore, the total number of epochs can be upper-bounded by $\widetilde{\mathcal{O}}(\sqrt{\frac{T\widetilde{d}}{\frac{T}{N \widetilde{d}}}})=\widetilde{\mathcal{O}}(\widetilde{d}\sqrt{N})=\widetilde{\mathcal{O}}\left( \gamma_{TKN}\sqrt{N} \right)$, which is sub-linear in $T$ since $\gamma_{TKN}=\widetilde{\mathcal{O}}((TKN)^{\frac{d-1}{d}})$.
Since we have chosen $D=\widetilde{\mathcal{O}}(\frac{T}{N \widetilde{d}})$ (line 1 of Algo.~\ref{algo:agent}), therefore, the total number of epochs can be upper-bounded by $\widetilde{\mathcal{O}}(\sqrt{\frac{T\widetilde{d}}{\frac{T}{N \widetilde{d}}}})=\widetilde{\mathcal{O}}(\widetilde{d}\sqrt{N})$. Now we can further make use of the relationship between $\widetilde{d}$ and $\gamma_{TKN}$: $\widetilde{d} \leq  2\gamma_{TKN} / \log(1+TKN/\lambda)$, which allows us to show that the worst-case communication complexity is upper-bounded by: $\widetilde{\mathcal{O}}(\widetilde{d}\sqrt{N})=\widetilde{\mathcal{O}}\left( \gamma_{TKN}\sqrt{N} \right)=\widetilde{\mathcal{O}}\left( (TKN)^{\frac{d-1}{d}}\sqrt{N} \right)=\widetilde{O}(T^{\frac{d-1}{d}} K^{\frac{d-1}{d}} N^{\frac{3d-2}{2d}})$, which is still sub-linear in $T$ even in the worst case.

The proof here, and hence Theorem~\ref{theorem:communication}, makes use of Lemma~\ref{lemma:bound:error:NTK:matrix}. Therefore, we only need condition 1 on $m$ listed in App.~\ref{app:conditions:on:m} to hold, and do not require any condition on $\eta$ and $J$.

\section{More Experimental Details}
\label{app:more:experimental:details}
Our code can be found at: \url{https://github.com/daizhongxiang/Federated-Neural-Bandits}.

Some of the experimental details (e.g., the number of layers and the width $m$ of the NN used in every experiment) are already described in the main text (Sec.~\ref{sec:experiments}).
Following the works of \cite{zhang2020neural,zhou2020neural}, when training the NN (line 14 of Algo.~\ref{algo:agent}) for agent $i$, we use the NN parameters resulting from the last gradient descent training of agent $i$ (instead of $\theta_0$) as the initial parameters, in order to accelerate the training procedure. 
Every time we train an NN, we use stochastic gradient descent to train the NN for $30$ iterations with a learning rate of $0.01$.
To save computational cost, we stop training the NNs after $2000$ iterations, i.e., after $2000$ iterations, all NN parameters are no longer updated.
Also to reduce the computational cost, when checking the criterion in line 11 of Algo.~\ref{algo:agent}, we diagonalize (i.e., only keep the diagonal elements of) the two matrices for which we need to calculate the determinant.
Our experiments are run on a server with 96 CPUs, an NVIDIA A100 GPU with a memory of 40GB, a RAM of 256GB, running the Ubuntu system.

The \texttt{shuttle} dataset is publicly available at \url{https://archive.ics.uci.edu/ml/datasets/Statlog+(Shuttle)} and contains no personally identifiable information or offensive content. It includes 58000 instances, has an input dimension of $d=9$ and contains $K=7$ classes/arms. As a result, according to the way in which the contexts are constructed (Sec.~\ref{subsec:real:experiments}), every context feature vector has a dimension of $9 \times 7=63$.
The \texttt{magic telescope} dataset is publicly available at \url{https://archive.ics.uci.edu/ml/datasets/magic+gamma+telescope} and contains no personally identifiable information or offensive content.
The dataset contains 19020 instances, has an input dimension of $d=10$ and $K=2$ classes/arms. As a result, every context feature vector has a dimension of $10 \times 2=20$.

\begin{figure}
     \centering
     \begin{tabular}{cc}
         \includegraphics[width=0.35\linewidth]{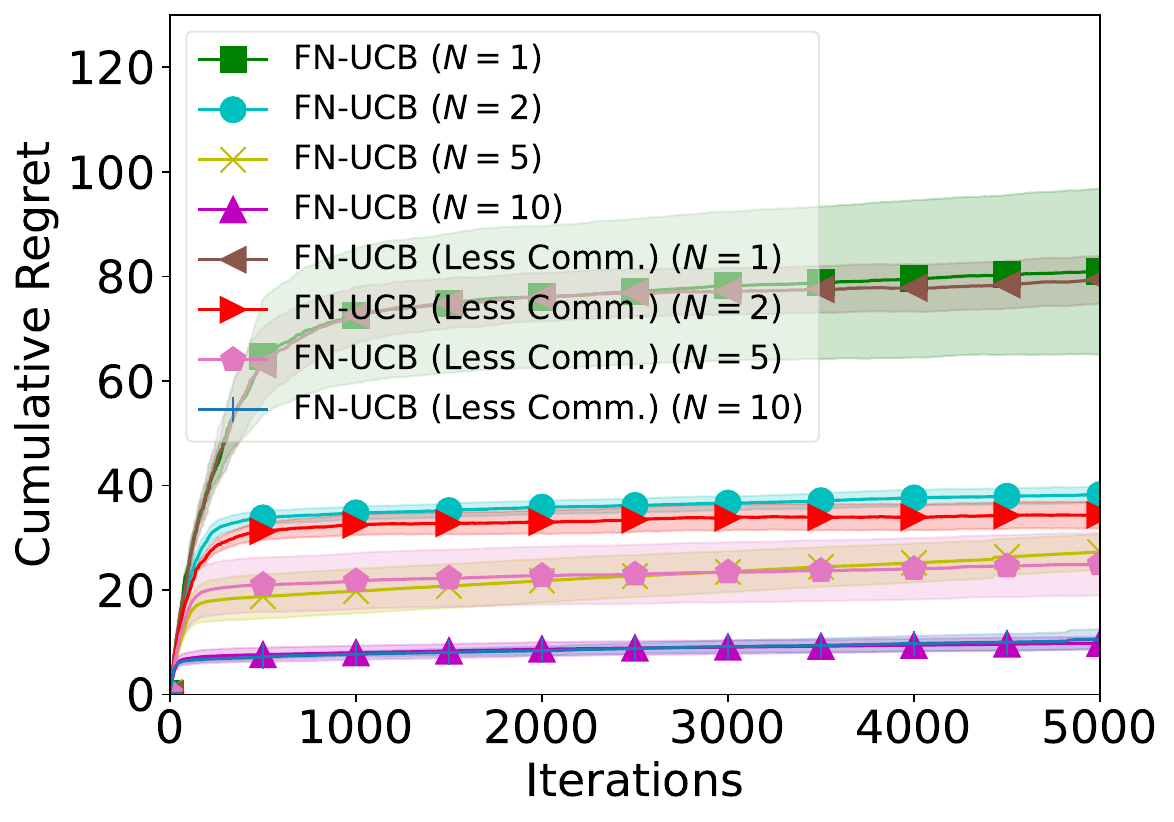} & \hspace{-6mm} 
         \includegraphics[width=0.35\linewidth]{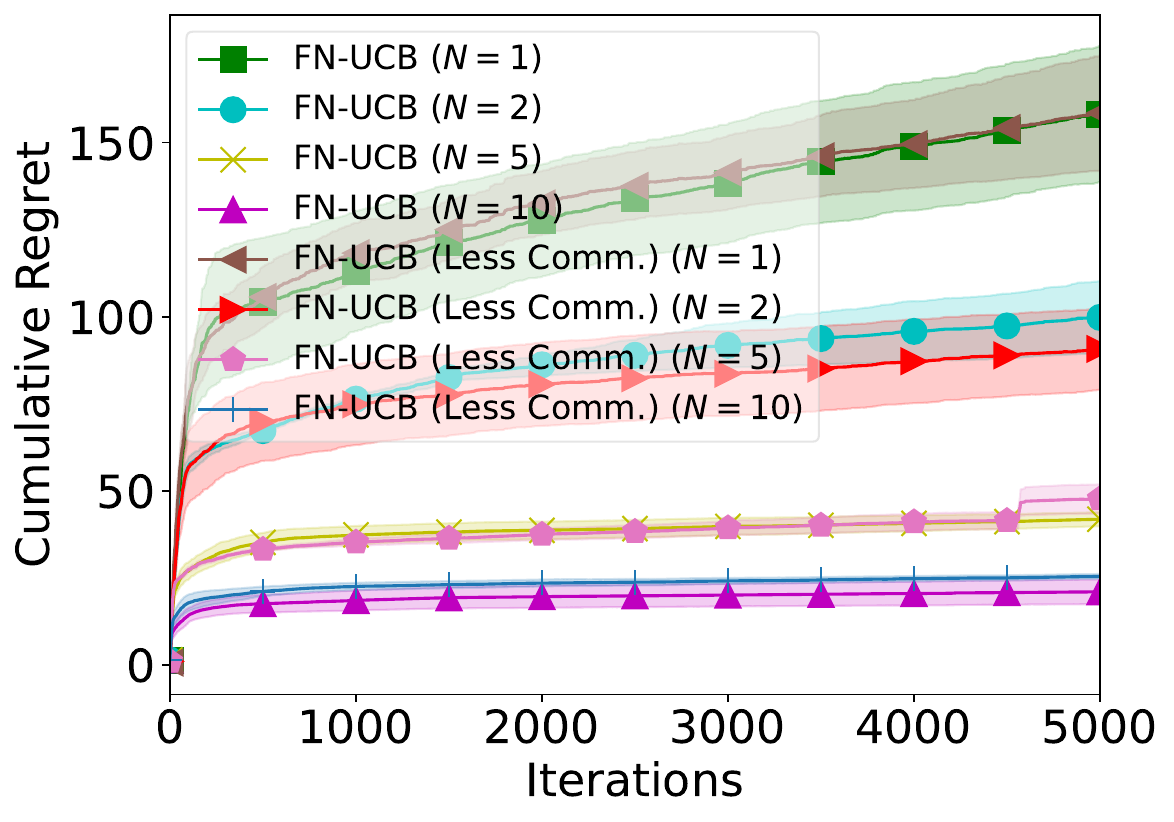}\\
         {\hspace{-2mm} (a) \texttt{cosine}} & {\hspace{-5mm} (b) \texttt{square}}
     \end{tabular}
% \vspace{-2.5mm}
     \caption{
    Cumulative regret of \texttt{FN-UCB} and \texttt{FN-UCB} (\texttt{Less Comm.}) for the \texttt{cosine} and \texttt{square} functions.
    Their performances are very similar for both functions.
     }
     \label{fig:exp:synth:with:less:comm}
% \vspace{-6mm}
\end{figure}

%\paragraph{Comparison with Linear and Kernelized Contextual Bandit Algorithms.}
When comparing with Linear-UCB, Linear TS, Kernelized UCB and Kernelized TS, we follow the work of \cite{zhang2020neural} to set $\lambda=1$ and perform a grid search within $\nu\in\{1, 0.1, 0.01\}$.
The results showing comparisons with these algorithms, for both the synthetic experiments (Sec.~\ref{subsec:synthetic:experiments}) and real-world experiments (Sec.~\ref{subsec:real:experiments}), are presented in Fig.~\ref{fig:exp:with:linear:kernel}.
The figures show that both linear and kernelized contextual bandit algorithms are outperformed by neural contextual bandit algorithms, which is consistent with the observations from \cite{zhang2020neural,zhou2020neural}.
\begin{figure}
     \centering
     \begin{tabular}{cc}
         \includegraphics[width=0.38\linewidth]{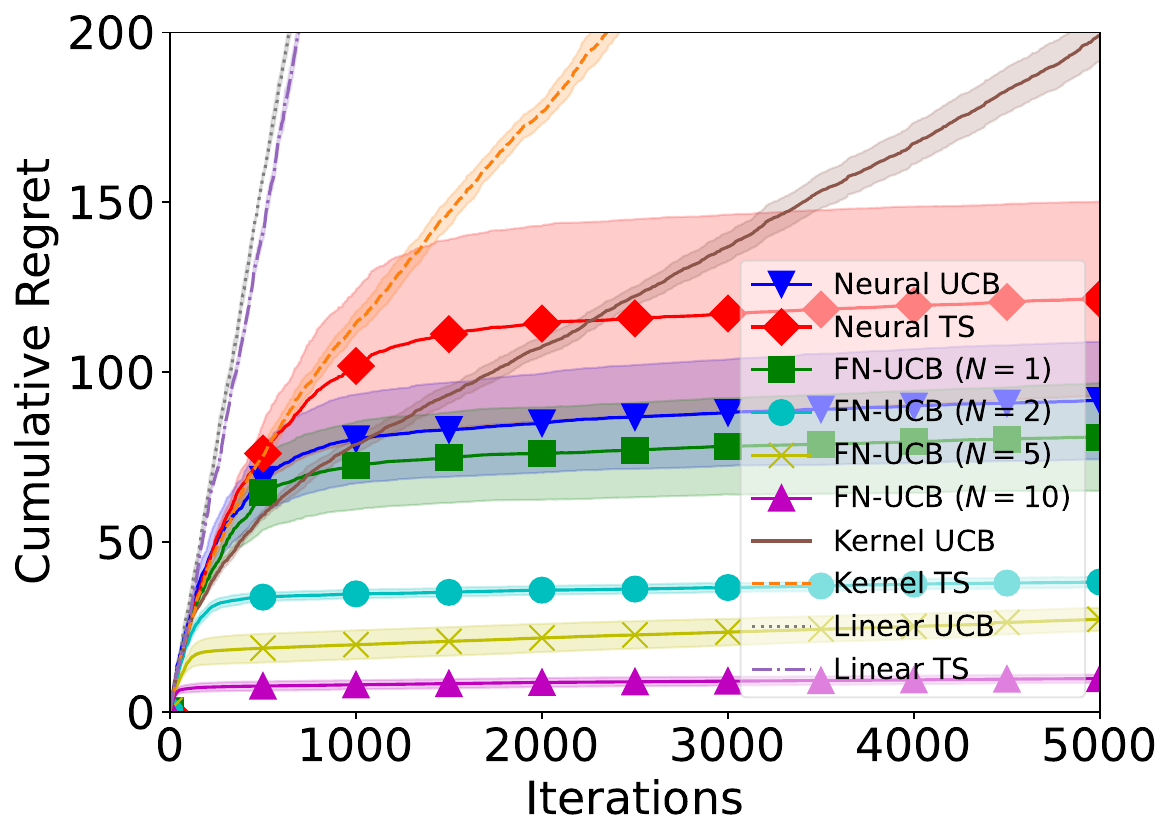} & \hspace{-7.5mm} 
         \includegraphics[width=0.38\linewidth]{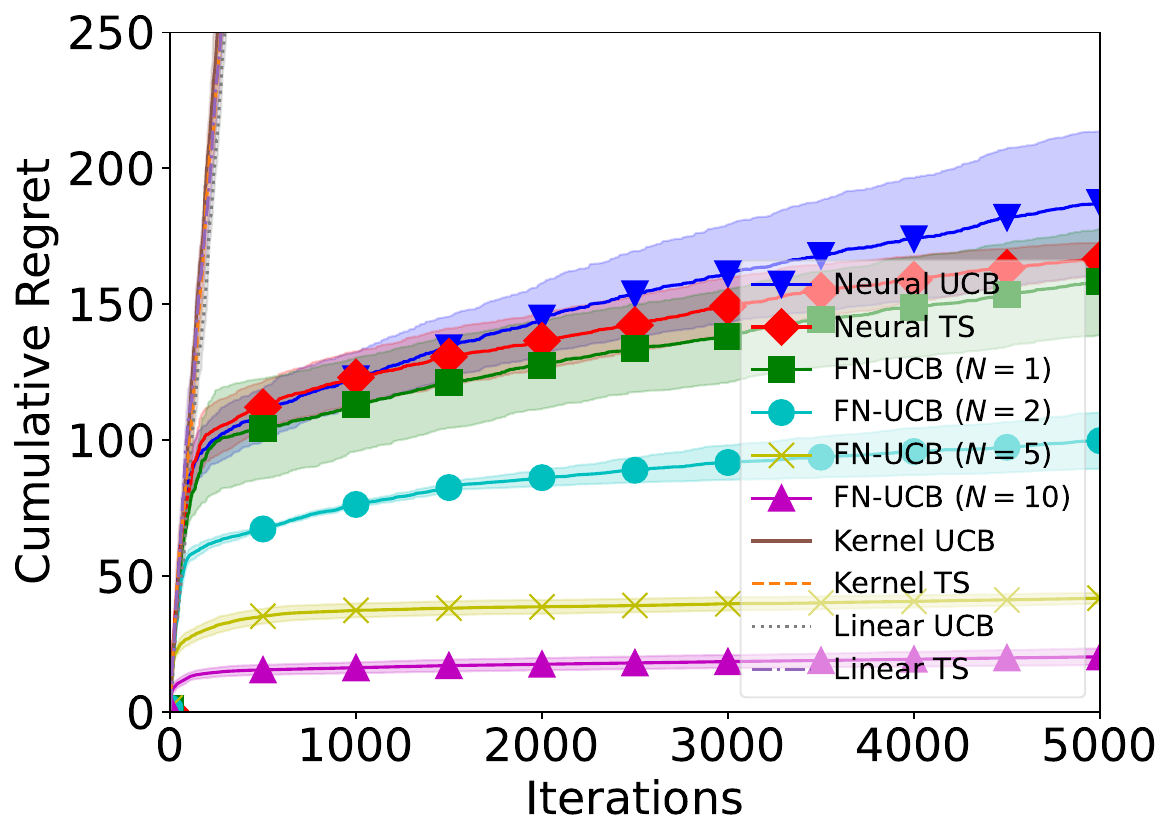}\\
         {(a) \texttt{cosine}} & {(b) \texttt{square}}\\
         \includegraphics[width=0.38\linewidth]{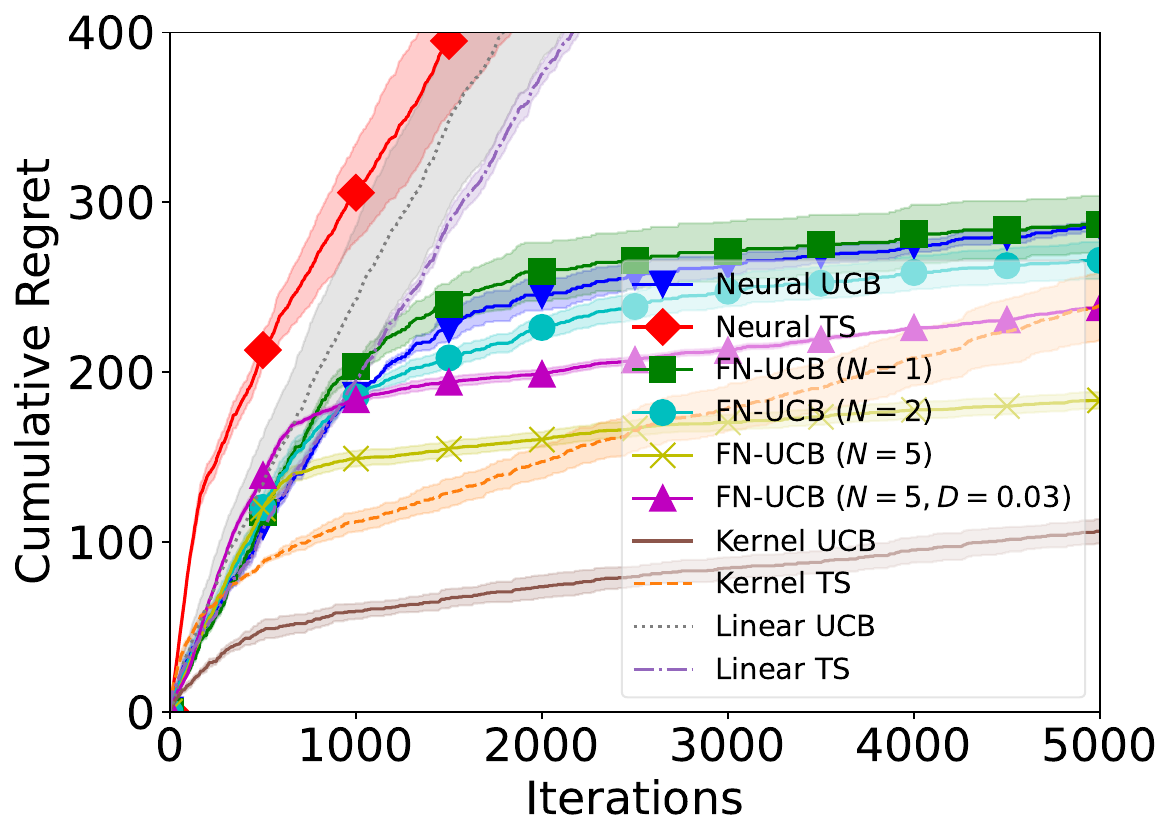}& \hspace{-7.5mm}
         \includegraphics[width=0.38\linewidth]{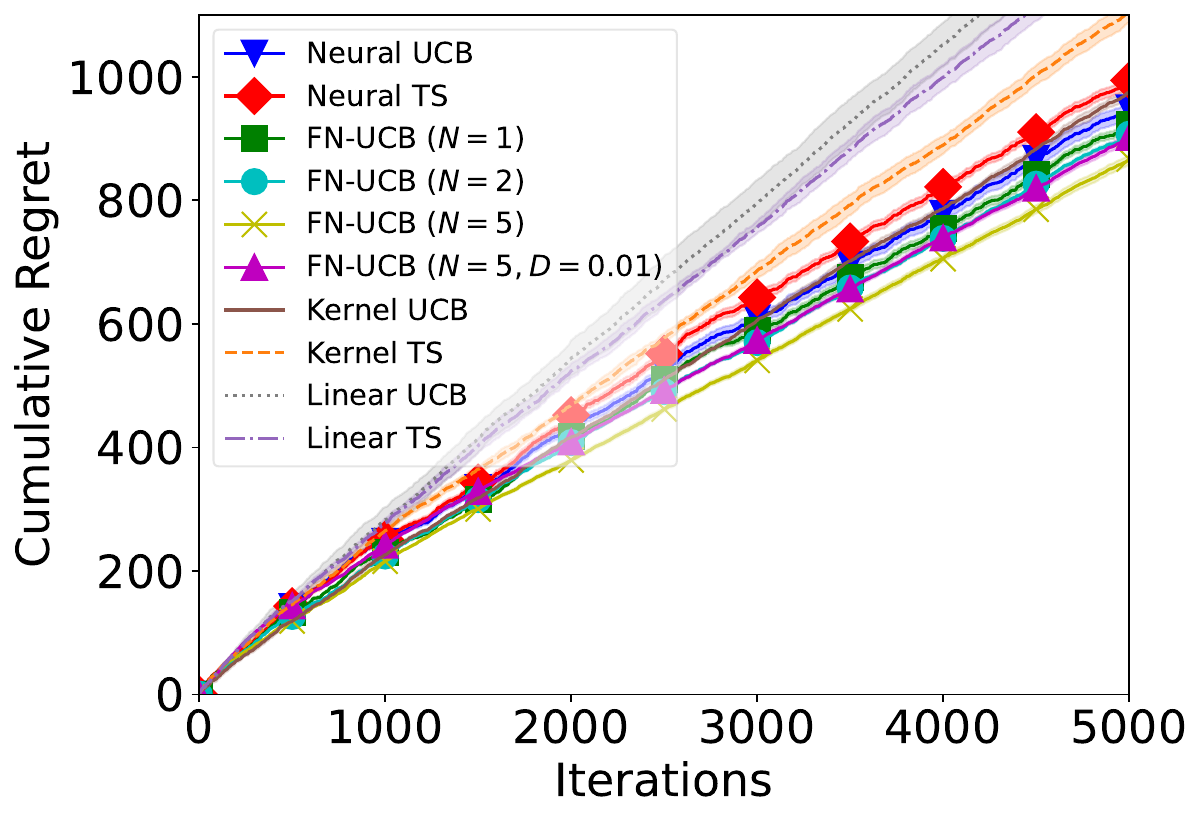}\\
         {(c) \texttt{shuttle}} & {(d) \texttt{magic telescope}}
     \end{tabular}
%\vspace{-2.5mm}
     \caption{
     Cumulative regrets for the (a) \texttt{cosine}, (b) \texttt{square}, (c) \texttt{shuttle} (with diagonalization), and (d) \texttt{magic telescope} experiments, with additional comparisons with Linear UCB, Linear TS, Kernel UCB and Kernel TS.
     }
     \label{fig:exp:with:linear:kernel}
%\vspace{-6mm}
\end{figure}
%\begin{figure}
%     \centering
%     \begin{tabular}{cccc}
%         \includegraphics[width=0.27\linewidth]{figures/synth_cosine_diff_Ns_with_kernel_linear.pdf} & \hspace{-7.5mm} 
%         \includegraphics[width=0.27\linewidth]{figures/synth_square_diff_Ns_with_linear_kernel.pdf}& \hspace{-7.5mm}
%         \includegraphics[width=0.27\linewidth]{figures/shuttle_diff_Ns_with_diag_with_linear_kernel.pdf}& \hspace{-7.5mm}
%         \includegraphics[width=0.27\linewidth]{figures/magic_diff_Ns_no_diag_with_linear_kernel.pdf}\\
%         {(a)} & {(b)} & {(c)} & {(d)}
%     \end{tabular}
%%\vspace{-2.5mm}
%     \caption{
%     }
%     \label{fig:exp:with:linear:kernel}
%%\vspace{-6mm}
%\end{figure}

\newpage

% {
% \color{red}
We have additionally evaluated the empirical impact of the technique of diagonalization of the matrices (Sec.~\ref{subsec:comm:cost}), using the \texttt{shuttle} dataset and a fixed width of $m=20$ for the NN. The results (Fig.~\ref{fig:exp:shuttle:compare:diag:no:diag}) show that for the same width of the NN, the technique of diagonalization indeed results in worse performances.
However, also note that diagonalization allows us to afford a larger value of $m$ in a computationally feasible way, which can lead to better performances than using a smaller $m$ without diagonalization. This is corroborated by our empirical results in Fig.~\ref{fig:exp:real}a and \ref{fig:exp:real}c, because the regrets in Fig.~\ref{fig:exp:real}c ($m$=50 , with diagonalization) are in general smaller than the regrets in Fig.~\ref{fig:exp:real}a ($m$=20 , without diagonalization), and the computational cost of Fig.~\ref{fig:exp:real}c ($244.9$ seconds) is smaller than that of Fig.~\ref{fig:exp:real}a ($361.8$ seconds). Furthermore, using $m=50$ without diagonalization would incur a significantly larger computational cost ($3134.3$ seconds). These results demonstrate the practical usefulness of diagonalization.
% }

\begin{figure}
     \centering
     \begin{tabular}{c}
         \includegraphics[width=0.45\linewidth]{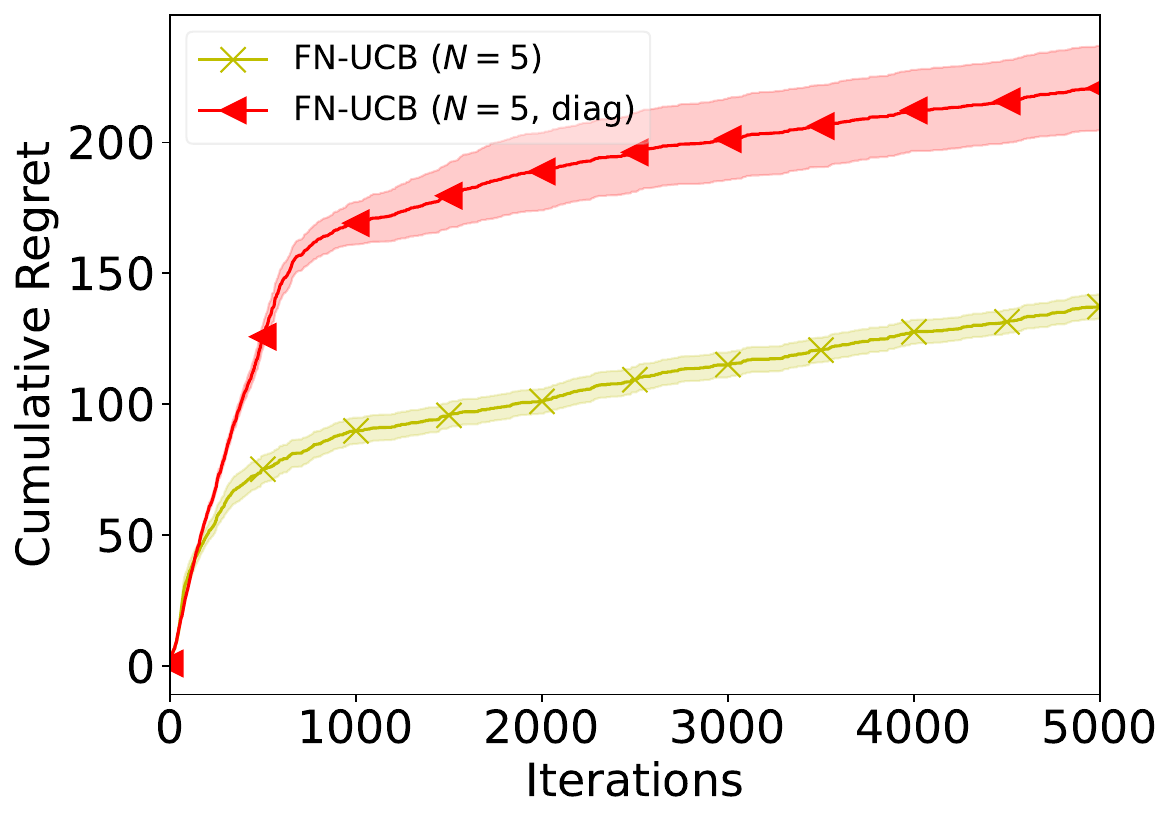}\\
         {\hspace{-2mm} \texttt{shuttle}} 
         % & {\hspace{-5mm} (b) \texttt{square}}
     \end{tabular}
% \vspace{-2.5mm}
     \caption{
%      {
% \color{red}
     Comparison between the performances without (yellow) and with (red) diagonalization, using $m=20$ with the \texttt{shuttle} dataset. The results show that using an NN with the same width $m=20$, diagonalization indeed deteriorates the performances.
    % Cumulative regret of \texttt{FN-UCB} and \texttt{FN-UCB} (\texttt{Less Comm.}) for the \texttt{cosine} and \texttt{square} functions.
    % Their performances are very similar for both functions.
     % }
     }
     \label{fig:exp:shuttle:compare:diag:no:diag}
% \vspace{-6mm}
\end{figure}

% {
% \color{red}
We have also visualized the empirical scaling of the final average cumulative regret (after $5000$ iterations) in terms of the number $N$ of agents, using the \texttt{cosine} and \texttt{square} experiments. The results (Fig.~\ref{fig:exp:scaling:in:N}) demonstrate that the average cumulative regret (averaged across all $N$ agents) is indeed decreasing as the number $N$ of agents increases.
% }
\begin{figure}
     \centering
     \begin{tabular}{cc}
         \includegraphics[width=0.38\linewidth]{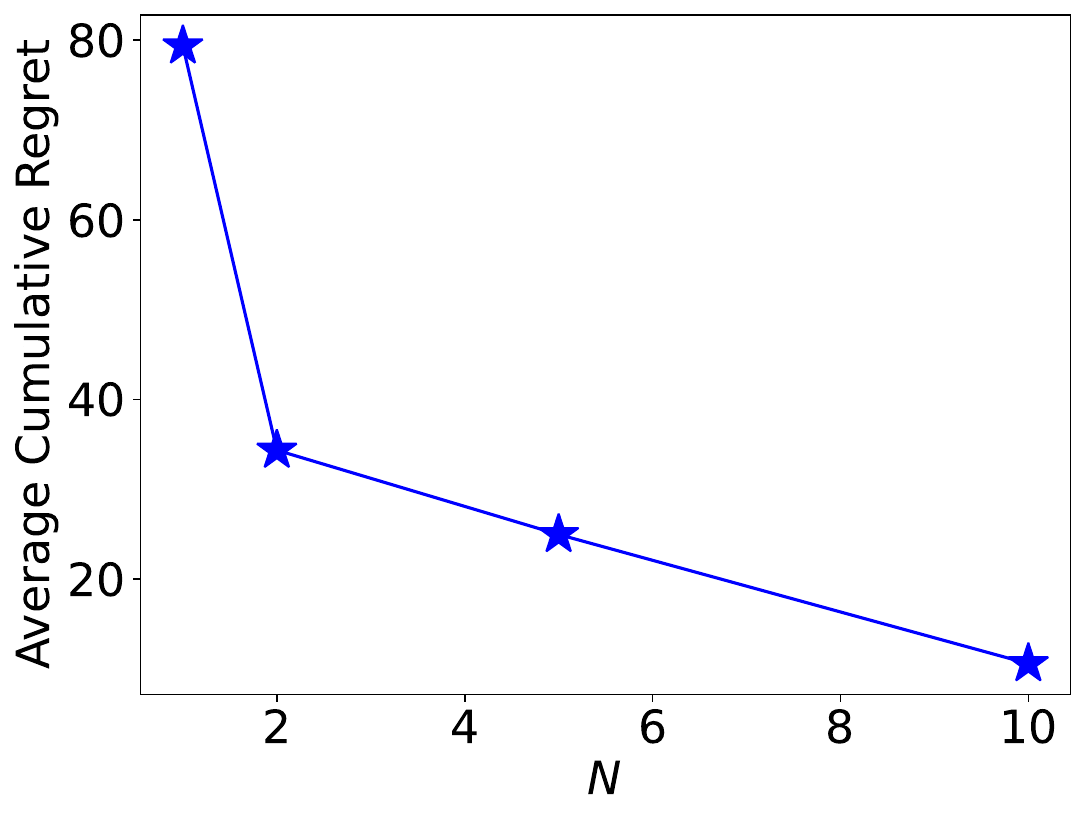} & \hspace{-3.5mm} 
         \includegraphics[width=0.38\linewidth]{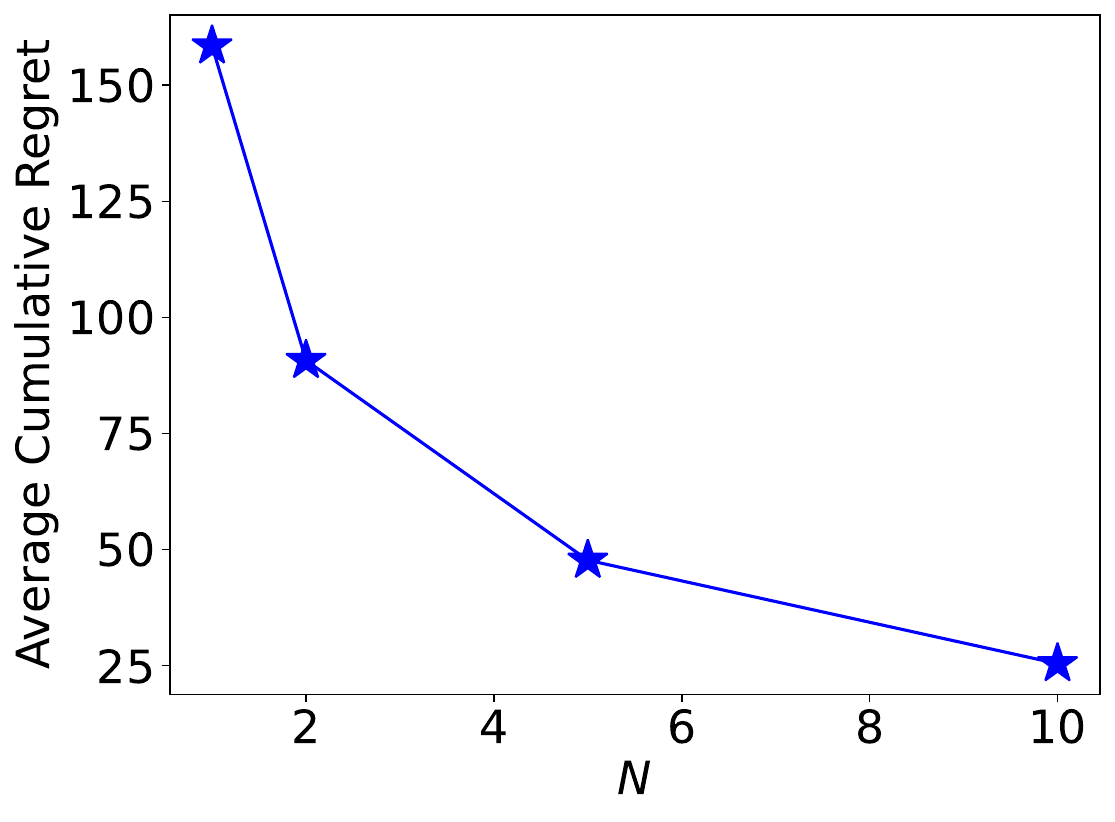}\\
         {(a) \texttt{cosine}} & {(b) \texttt{square}}
     \end{tabular}
%\vspace{-2.5mm}
     \caption{
     % {\color{red}
     The scaling of the final average cumulative regret after $5000$ iterations (averaged across all $N$ agents) in terms of the number $N$ of agents, using the \texttt{cosine} and \texttt{square} experiments. The results correspond to Fig.~\ref{fig:exp:synth} a and Fig.~\ref{fig:exp:synth} b, respectively.
     }
     % }
     \label{fig:exp:scaling:in:N}
%\vspace{-6mm}
\end{figure}

% {\color{red}
\section{Extended Analysis for The General Algorithm}
\label{app:extended:analysis:general:algo}
Recall that it has been mentioned at the beginning of Sec.~\ref{subsec:theoretical:results} that our main regret analysis (Theorem \ref{theorem:regret}) has focused on a simpler version of our \texttt{FN-UCB} algorithm, in which we only choose the value of $\alpha$ using the method described in Sec.~\ref{subsec:weight:two:ucbs} in the first iteration of every epoch and set $\alpha=0$ in the other iterations.
Here, we show how our regret analysis can be extended to derive a regret upper bound for the general \texttt{FN-UCB} algorithm, in which we choose $\alpha$ using the method described in Sec.~\ref{subsec:weight:two:ucbs} in every iteration, i.e., we do not set $\alpha=0$ in any iteration.
To achieve this, 
we need an additional assumption of an upper bound on the amount of new information collected by every agent $i$ in every epoch $p$.
% we need an additional assumption on the contexts $\{x^k_{t,i}\}_{k\in[K],t\in[T],i\in[N]}$ such that the new information collected by every agent in every epoch is upper-bounded. 
Specifically, we assume that 
\begin{equation}
    \frac{\text{det} V^{\text{local}}_{t_p+E_p-2,i} }{\text{det} V^{\text{local}}_{t_p-1,i}} \leq \overline{D}, \forall i\in[N],p\in[P]
\label{eq:additional:assumption}
\end{equation}
for a constant $\overline{D}\geq 1$.
This can in fact be viewed as an additional property of the sequence of contexts for each agent. Intuitively, if the contexts for each agent are received in such an order that similar contexts also arrive in similar iterations, then the constant $\overline{D}$ is likely to be small. This can be seen as a "stationarity" property of the sequence of contexts, which is reasonable in many practical scenarios. For example, in a healthcare application, the patients arriving within the same time period are likely to have similar characteristics due to factors such as the local transmission of a seasonal flu. In addition, another scenario where $\overline{D}$ is likely to be small is when every agent has some previously observed offline contexts before running our algorithm. If these offline contexts have a good coverage of the space of contexts, then conditioned on these offline contexts, the newly collected information by every agent in every epoch is highly likely to be small.
% Intuitively, if the contexts for each agent arrive in such an order that similar contexts also arrive in similar iterations, then the constant $\overline{D}$ is likely to be small. This can be seen as a "stationarity" assumption on the sequence of contexts, which is reasonable in many practical scenarios, e.g., in a healthcare application, the patients arriving within the same time period are likely to have similar characteristics due to factors such as a seasonal flu.

With this additional assumption, the most important step in the proof that we need to modify is the proof in Appendix~\ref{app:subsubsec:upper:bound:sum:of:second:term:good:epochs}, in which we proved an upper bound on the sum of the second term of \eqref{eq:upper:bound:inst:regret}.
To begin with, $\forall t=t_p,\ldots,t_p+E_p-1$, we have that
\begin{equation}
\begin{split}
\widetilde{\sigma}^{\text{local}}_{t_p-1,j}(x_{t,i}) &\stackrel{(a)}{=} \sqrt{\lambda} \norm{g(x_{t,i};\theta_0) / \sqrt{m}}_{(V^{\text{local}}_{t_p-1,j})^{-1}}\\
&=\sqrt{\lambda g(x_{t,i};\theta_0)^{\top} (V^{\text{local}}_{t_p-1,j})^{-1} g(x_{t,i};\theta_0) / m}\\
&\stackrel{(b)}{\leq} \sqrt{\lambda g(x_{t,i};\theta_0)^{\top} (V^{\text{local}}_{t-1,j})^{-1} g(x_{t,i};\theta_0) / m \frac{\text{det} V^{\text{local}}_{t-1,j}}{\text{det} V^{\text{local}}_{t_p-1,j}} }\\
&\stackrel{(c)}{\leq} \sqrt{\lambda g(x_{t,i};\theta_0)^{\top} (V^{\text{local}}_{t-1,j})^{-1} g(x_{t,i};\theta_0) / m \frac{\text{det} V^{\text{local}}_{t_p+E_p-1-1,j}}{\text{det} V^{\text{local}}_{t_p-1,j}} }\\
&\stackrel{(d)}{\leq} \sqrt{\lambda g(x_{t,i};\theta_0)^{\top} (V^{\text{local}}_{t-1,j})^{-1} g(x_{t,i};\theta_0) / m \overline{D} }\\
&\stackrel{(e)}{=}\sqrt{\overline{D}} \widetilde{\sigma}^{\text{local}}_{t-1,j}(x_{t,i}).
\end{split}
\label{eq:good:epochs:second:term:eq:1:new:bound:diff:gp:post:std}
\end{equation}
Step $(a)$ has made use of the definition of $\widetilde{\sigma}^{\text{local}}_{t_p-1,j}(x_{t,i})$ (see the paragraph below \eqref{eq:upper:bound:inst:regret}), step $(b)$ results from Lemma 12 of~\cite{abbasi2011improved}, step $(c)$ follows because $V^{\text{local}}_{t_p+E_p-1-1,j}$ contains more information than $V^{\text{local}}_{t-1,j}$ $\forall t=t_p,\ldots,t_p+E_p-1$, step $(d)$ follows from \eqref{eq:additional:assumption}, and step $(e)$ has again made use of the definition of $\widetilde{\sigma}^{\text{local}}_{t_p-1,j}(x_{t,i})$.

Using \eqref{eq:good:epochs:second:term:eq:1:new:bound:diff:gp:post:std}, we can modify the proof in \eqref{eq:good:epochs:second:term:eq:1} (Appendix~\ref{app:subsubsec:upper:bound:sum:of:second:term:good:epochs}):
\begin{equation}
\begin{split}
\sum^N_{i=1} \sum_{p \in \mathcal{E}^{\text{good}}} \sum_{t\in\mathcal{T}^{(p)}} \alpha  2 \nu_{TK}  \frac{1}{N} &\sum^N_{j=1} \widetilde{\sigma}^{\text{local}}_{t_p-1,j}(x_{t,i}) \leq 2 \nu_{TK}  \frac{1}{N}  \sum^N_{i=1} \sum_{p \in [P]} \sum_{t\in\mathcal{T}^{(p)}} \alpha  \sum^N_{j=1} \widetilde{\sigma}^{\text{local}}_{t_p-1,j}(x_{t,i})\\
&\leq 2 \nu_{TK}  \frac{1}{N}  \sum^N_{i=1} \sum^N_{j=1} \sum_{p \in [P]} \sum^{t_p+E_p-1}_{t=t_p}  \alpha \widetilde{\sigma}^{\text{local}}_{t_p-1,j}(x_{t,i})\\
&\stackrel{(a)}{\leq} 2 \nu_{TK}  \frac{1}{N}  \sum^N_{i=1} \sum^N_{j=1} \sum_{p \in [P]} \sum^{t_p+E_p-1}_{t=t_p} \widetilde{\sigma}^{\text{local}}_{t_p-1,j}(x_{t,j})\\
&\stackrel{(b)}{\leq} 2 \nu_{TK}  \frac{1}{N}  \sum^N_{i=1} \sum^N_{j=1} \sum_{p \in [P]} \sum^{t_p+E_p-1}_{t=t_p} \sqrt{\overline{D}} \widetilde{\sigma}^{\text{local}}_{t-1,j}(x_{t,j})\\
&= \sqrt{\overline{D}} 2 \nu_{TK}  \frac{1}{N}  \sum^N_{i=1} \sum^N_{j=1} \sum^T_{t=1}  \widetilde{\sigma}^{\text{local}}_{t-1,j}(x_{t,j}).
% &\stackrel{(d)}{\leq} 2 \nu_{TK}  \frac{1}{N}  \sum^N_{i=1} \sum^N_{j=1} \sum^T_{t=1}  \widetilde{\sigma}^{\text{local}}_{t-1,j}(x_{t,j})\\
\end{split}
\label{eq:good:epochs:second:term:eq:1:new:general:new}
\end{equation}
Step $(a)$ follows from the same reasoning as step $(c)$ of \eqref{eq:good:epochs:second:term:eq:1}, step $(b)$ has made use of \eqref{eq:good:epochs:second:term:eq:1:new:bound:diff:gp:post:std}, and all other steps follow the same corresponding steps of \eqref{eq:good:epochs:second:term:eq:1}.

As a result, by comparing the modified \eqref{eq:good:epochs:second:term:eq:1:new:general:new} with the original \eqref{eq:good:epochs:second:term:eq:1}, the only modification to the result in \eqref{eq:good:epochs:second:term:eq:1} is the additional multiplicative term of $\sqrt{\overline{D}}$. Therefore, after propagating this modification to all the analysis in Appendix~\ref{app:subsubsec:upper:bound:sum:of:second:term:good:epochs}, we have that a multiplicative term of $\sqrt{\overline{D}}$ will also be introduced into \eqref{eq:final:upper:bound:on:second:term:eq:13}. Subsequently, the upper bound on the total regrets from all good epochs (i.e., \eqref{eq:final:regret:upper:bound:in:good:epochs}) will be correspondingly modified to be:
\begin{equation}
\begin{split}
R_T^{\text{good}} =\widetilde{O}\left(\widetilde{d} \sqrt{TN} + \sqrt{\overline{D}} \widetilde{d}_{\max} N \sqrt{T} + TN\varepsilon_{\text{linear}}(m,T)
\right).
\end{split}
\label{eq:final:regret:upper:bound:in:good:epochs:new:general}
\end{equation}

Next, we also need to modify the proof of the upper bound on the total regrets from all \emph{bad epochs} (Appendix \ref{subsec:regret:bad:epochs}). Following the roadmap of Appendix \ref{subsec:regret:bad:epochs}, we start by upper-bounding the total regrets from a particular bad epoch $p$:
\begin{equation}
\begin{split}
R^{[p]} &= \sum^N_{i=1}\sum^{t_p + E_p - 1}_{t=t_p} r_{t,i} = \sum^N_{i=1}\sum^{t_p + E_p - 1}_{t=t_p} [\alpha h(x^*_{t,i}) + (1-\alpha) h(x^*_{t,i}) - h(x_{t,i})]\\
&\stackrel{(a)}{\leq} \sum^N_{i=1}\sum^{t_p + E_p - 1}_{t=t_p} \Big[ \alpha \text{UCB}^{b}_{t,i}(x^*_{t,i}) + \alpha \varepsilon_{\text{linear}}(m,T) + (1-\alpha) \text{UCB}^{a}_{t,i}(x^*_{t,i}) - h(x_{t,i}) \Big]\\
&\stackrel{(b)}{\leq} \sum^N_{i=1}\sum^{t_p + E_p - 1}_{t=t_p} \Big[ \alpha \text{UCB}^{b}_{t,i}(x_{t,i}) + \alpha \varepsilon_{\text{linear}}(m,T) + (1-\alpha) \text{UCB}^{a}_{t,i}(x_{t,i}) - h(x_{t,i}) \Big]\\
&= \sum^N_{i=1}\sum^{t_p + E_p - 1}_{t=t_p} \Big[ \alpha (\text{UCB}^{b}_{t,i}(x_{t,i})-h(x_{t,i})) + (1-\alpha) (\text{UCB}^{a}_{t,i}(x_{t,i}) - h(x_{t,i})) \\
&\qquad + \alpha \varepsilon_{\text{linear}}(m,T) \Big]\\
% &\stackrel{(c)}{\leq} \sum^N_{i=1}\Big[4 + \sum^{t_p + E_p - 1}_{t=t_p} \big( (1-\alpha) (\text{UCB}^{a}_{t,i}(x_{t,i}) - h(x_{t,i})) \big)\Big] + \\
% &\qquad \sum^N_{i=1}\Big[\sum^{t_p + E_p - 1}_{t=t_p} \big( \alpha (\text{UCB}^{b}_{t,i}(x_{t,i})-h(x_{t,i})) + \alpha \varepsilon_{\text{linear}}(m,T) \big)\Big]
% \\
% &\stackrel{(c)}{\leq} \sum^N_{i=1}\Big[4 + \sum^{t_p + E_p - 2}_{t=t_p+1} \big( (1-\alpha) (\text{UCB}^{a}_{t,i}(x_{t,i}) - h(x_{t,i})) \big)\Big] + \\
% &\qquad \sum^N_{i=1}\Big[\sum^{t_p + E_p - 1}_{t=t_p} \big( \alpha (\text{UCB}^{b}_{t,i}(x_{t,i})-h(x_{t,i})) + \alpha \varepsilon_{\text{linear}}(m,T) \big)\Big]
% \\
&\stackrel{(c)}{\leq} \underbrace{\sum^N_{i=1}\Big[4 + \sum^{t_p + E_p - 2}_{t=t_p+1} \big(\text{UCB}^{a}_{t,i}(x_{t,i}) - h(x_{t,i}) \big)\Big]}_{A} \\
&\qquad + \underbrace{\sum^N_{i=1}\sum^{t_p + E_p - 1}_{t=t_p} \Big[ \alpha (\text{UCB}^{b}_{t,i}(x_{t,i})-h(x_{t,i})) + \alpha \varepsilon_{\text{linear}}(m,T) \Big]}_{B}.
% \\
% &\stackrel{(d)}{\leq} \sum^N_{i=1}\Big[4 + \sum^{t_p + E_p - 2}_{t=t_p+1} 2\nu_{TKN} \sqrt{\lambda} \norm{g(x_{t,i};\theta_0) / \sqrt{m}}_{\overline{V}_{t,i}^{-1}} \Big]\\
% %&\leq \sum^N_{i=1}\left(4 + \sum^{t_p + E_p - 2}_{t=t_p} 2\nu_{TN} \min\{  \sqrt{\lambda} \norm{g(x_{t,i};\theta_0) / \sqrt{m}}_{\overline{V}_{t,i}^{-1}}, 1 \}\right)\\
% &\stackrel{(e)}{\leq} \sum^N_{i=1}\Big(4 + 2\nu_{TKN} \sqrt{\kappa_0} \sum^{t_p + E_p - 2}_{t=t_p} \min\{ \sqrt{\lambda} \norm{g(x_{t,i};\theta_0) / \sqrt{m}}_{\overline{V}_{t,i}^{-1}}, 1 \}\Big)\\
% &\stackrel{(f)}{\leq} \sum^N_{i=1}\Big(4 + 2\nu_{TKN} \sqrt{\kappa_0\lambda} \sum^{t_p + E_p - 2}_{t=t_p} \min\{  \norm{g(x_{t,i};\theta_0) / \sqrt{m}}_{\overline{V}_{t,i}^{-1}}, 1 \}\Big)
\end{split}
\label{eq:regret:bad:epochs:eq:1:new:general}
\end{equation}
Step $(a)$ follows from Lemma \ref{lemma:confidence:bound:ucb:2} (i.e., the validity of $\text{UCB}^{a}_{t,i}$) and Lemma \ref{lemma:confidence:bound:ucb:1} (i.e., the validity of $\text{UCB}^{b}_{t,i}$).
Step $(b)$ results from the way in which $x_{t,i}$ is selected (line 7 of Algo.~\ref{algo:agent}): $x_{t,i} = {\arg\max}_{x\in\mathcal{X}_{t,i}} (1- \alpha) \text{UCB}^{a}_{t,i}(x) + \alpha \text{UCB}^{b}_{t,i}(x)$.
For step $(c)$, the term $A$ is obtained by upper-bounding the regrets of the first and last iteration within this epoch by $2$ and using the fact that $\alpha\leq1$.
% derived using steps $(a)$, $(b)$ and $(c)$ of \eqref{eq:regret:bad:epochs:eq:1}, 

Next, we can separately analyze the terms $A$ and $B$ in \eqref{eq:regret:bad:epochs:eq:1:new:general}.
Firstly, note that the term $A$ is the same as step $(c)$ of \eqref{eq:regret:bad:epochs:eq:1}, therefore, we can follow the same steps of analyses in App.~\ref{subsec:regret:bad:epochs} (i.e., \eqref{eq:regret:bad:epochs:eq:1}, \eqref{eq:regret:bad:epochs:eq:2}, \eqref{eq:regret:bad:epochs:eq:3}, \eqref{eq:final:regret:upper:bound:in:bad:epochs} and \eqref{eq:final:final:regret:upper:bound:in:bad:epochs}) to show that after summing the term $A$ across all bad epochs, we get an upper bound of $\mathcal{O}\left( \widetilde{d}\sqrt{TN} \right)$. Secondly, for the term $B$, we can in fact follow similar steps of analysis in \eqref{eq:upper:bound:inst:regret} to show that every term inside the square bracket of the term $B$ is upper-bounded by the last two terms in \eqref{eq:upper:bound:inst:regret}. That is, 
\begin{equation}
    \alpha (\text{UCB}^{b}_{t,i}(x_{t,i})-h(x_{t,i})) + \alpha \varepsilon_{\text{linear}}(m,T) \leq \alpha 2 \nu_{TK}  \frac{1}{N} \sum^N_{j=1} \widetilde{\sigma}^{\text{local}}_{t_p-1,j}(x_{t,i}) + 2 \alpha \varepsilon_{\text{linear}}(m,T).
\end{equation}
As a result, we can follow the same steps of analysis in App.~\ref{app:subsubsec:upper:bound:sum:of:second:term:good:epochs} (after making the modification using \eqref{eq:good:epochs:second:term:eq:1:new:general:new}; note that the analysis in App.~\ref{app:subsubsec:upper:bound:sum:of:second:term:good:epochs} is applicable to both good and bad epochs)
% (because the analysis there is applicable to both good and bad epochs) 
to show that after summing over all bad epochs, the term $B$ can be upper-bounded by $\widetilde{O}\left(\sqrt{\overline{D}} \widetilde{d}_{\max} N \sqrt{T} + TN\varepsilon_{\text{linear}}(m,T)\right)$.
Next, combining the upper bounds on both $A$ and $B$ (after summing across all bad epochs), we have that for the general algorithm, the total regrets from all bad epochs can be upper bounded by
\begin{equation}
\begin{split}
R_T^{\text{bad}} =\widetilde{O}\left( \widetilde{d}\sqrt{TN} + \sqrt{\overline{D}} \widetilde{d}_{\max} N \sqrt{T} + TN\varepsilon_{\text{linear}}(m,T) \right),
\end{split}
\label{eq:final:regret:upper:bound:in:bad:epochs:new:general}
\end{equation}
which is in fact the same as the upper bound on the total regrets from all good epochs which we have derived in \eqref{eq:final:regret:upper:bound:in:good:epochs:new:general}.

Finally, following the same analysis in App.~\ref{subsec:regret:total}, we can show that the final regret upper bound for the general algorithm, in which we do not set $\alpha=0$ in any iteration, is 
\begin{equation}
\begin{split}
R_T = \widetilde{O}\Big(\widetilde{d} \sqrt{TN} + \sqrt{\overline{D}} \widetilde{d}_{\max} N \sqrt{T}\Big).
\end{split}
\label{eq:final:regret:upper:bound:new:general:algorithm}
\end{equation}
Note that compared to our regret upper bound from Theorem \ref{theorem:regret}, the regret upper bound for the general algorithm (i.e., when we choose the value of $\alpha$ using the method from Sec.~\ref{subsec:weight:two:ucbs} in every iteration) only includes an additional multiplicative term of $\sqrt{\overline{D}}$ in the second term.
Of note, when communication indeed occurs after each iteration (i.e., $E_p=1$ for every epoch $p$), we have that $\overline{D}=1$ because $\frac{\text{det} V^{\text{local}}_{t_p+E_p-2,i} }{\text{det} V^{\text{local}}_{t_p-1,i}} =1$ (\eqref{eq:additional:assumption}). 
In this case, the version of our algorithm analyzed in Theorem \ref{theorem:regret} becomes the same as our general algorithm (Sec.~\ref{subsec:theoretical:results}), and interestingly, the regret upper bound of our general algorithm (\eqref{eq:final:regret:upper:bound:new:general:algorithm}) also becomes the same as Theorem \ref{theorem:regret} because $\overline{D}=1$.

% }

%\appendix
%\section{Appendix}
%You may include other additional sections here.

\end{document}